\documentclass[onefignum,onetabnum]{siamonline220329}

\usepackage{lipsum}
\usepackage{amsfonts}
\usepackage{graphicx}
\usepackage{epstopdf}
\usepackage{algorithmic}
\usepackage{mathtools}
\usepackage{mathrsfs}
\usepackage{subfigure}
\usepackage{multirow}
\usepackage{float}
\ifpdf
  \DeclareGraphicsExtensions{.eps,.pdf,.png,.jpg}
\else
  \DeclareGraphicsExtensions{.eps}
\fi

\usepackage{enumitem}
\setlist[enumerate]{leftmargin=.5in}
\setlist[itemize]{leftmargin=.5in}

\newsiamremark{remark}{Remark}
\newsiamremark{hypothesis}{Hypothesis}
\crefname{hypothesis}{Hypothesis}{Hypotheses}
\Crefname{ALC@unique}{Line}{Lines}
\newsiamthm{claim}{Claim}

\headers{Generalized Opponent Transformation Total Variation}{Zhantao Ma, Michael K. Ng}

\title{Multispectral Image Restoration 
by Generalized Opponent Transformation Total Variation\thanks{Submitted to the editors March 19, 2024.
\funding{This work was funded by HKRGC GRF 12300519, 17201020 and 17300021, HKRGC CRF C1013-21GF
and C7004-21GF, and Joint NSFC and RGC N-HKU769/21.}}}

\author{Zhantao Ma\thanks{Department of Mathematics, The University of Hong Kong, Pokfulam, Hong Kong (mazhantao@connect.hku.hk).}
\and Michael K. Ng\thanks{Department of Mathematics, Hong Kong Baptist University, Kowloon Tong, Hong Kong (michael-ng@hkbu.edu.hk).}
}

\usepackage{amsopn}

\ifpdf
\hypersetup{
}
\fi

\begin{document}

\maketitle

\begin{abstract}
Multispectral images (MSI) contain light information in different wavelengths of objects, which convey spectral-spatial information and help improve the performance of various image processing tasks. Numerous techniques have been created to extend the application of total variation regularization in restoring multispectral images, for example, based on channel coupling and adaptive total variation regularization. The primary contribution of this paper is to propose and develop a new multispectral total variation regularization in a generalized opponent transformation domain instead of the original multispectral image domain. Here opponent transformations for multispectral images are generalized from a well-known opponent transformation for color images. We will explore the properties of generalized opponent transformation total variation (GOTTV) regularization and the corresponding optimization formula for multispectral image restoration. To evaluate the effectiveness of the new GOTTV method, we provide numerical examples that showcase its superior performance compared to existing multispectral image total variation methods, using criteria such as MPSNR and MSSIM.
\end{abstract}



\section{Introduction}
\label{sec:Introduction}
Images serve as crucial information carriers, exhibiting various forms such as grayscale, color, and multispectral images. Among them, multispectral images (MSI) capture individual images at specific wavelengths, typically spanning more than two and often dozens of wave numbers across the electromagnetic spectrum \cite{said2016multispectral}\cite{amigo2020hyperspectral}.
MSI contain a wealth of spectral-spatial information, encompassing ranges from the visible spectrum to infrared and ultraviolet ranges. This characteristic renders multispectral imaging particularly advantageous for applications in the food industry \cite{jaillais2012identification}. Furthermore, the moderate number of spectra in multispectral images, as compared to hyperspectral images, makes them well-suited for online processing tasks \cite{elmasry2019recent}, including online classification \cite{sendin2018classification}, etc.  

Despite their benefits, multispectral images often suffer from noise and blurring caused by internal factors, such as thermal effects, and external factors, such as insufficient lighting \cite{peng2014decomposable} and aerodynamic turbulence \cite{moffat1969theoretical}. These interferences significantly degrade image quality and have a detrimental impact on subsequent applications, including target detection \cite{reddy2020multispectral} and classification \cite{akar2012classification}.

The total variation (TV) grayscale image restoration model, initially proposed by Rudin, Osher, and Fatemi \cite{rudin1992nonlinear}, has gained considerable attention from researchers due to its ability to preserve image boundaries while effectively removing noise.
For a grayscale image $V\in\mathbb{R}^{m\times n}$ degraded by blurring with kernel $K\in\mathbb{R}^{m\times n}$ and i.i.d. additive Gaussian white noise, the total variation image restoration model is shown as follows: 
\begin{equation}
    U^{\ast}=\underset{U}{\text{argmin}}\left\{ TV(U)+\frac{\lambda}{2}\left\| K \star U-V\right\|^{2}\right\},
		TV(U):=\sum_{i=1}^{m}\sum_{j=1}^{n}\sqrt{(D_{x} U(i,j))^{2}+(D_{y} U(i,j))^{2}},
\end{equation}
where $\| \cdot \|$ is the Euclidean norm, $\star$ is the circular convolution, and 
$$
D_{x} U(i,j)= U(i,j+1)-U(i,j), \quad 
D_{y} U(i,j)= U(i+1,j)-U(i,j), 
$$
refer to the difference between two neighborhood pixels
in $x$-spatial and $y$-spatial directions under periodic boundary conditions.
The restoration process can be achieved by solving the minimization problem formulated by the above equation. 

For multispectral images ${\bf U} \in \mathbb{R}^{m\times n\times d}$ with $d$ channels, 
Chan et al. extended TV regularization to vectorial total variation (VTV) \cite{bresson2008fast} by coupling channel locally \cite{aastrom2016double}:
\begin{equation}
VTV( {\bf U}) :=\sum_{i=1}^{m}\sum_{j=1}^{n}
    \sqrt{\sum_{k=1}^{d} (D_{x} {\bf U}(i,j,k))^{2}+(D_{y} {\bf U}(i,j,k))^{2}}.
\end{equation}
However, the recovery result of this coupling method suffers from insufficient smoothness in homogeneous areas, i.e., shimmering \cite{aastrom2016double}. Later, Yuan et al. proposed the spectral-spatial adaptive TV (SSAHTV) by adding a pixel-by-pixel adaptive weight into VTV, which aims at keeping it smooth in a homogeneous area and sharp in the edge \cite{yuan2012hyperspectral}:
\begin{equation}
\label{SSAHTV}
SSAHTV(\mathbf{U}) :=\sum_{i=1}^{m}\sum_{j=1}^{n} W(i,j)
    \sqrt{\sum_{k=1}^{d} (D_{x}\mathbf{U}(i,j,k))^{2}+(D_{y}\mathbf{U}(i,j,k))^{2}},
\end{equation}
where 
$$
W(i,j)=\frac{G(i,j)}{\overline{G}(i,j)},\quad 
G(i,j)=\frac{1}{1+\mu\sqrt{\sum_{k=1}^{d} (D_{x}\mathbf{V}(i,j,k))^{2}+(D_{y}\mathbf{V}(i,j,k))^{2}}},
$$
and 
$\overline{G}(i,j)$ is mean value of $G(i,j)$. 

The anisotropic spectral-spatial TV (ASSTV) discovers the connection between spectra in the viewpoint of spectral smoothness. It performs the forward difference along both the spatial and spectrum direction \cite{chang2015anisotropic}. The requirement of spectral smoothness helps retain smoothness in a homogeneous area:
\begin{equation}
    ASSTV({\bf U}) := \sum_{i=1}^{m}\sum_{j=1}^{n}
		\sum_{k=1}^{d}|D_{x} {\bf U}(i,j,k)|+|D_{y}
		{\bf U}(i,j,k)|+|D_{f} {\bf U}(i,j,k)|,
\end{equation}
where $D_{f} {\bf U}(i,j,k)= {\bf U}(i,j,k+1)
- {\bf U}(i,j,k)$
refers to the difference between two neighborhood pixels
in the frequency direction. However,
the following experimental part shows some restoration results that cannot preserve texture well.
     
For color images, a special form of MSI, the saturation-value total variation 
restoration model proposed by Jia et al.
\cite{jia2019color} performs better than VTV in terms of color fidelity and reducing shimmering artifacts. SVTV utilizes the saturation and value information to perform restoration, which has the ability in well detecting the edge \cite{denis2007spatial}. From a different perspective, the saturation and value functions convey information about the color pixels in the opponent space. 
It is worth noting that numerous color image processing methods mimic the opponent pattern observed in the human visual system (HVS) to achieve improved results \cite{muhammad2011opponent}. 
In the opponent color space, the opponent 
transformation \cite{lammens1994computational}
is applied to obtain new channels, i.e.,
the difference between pixel values in red and green channels (red-green
channel), the sum 
of the difference between pixel values in the red and blue channel 
and the difference between pixel values in the green and blue channel 
(yellow-blue channel), and the average of pixel values in red, green
and blue channels. Indeed, the SVTV can be viewed as performing VTV restoration in the 
opponent color space shown in \cite{jia2019color}. By operating in this color space, SVTV leverages the properties of the opponent color representation to enhance color image restoration, resulting in improved overall performance compared to traditional TV-based methods.

In the context of multispectral images, which comprise multiple closely related channels, our objective is to propose opponent multispectral spaces that share similarities with the opponent color space. The fundamental idea revolves around combining the differences between pixel values in various multispectral channels (analogous to the red-green and yellow-blue channels in color images) and the average pixel values across all multispectral channels.

Our paper illustrates how to construct a generalized opponent transformation matrix with dimensions of $d$-by-$d$, capable of transforming multispectral images with $d$ channels into a generalized opponent domain. We further investigate the properties of such a generalized opponent transformation. Using this transformation, we establish an optimization model for multispectral image restoration. To evaluate the proposed multispectral image restoration model, we present numerical examples demonstrating its superior performance compared to existing multispectral image total variation methods. We assess the model's effectiveness based on criteria such as mean peak signal-to-noise ratio (MPSNR) and mean structural similarity index (MSSIM), showcasing its ability to yield improved restoration results.

The outline of this paper is given as follows. In \cref{sec:Related total variation regularization}, we review saturation value total variation regularization. 
In \cref{sec:Generalized Opponent Transformations}, we define the generalized opponent transformation 
matrices for multispectral images. Also, we propose the corresponding generalized opponent transformation MSI restoration model.  
In \cref{sec:Numerical Experiments}, we report numerical examples to demonstrate
the proposed model. Finally, some concluding 
remarks are given in \cref{sec:Conclusion}.

\section{Opponent Transformation for 
Total Variation Regularization}
\label{sec:Related total variation regularization}
\subsection{Opponent transformation in SVTV}

Let's review the SVTV model for color images and its connection with the opponent color space and opponent transformation. Let us begin with a symbolic clarification: bold uppercase letters, $\mathbf{U}$, denote tensors, while uppercase letters represent matrices: $U$. For $\mathbf{U}(i, j)$, the notation signifies a vector composed of the values of the tensor at the $(i, j)$ position within its third dimension, while $\mathbf{U}(i, j,k)$ denotes the value of the tensor at the $(i, j,k)$ position. Occasionally, $U_1, ..., U_n$ is employed to denote the matrices corresponding to the first through nth channels of the tensor $\mathbf{U}$. The SVTV \cite{jia2019color} is defined as the total variation based on 
the saturation and value components in a color image
${\bf U} \in \mathbb{R}^{m\times n\times 3}$: 
{\small
\begin{equation}
\label{SVTVintro2}
\begin{aligned}
SVTV( {\bf U}) :=&  
\sum_{i=1}^{m}\sum_{j=1}^{n}
\sqrt{ SAT( D_x {\bf U}(i,j))^{2} + SAT ( D_y {\bf U}(i,j))^{2} } \\
&+
\alpha 
\sqrt{ VAL( D_x {\bf U}(i,j))^{2} + VAL ( D_y {\bf U}(i,j))^{2} }, 
\end{aligned}
\end{equation}
}
where $\alpha$ is a positive number to balance 
the saturation component and the value component 
in the regularization term, and
\begin{equation}
\label{satandval}
\begin{aligned}
&SAT( D_x {\bf U}(i,j)) :=   
    \sqrt{ 
	{\bf \Phi}_x(i,j,1)^2 +
{\bf \Phi}_x(i,j,2)^2},\quad VAL( D_x {\bf U}(i,j)) := | {\bf \Phi}_x(i,j,3)|,\\
&SAT( D_y {\bf U}(i,j)) :=   
    \sqrt{ 
	{\bf \Phi}_y(i,j,1)^2 +
{\bf \Phi}_y(i,j,2)^2}
,\quad
VAL( D_y {\bf U}(i,j)) := | {\bf \Phi}_y(i,j,3)|,
\end{aligned}
\end{equation}
\begin{equation}
\label{SVTVintro3}
\left [
\begin{array}{c}
{\bf \Phi}_z(i,j,1) \\
{\bf \Phi}_z(i,j,2) \\
{\bf \Phi}_z(i,j,3) \\
\end{array}
\right ] = Q_{1} \left [
\begin{array}{c}
D_z  {\bf U}(i,j,1) \\
D_z  {\bf U}(i,j,2) \\
D_z  {\bf U}(i,j,3) \\
\end{array}
\right ] \quad (z = x \ {\rm or} \ y),
\end{equation}
and 
\begin{equation}  \label{opppatt1}
Q_{1} =
\begin{bmatrix}
    \frac{1}{\sqrt{2}}&-\frac{1}{\sqrt{2}}&0\\
    \frac{1}{\sqrt{6}}&\frac{1}{\sqrt{6}}&-\frac{2}{\sqrt{6}}\\
    \frac{1}{\sqrt{3}}&\frac{1}{\sqrt{3}}&\frac{1}{\sqrt{3}}
    \end{bmatrix}.
			\end{equation}
The z-direction forward difference of saturation components (${\bf \Phi}_{z}(\cdot,\cdot,1)$ and 
${\bf \Phi}_{z}(\cdot,\cdot,2)$) 
and the value component (${\bf \Phi}_{z}(\cdot,\cdot,3)$) 
in the opponent color space are obtained by applying
$Q_{1}$ to $D_{z} {\bf U}$ with $z = x \ {\rm or} \ y$. We note that the first row of $Q_{1}$ refers to taking the difference 
between red and green channels, and the second row
of $Q_{1}$ refers to taking the sum of the 
differences between red and blue channels and between 
green and blue channels. We call the first two rows of $Q_{1}$ having the opponent structure since the first two rows of $Q_{1}$ decorrelate the opponent relationship between channels. The last row of $Q_{1}$ refers to taking the average of red, green, and blue 
channels, and we call the last row the average structure.
It is also clear that $Q_1$ is an orthogonal matrix, i.e., $Q_1^T Q_1 = Q_1 Q_1^T 
= I$.

In \cite{jia2019color}, it has been shown that the saturation component of 
$SVTV({\bf U})$ can also be expressed as follows:
\begin{equation}
\label{SVTVintro2a}
\begin{aligned}
&\sum_{i=1}^{m}\sum_{j=1}^{n}
\sqrt{ SAT( D_x {\bf U}(i,j))^{2} + SAT ( D_y {\bf U}(i,j))^{2} } \\
=&\sum_{i=1}^{m}\sum_{j=1}^{n} \frac{1}{3}
\sqrt{  \| C D_x {\bf U}(i,j)) \|^2 +  \| C D_y {\bf U}(i,j)) \|^2 },
\end{aligned}
\end{equation}
where 
\begin{equation}
\label{Cmatrix}
C=\begin{bmatrix}
  2 & -1 & -1 \\
-1 & 2 & -1 \\
-1 & -1 & 2 \\
    \end{bmatrix}.
    \end{equation}
Here $Q_1$ are eigenvectors of $C$ corresponding to the eigenvalues of $C$:
3, 3 0, respectively, i.e.,
\begin{equation} \label{diagC}
C Q_1^T = Q_1^T \Lambda, \quad {\rm where} \ 
\Lambda = \left [ 
\begin{array}{ccc}
3 & 0 & 0 \\
0 & 3 & 0 \\
0 & 0 & 0 \\
\end{array}
\right].
\end{equation}

These results establish the relationship between the saturation component of $SVTV$ and the matrix $C$ and further connect it to the eigenvectors and eigenvalues of $C$ represented by $Q_1$ and $\Lambda$, respectively. 

Also, the aforementioned equation \cref{SVTVintro2a} offers an explanation from the perspective of the human visual color system as to why SVTV using opponent transformation can effectively eliminate shimmering in noisy images. Specifically, given $\boldsymbol{\mu}$=(1/$\sqrt{3}$,1/$\sqrt{3}$,1/$\sqrt{3}$)$^{T}$, which serves as the grayscale axis for a color image, being a unit vector with equal RGB values. We know that in SVTV, SAT and VAL utilize the opponent information and average information from the opponent transformation, respectively. By leveraging the auxiliary intermediate step derived using the matrix $C$ mentioned in equation \cref{SVTVintro2a,Cmatrix}, we can easily obtain the following results expressed by $\boldsymbol{\mu}$:
\begin{equation}
    \label{SVTVexpla1}
    \begin{aligned}
            SAT(D_{z}&\mathbf{U}(i,j))=\frac{1}{3}\| CD_{z}\mathbf{U}(i,j)\|=\|D_{z}\mathbf{U}(i,j)\cdot\boldsymbol{\mu}\boldsymbol{\mu}-D_{z}\mathbf{U}(i,j)\|,\\
            &VAL(\mathbf{U}(i,j))=\|D_{z}\mathbf{U}(i,j)\cdot\boldsymbol{\mu}\|\ (z=x\text{ or }y).
    \end{aligned}
\end{equation}

Therefore, the geometric interpretations of SAT and VAL are the distance to $\boldsymbol{\mu}$ and the projection length onto $\boldsymbol{\mu}$, respectively. Consequently, SAT quantifies how far a color deviates from neutral gray, representing the richness of the color. In other words, it represents a similar meaning to the amount of colorfulness in the human visual system (saturation). Additionally, it is evident that VAL denotes a similar meaning to the brightness of the color in the human visual system (value). In conclusion, the opponent transformation can express both the richness and brightness information of a color. This meaning elucidates why SVTV suppresses the shimmering phenomenon. Obviously, shimmering patches caused by noise lead to significant differences in color richness and brightness compared to the surrounding areas. This results in higher SVTV values. Therefore, minimizing the SVTV effectively removes shimmering patches from the image. 

In another angle, the color image is transformed into the opponent space in SVTV. This space helps decode the opponent's relationship between colors. This concept finds similarity with the opponent pattern observed in the human visual system (HVS), where a similar opponent process is employed as a second-stage procedure to aid in color perception as pointed out in \cite{lammens1994computational} and \cite{gose1996pattern}. 
Rao et al. \cite{muhammad2011opponent} imitated the opponent process in the human visual system and utilized the opponent transformation of digital color images, as shown in \cref{SVTVintro3}, to obtain better detection results. These results demonstrate the usefulness of opponent transformations and provide insight to explain why the SVTV model performs well for color image restoration by using the opponent transformation.

\subsection{Other opponent transformations}

Another interesting observation is the existence of alternative opponent transformations that can be utilized.
In \cite{ono2014decorrelated}, Ono and Yamada proposed another opponent transformation matrix for decorrelating color images. The opponent transformation matrix is
given by,
\begin{equation}
\label{opppatt2}
Q_{2}=\begin{bmatrix}
    \frac{1}{\sqrt{2}}&0&-\frac{1}{\sqrt{2}}\\
    \frac{1}{\sqrt{6}}&-\frac{2}{\sqrt{6}}&\frac{1}{\sqrt{6}}\\
    \frac{1}{\sqrt{3}}&\frac{1}{\sqrt{3}}&\frac{1}{\sqrt{3}}
    \end{bmatrix}.
    \end{equation}
    
			It is clear that the opponent pattern of $Q_2$ is different from that of $Q_1$. We note that the first row of $Q_2$ refers to taking the difference between red and blue
channels, and the second row of $Q_2$ refers to the sum of the differences between red and green
channels and between blue and green channels. In \cite{plataniotis2000color}, it is mentioned that this opponent pattern can
provide useful features for color image processing. 
		 Also, we observe that $Q_2$ can be obtained by 
swapping the second column and the third 
column of $Q_1$. The last row of $Q_{2}$ 
(the average of red, green, and blue 
channels)
is the same as that of 
$Q_1$. Note that $Q_2$ is orthogonal and $Q_2^T$ can diagonalize $C$ 
in \cref{Cmatrix}, i.e., $ C Q_2^T = Q_2^T \Lambda$
that is similar to \cref{diagC}.

Similarly, we can construct the following matrix:
\begin{equation}
\label{opppatt3}
Q_{3}=\begin{bmatrix}
    0&-\frac{1}{\sqrt{2}}&\frac{1}{\sqrt{2}}\\
    -\frac{2}{\sqrt{6}}&\frac{1}{\sqrt{6}}&\frac{1}{\sqrt{6}}\\
    \frac{1}{\sqrt{3}}&\frac{1}{\sqrt{3}}&\frac{1}{\sqrt{3}}
    \end{bmatrix},
    \end{equation}
by swapping the first column of $Q_1$ and 
the third column of $Q_1$. 
Here the first row of $Q_3$ refers to taking the difference between blue and green
channels, and the second row of $Q_3$ refers to taking the sum of the differences 
between green and red
channels and between blue and red channels. 
It is clear that the opponent structure of $Q_3$ is 
different from that by $Q_1$ and $Q_2$.
Note that $Q_3^T$ can diagonalize $C$ 
in \cref{Cmatrix}, i.e., $ C Q_3^T = Q_3^T \Lambda$
that is similar to \cref{diagC}.

We remark that when we swap the first column of $Q_1$ and 
the second column of $Q_1$, we obtain the same set of eigenvectors as $Q_1$ 
except for the minus sign in the first row. Similarly, by exchanging the first and third columns of $Q_2$ or the second and third columns of $Q_3$, we obtain the same set of eigenvectors as $Q_2$ or $Q_3$, respectively, with the only difference being the sign change in the first row. The difference of sign in the first row remains the same two columns for comparison. Consequently, we conclude that three opponent transformation matrices, namely $Q_1$, $Q_2$, and $Q_3$, are applicable in the saturation-value total variation regularization for color image restoration.

\section{Generalized Opponent Transformations}
\label{sec:Generalized Opponent Transformations}
The main aim of this section is to generalize the opponent transformation
to MSI images with $d$ channels. 
Similar to $Q_1$, $Q_2$ and $Q_3$, we can construct 
a $d$-by-$d$ opponent transformation matrix $Q$ such that
$Q$ is orthogonal and the last row of $Q$ refers to taking the average 
of all channel values:
\begin{equation} \label{last}
\left [ \frac{1}{\sqrt{d}}, \frac{1}{\sqrt{d}}, 
\cdots, \frac{1}{\sqrt{d}} \right ],
\end{equation}
while the other rows of $Q$ refer to taking the sum of the differences 
between two channels.
More precisely, we impose the following three conditions on 
the first $(d-1)$ rows of $Q$.
\begin{itemize}
\item[(G1)] The row sum of each row is equal to 0;
\item[(G2)] There is only one entry of negative value in each row;
\item[(G3)] There are exactly $i$ entries with the same positive value in $i$-th row for 
$1 \le i \le d-1$.
\end{itemize} 
It is easy to check that the first two rows of 
$Q_1$, $Q_2$ and $Q_3$ satisfy (G1), (G2) and (G3) for $d=3$.
Note that with these conditions, the first $(d-1)$ rows of $Q$ provide us with the opponent structure 
among $d$ channels. In (G1), we set the summation of channel values to be zero 
so that some channel values are set to be positive and some channel values are set to be
negative. In (G2) and (G3), we follow the idea of $Q_1$, $Q_2$, and $Q_3$, and consider
the differences of channel values. The following theorem states the characterization of 
orthogonal matrices satisfying (G1), (G2), and (G3).

\begin{theorem}
\label{T3.1}
The set $\mathcal{Q}_{d}$
of $d$-by-$d$ orthogonal matrices satisfying (G1), (G2), (G3), and the last row given
by (\ref{last}) is equal to 
$$
\{ Q = BP: P \ \text{are permutation matrices obtained by 
    permuting columns of the $I$} \},
    $$
where $B$ is a $d$-by-$d$ orthogonal matrix whose last row is the same as (\ref{last}) and its entries for former $d-1$ rows are given by
$$
[B]_{i,j} = 
\left \{
\begin{array}{ll}
\frac{1}{ \sqrt{ i(i+1) } }, & \quad j \le i, \\
- \frac{i}{ \sqrt{  i(i+1) } }, & \quad j = i+1, \\
0, & \quad j > i + 1, \\
\end{array}
\right.
$$
i.e., 
\begin{equation}
\label{sfmnf}
B = \begin{bmatrix}
\frac{1}{\sqrt{2}} & -\frac{1}{\sqrt{2}} & 0 & \dots  &  0 \\
\frac{1}{\sqrt{6}} & \frac{1}{\sqrt{6}} & -\frac{2}{\sqrt{6}} & \dots    & 0 \\
\vdots & \vdots & \vdots & \ddots &  \vdots \\
\frac{1}{\sqrt{d(d-1)}} & \frac{1}{\sqrt{d(d-1)}} & \frac{1}{\sqrt{d(d-1)}} & \dots    & -\frac{d-1}{\sqrt{d(d-1)}}\\
\frac{1}{\sqrt{d}} & \frac{1}{\sqrt{d}} & \frac{1}{\sqrt{d}} & \dots   &  \frac{1}{\sqrt{d}}
\end{bmatrix}.
\end{equation}
We treat two matrices that are identical except for the first row with a different sign as the same element. The total number of elements in $\mathcal{Q}_{d}$ equals $d!/2$. 
\end{theorem}

Note that when $d=3$, $Q_1$ in \cref{opppatt1}
is the same as $B$ stated in the theorem. Also $Q_2$ in 
\cref{opppatt2} is obtained by permuting the second 
and the third columns of $Q_1$,
and $Q_3$ in 
(\ref{opppatt3}) is obtained by permuting the first and the second columns
of $Q_1$. It is clear that $Q_1$, $Q_2$ and $Q_3$ are the three elements in 
$\mathcal{Q}_{3}$.

\begin{proof}
It is clear that $Q=BP$ satisfies (G1), (G2), (G3), and the last row given by (\ref{last}). 

Now we would like to show that any element $Q$ in $\mathcal{Q}_{d}$
should be in the form of $BP$.
We note that the Euclidean norm of the $i$-th row is equal to 1.
By using (G1), (G2) and (G3),
the value of the negative entry of the $i$-th row should be equal to
$- \frac{i}{ \sqrt{  i(i+1) } }$, 
and the value of the positive entry of the $i$-th row should be equal to 
$\frac{1}{ \sqrt{ i(i+1) } }$.
It is clear that the value of the other entries of the $i$-th row is zero. Therefore,
the values and quantities of positive and negative elements in each row of $Q$ 
are consistent with $B$.

Next, we identify the value above or below the negative entry ($[Q]_{i,j}$) of
the $i$-th row. 
Let us consider $k$-th row of $Q$ where $1 \le k < i \le d-1$. 
We plan to show that $[Q]_{k,j} = 0$. 
Suppose $[Q]_{k,j} \ne 0$. 
Case (i) $[Q]_{k,j} < 0$, the inner product of the $k$-th and the $i$-th rows of $Q$ 
is positive because of the summation of $[Q]_{k,j} [Q]_{i,j} > 0$ and nonnegative values. 
Case (ii) $[Q]_{k,j} > 0$, the inner product of the $k$-th and the $i$-th rows of $Q$ 
is negative because $[Q]_{k,j} [Q]_{i,j} = -\frac{i}{\sqrt{i(i+1)}\sqrt{k(k+1)}}$ and the summation of nonnegative values $\leq \frac{k-1}{\sqrt{i(i+1)}\sqrt{k(k+1)}}$. 
Both cases contradict the fact that 
the $k$-th row and the $i$-th row are orthogonal (the inner product must be zero). 
Let us consider $k$-th row of $Q$ where $1 \le i < k \le d-1$. 
By the above argument, $[Q]_{k,j}$ cannot be negative. 
Here we would like to show that $[Q]_{k,j}$ must be positive, which is equivalent to excluding the 0-value case.
Firstly, when $k=d-1$, it is clear that $[Q]_{d-1,j}$ is positive. 
When $k=d-2$, there is only one entry to be zero. The column 
position of the zero entry in the $d-2$ row should be the same as the negative entry in rows $d-1$. 
Therefore, $[Q]_{d-2,j}$ cannot be zero. 
By repeating the same argument, we can show that 
$[Q]_{k,j}$ cannot be zero.
It follows that $[Q]_{k,j}$ must be positive.

Moreover, $B$ also exhibits the same property in terms of the entries in the rows above and below the negative entry. Combining it with the first conclusion, which tells us that the number and magnitude of positive and negative entries in each row of $B$ and $Q$ are the same. Therefore, it can be concluded that the set of column vectors in $B$ and $Q$ are identical, allowing us to obtain Q by permutating the columns of $B$. We treat two matrices that are identical except for the first row with a different sign as the same element. Then there are $d!/2$ permutation matrices obtained by 
permuting columns of the identity matrix. Hence the result follows.
\end{proof}

Next we show that each element in $\mathcal{Q}_{d}$ 
are eigenvectors of a special matrix given in the following theorem.
Remind that $Q_1$, $Q_2$ and $Q_3$ in  
\cref{opppatt1}, 
\cref{opppatt2} and 
\cref{opppatt3} respectively are eigenvectors of $C$ in \cref{Cmatrix}.

\begin{theorem}
\label{T3.2}
Let $C_{d}$ be a $d$-by-$d$ symmetric matrix given by 
\begin{equation}
    \begin{aligned}
    C_{d}= \begin{bmatrix}
d-1& -1 & \cdots& -1\\
-1 & d-1 & \cdots & -1\\
\vdots & & \ddots & \vdots \\
-1 & \cdots & -1 &d-1
\end{bmatrix},  
\end{aligned}
\end{equation} 
(i.e., the main diagonal entries of $C_{d}$ 
are $(d-1)$ and its off-diagonal entries are all -1.)
Then the eigenvalues of $C_{d}$ are $d$ (with $d-1$ algebraic multiplicities) and 0,
and each element in $\mathcal{Q}_{d}$ is the associated eigenvectors, i.e.,
\begin{equation} \label{diagCd}
C_d Q^T = Q^T \Lambda_d,
\end{equation}
where $\Lambda_d$ is a diagonal matrix with the main diagonal entries given by 
$(\underbrace{d,d,\cdots,d}_{d-1 \ {\rm values}},0)$.
\end{theorem}  

\begin{proof}
It is straightforward to obtain the equation of eigenvalue $d$ and $0$, respectively:
\begin{equation}
\label{eigenequf}
\begin{aligned}
       \begin{bmatrix}
-1&\dots&-1\\
-1&\ddots&-1\\
-1&\dots&-1
\end{bmatrix} \mathbf{x}=0,\ \begin{bmatrix}
d-1&\dots&-1\\
-1&\ddots&-1\\
-1&\dots&d-1
\end{bmatrix} \mathbf{x}=0.
\end{aligned}
\end{equation}
The opponent structure and the average structure of $Q=BP$ make the former d-1 row vectors and its last row vector respectively the solutions of the first and second equation in (\ref{eigenequf}). It implies the equality: $C_d Q^T = Q^T \Lambda_d$ holds for any $Q\in\mathcal{Q}_{d}$.
\end{proof}

The matrix $C_{d}$ defined above shares several commonalities with $C$. $C_{d}$ and $C$ not only share the property of having all opponent transformations as their eigenbasis, but also exhibit consistent diagonal values and non-diagonal values and adhere to the properties of symmetry where the sum of each row equals zero. 

Moreover, \cref{connectCdAC} tells that $C_{d}$ can be decomposed as the sum of matrices whose elements are zero except one 3-order principal submatrix is $C$.
\begin{theorem}
    \label{connectCdAC}
\begin{equation}
\label{connectCdACeq}
C_{d}=\frac{1}{d-2}\sum_{H\in \mathcal{H}_{d}} H,
\end{equation}
\end{theorem}
where $\mathcal{H}_{d}$ collects matrices which are zero except one 3-order principal submatrix is $C$:
\begin{equation}
\mathcal{H}_{d}=\{H\in \mathbb{R}^{d\times d}|\begin{bmatrix}
H({l_{1},l_{1}}) & H({l_{1},l_{2}}) & H({l_{1},l_{3}}) \\
H({l_{2},l_{1}}) & H({l_{2},l_{2}}) & H({l_{2},l_{3}}) \\
H({l_{3},l_{1}}) & H({l_{3},l_{2}}) & H({l_{3},l_{3}}) 
\end{bmatrix}=C,l_{1}<l_{2}<l_{3}\}.
\end{equation}

\begin{proof}
    The proof contains three steps. First, when $d=3$, only $C$ in $\mathcal{H}_3$, so the equality holds. For $d>3$, the k-th diagonal value: $\sum_{H\in \mathcal{H}_{d}} H({k,k})$ equals to $(d-1)(d-2)$. Since there are $\begin{pmatrix}
     2     \\
     d-1    
    \end{pmatrix}$ number of $H$ takes value 2 at k-th diagonal position, i.e. $H({k,k})=2$. As for the non-diagonal value at the i-th row, j-th column: $\sum_{H\in \mathcal{H}_{d}} H({i,j})$ equals to -(d-2). Since there are $\begin{pmatrix}
     1     \\
     d-2    
    \end{pmatrix}$ number of $H$ takes value -1 at the i-th row, j-th column, i.e. $H({i,j})=-1$.
\end{proof}


In the next subsection, we will use $B$ and $C_d$ to construct  
generalized opponent transformation total variation regularization
for multispectral images. Also, we will use \cref{T3.2,connectCdAC} to demonstrate why our GOTTV restoration model performs well.

\subsection{Generalized Opponent 
Transformation Total Variation}

As mentioned in \cref{sec:Introduction}, the TV regularization term is effective in suppressing noise while preserving the boundary information of an image. Here we propose applying the TV regularization term to the multispectral space by using generalized opponent transformations $BP \in {\cal Q}_d$, and the resulting regularization term is called 
a Generalized Opponent Transformation Total Variation (GOTTV) regularization. This approach aims at utilizing the benefit of the TV to preserve the image boundaries while effectively exploiting the opponent information in multispectral images. 
For simplicity, we employ the generalized opponent transformation $B$
in the following discussion.
The GOTTV regularization is defined as follows:
\begin{equation}
\label{GOTTVdef}
    \begin{aligned}
&GOTTV(\mathbf{U})\\
&\coloneqq\sum_{i=1}^{m}\sum_{j=1}^{n}\sqrt{\left\| [\mathbf{b}_{1}D_{x}\mathbf{U}(i,j),\dots,\mathbf{b}_{d-1}D_{x}\mathbf{U}(i,j)]\right\|^{2}+\left\| [\mathbf{b}_{1}D_{y}\mathbf{U}(i,j),\dots,\mathbf{b}_{d-1}D_{y}\mathbf{U}(i,j)]\right\|^{2}}\\
        & \quad +\alpha\sqrt{\left\| \mathbf{b}_{d}D_{x}\mathbf{U}(i,j)\right\|^{2}+\left\| \mathbf{b}_{d}D_{y}\mathbf{U}(i,j)\right\|^{2}},
    \end{aligned}
\end{equation}
where $\mathbf{b}_k$ represents the $k$-th row of matrix $B$ ($1\leq k \leq d$), $\alpha$ is a positive number to balance the opponent component and the average component.
The first component of GOTTV transforms the gradients $D_x {\bf U}$ and $D_y {\bf U}$ using opponent structure of matrix $B$ and takes the point-wise sum of the $l_2$ norm of each transformed gradient. These norms summed over all channels ($k=1$ to $d-1$) to capture the opponent information. The second component transforms the gradients $D_x {\bf U}$ and $D_y {\bf U}$ using the average structure of matrix $B$ and computes the point-wise sum of the $l_2$ norm of the transformed gradient corresponding to the average channel ($k=d$). This part captures the average information across all channels.

By incorporating the TV regularization into the multispectral space defined by the generalized opponent transformation, the GOTTV regularization allows for effective exploitation of the opponent information present in the MSI for image restoration tasks. This regularization term can enhance the preservation of image boundaries while taking advantage of the interconnectedness of spectral information in the MSI.

It should be noted that from the definition of SVTV in \cref{SVTVintro2}, \cref{satandval}, \cref{SVTVintro3}, SVTV also consists of two components. 
The first component is the point-wise sum of the $l_{2}$ norm of the gradient transformed by the opponent structure of $Q_{1}$. The second component is the point-wise sum of the $l_{2}$ norm of the gradient transformed by the average structure of $Q_{1}$. When $d=3$, the $Q_{1}=B$ and the GOTTV can be reverted into the SVTV form.

In addition, the opponent term of the GOTTV can be expressed by $C_{d}$ similar to \cref{SVTVintro2a} as follows shown, 
\begin{equation}
\label{opponent trans eq1}
    \begin{aligned}
        &\sum_{i=1}^{m}\sum_{j=1}^{n}\sqrt{\left\| [\mathbf{b}_{1}D_{x}\mathbf{U}(i,j),\dots,\mathbf{b}_{d-1}D_{x}\mathbf{U}(i,j)]\right\|^{2}+\left\| [\mathbf{b}_{1}D_{y}\mathbf{U}(i,j),\dots,\mathbf{b}_{d-1}D_{y}\mathbf{U}(i,j)]\right\|^{2}}\\
        &=\sum_{i=1}^{m}\sum_{j=1}^{n}\frac{1}{d}\sqrt{\left\| \Lambda_{d}BD_{x}\mathbf{U}(i,j)\right\|^{2}+\left\| \Lambda_{d}BD_{y}\mathbf{U}(i,j)\right\|^{2}}\\
        &=\sum_{i=1}^{m}\sum_{j=1}^{n}\frac{1}{d}
\sqrt{  \| C_{d} D_x {\bf U}(i,j)) \|^2 +  \| C_{d} D_y {\bf U}(i,j)) \|^2 }.\\
&=\sum_{i=1}^{m}\sum_{j=1}^{n}\frac{1}{d(d-2)}
\sqrt{  \| \sum_{H\in \mathcal{H}_{d}} H D_x {\bf U}(i,j)) \|^2 +  \| \sum_{H\in \mathcal{H}_{d}} H D_y {\bf U}(i,j)) \|^2 }
    \end{aligned}
\end{equation}
{\small
\begin{equation}
\label{opponent trans eq2}
    \begin{aligned}
        &\sum_{i=1}^{m}\sum_{j=1}^{n}\sqrt{\left\| \mathbf{b}_{d}D_{x}\mathbf{U}(i,j)\right\|^{2}+\left\| \mathbf{b}_{d}D_{y}\mathbf{U}(i,j)\right\|^{2}}\\
        &=a_{1}\sum_{i=1}^{m}\sum_{j=1}^{n}\left((\sum_{1\leq l_{1}<l_{2}<l_{3}\leq d}VAL(D_{x}\begin{bmatrix}
            {\bf U}(i,j,l_{1})\\
            {\bf U}(i,j,l_{2})\\
            {\bf U}(i,j,l_{3})
        \end{bmatrix}))^2+(\sum_{1\leq l_{1}<l_{2}<l_{3}\leq d}VAL(D_{y}\begin{bmatrix}
            {\bf U}(i,j,l_{1})\\
            {\bf U}(i,j,l_{2})\\
            {\bf U}(i,j,l_{3})
        \end{bmatrix}))^2\right)^{\frac{1}{2}}\\
        &\text{where } a_{1}=\frac{2\sqrt{3}}{\sqrt{d}(d-1)(d-2)} \text{ is a constant.}
    \end{aligned}
\end{equation}
}

Remark that the above results hold for any element in $\mathcal{Q}_{d}$. The second last and the last equality holds by using the relationship \cref{diagCd}, \cref{connectCdACeq} deduced in \cref{T3.2} and \cref{connectCdAC}.  

The above equalities demonstrate that the generalized opponent transformation can convey the sum of saturation and value information of the color sub-images formed by any three channels of the multispectral image. Specifically, as each $H$ is a matrix except its principle submatrix at positions $l_{1}, l_{2}, l_{3}$, is equal to $C$, and the rest is zero, each $H D_{z} {\bf U}(i,j)(z=x\text{ or }y)$ essentially involves extracting the $l_{1}, l_{2}, l_{3}$ channels from the multispectral image to form a color sub-image. Subsequently, it calculates the saturation information of the x and y-directional gradients of this color sub-image at positions $(i,j)$. Hence, summing over $\mathcal{H}_{d}$ is equivalent to extracting all three channels that can compose a color sub-image from the multispectral image and adding up their saturation information in the x and y-directional gradients as shown in the last equality in \cref{opponent trans eq1}. 

Therefore, the opponent component of the GOTTV conveys the sum of saturation information of the x and y-directional gradients of the color sub-images formed by any three channels of the multispectral image. Additionally, equation \cref{opponent trans eq2} shows that the average component in GOTTV conveys the sum of value information of the x and y-directional gradients of the color sub-images formed by any three channels of the multispectral image. Combining the ability of saturation and value to mitigate the shimmering phenomenon of color images analyzed in the preceding sections, GOTTV effectively possesses the capability to reduce the shimmering phenomenon in each color sub-image of a multispectral image formed by its arbitrary three channels. The subsequent experimental section validates that GOTTV exhibits superior denoising results in easing shimmering. The above illustration not only elucidates from the perspective of color sub-image that our generalized opponent transformation processes excellent denoising capabilities but also emphasizes that it is precisely due to the inheritance of the key properties (G1-G3) from the 3-d opponent transformation by our generalized opponent transformation, that \cref{T3.2,connectCdAC} are validated. This underscores that our generalized opponent transformation holds not only practical significance but also theoretical value.

\subsection{Multispectral Image Restoration Model}
The image restoration problem in the context of observed multispectral image (MSI) can be formulated using the Maximum a posteriori framework (MAP) \cite{chambolle2010introduction}. The goal is to recover the cleaned MSI image $\mathbf{U}$ from the observed degraded image $\mathbf{V}\in\mathbb{R}^{m\times n\times d}$. The degradation is caused by blurring with the kernel $\mathbf{K}$ and the addition of Gaussian noise $\mathbf{N}$. The MAP framework formulates image restoration as a minimization problem that combines a fidelity term and a regularization term. In this case, we choose GOTTV as the regularization term. The GOTTV image restoration model can be formulated as follows:
\begin{equation}
    \label{orimiinformf}
    \mathbf{U}^{\ast}=\underset{\mathbf{U}\in \mathbb{R}^{m\times n\times d}}{\arg\min}\{F(\mathbf{U})=GOTTV(\mathbf{U}) + \frac{\lambda}{2}\left\| \mathbf{K}\star\mathbf{U}-\mathbf{V}\right\|^{2}\},
\end{equation}
where $\lambda$ is a positive parameter to balance the regularization and fidelity terms. 

The first term in the equation corresponds to the GOTTV regularization, which effectively utilizes the opponent information across spectra to preserve edges and boundaries in the image. The second term represents the fidelity term, which measures the discrepancy between the recovered and degraded images, accounting for the noise and blurring effects. By minimizing this objective function, the restoration algorithm aims to find the optimal solution that balances the GOTTV regularization term for edge preservation and the fidelity term for faithful recovery from the degraded image. This joint optimization process facilitates the removal of noise and blurring while retaining important structural details in the restored image.

Under a assumption \cite{wu2010augmented} that the null space of forward difference operator: $D=(D_{x},D_{y})$ and blurring operator $\mathbf{K}$ have 0 as the only common element, i.e. $\text{Null}(D)\cap\text{Null}(\mathbf{K})=\{0\}$, then we have the below result:
\begin{theorem}
The solution of the minimization formula \cref{orimiinformf} exists. If $\text{Null}(\mathbf{K})=\{0\}$, then the solution is unique.    
\end{theorem}
\begin{proof}
We introduce a new variable $\mathbf{\Phi}\in\mathbb{R}^{m\times n\times d}$, satisfies:
\begin{equation}
\label{operator}
\begin{bmatrix}
    \Phi_{1}\\
    \vdots\\
    \Phi_{d}
\end{bmatrix}=\mathbf{B}\begin{bmatrix}
    U_{1}\\
    \vdots\\
    U_{d}
\end{bmatrix}, \text{where }\mathbf{B}=\begin{bmatrix}
    B(1,1)I &\dots &B(1,d)I\\
    \vdots &\ddots&\vdots\\
    B(d,1)I &\dots &B(d,d)I
\end{bmatrix}
\end{equation}
i.e. in any position, $\mathbf{\Phi}(i,j)=B \mathbf{U}(i,j)$ holds. Then below equality holds:
{ 
\begin{equation}
\label{GOTTVVTVdedu}
    \begin{aligned}
&GOTTV(\mathbf{U})\\
&=\sum_{i=1}^{m}\sum_{j=1}^{n}\sqrt{\left\| [D_{x}\mathbf{b}_{1}\mathbf{U}(i,j),\dots,D_{x}\mathbf{b}_{d-1}\mathbf{U}(i,j)]\right\|^{2}+\left\| [D_{y}\mathbf{b}_{1}\mathbf{U}(i,j),\dots,D_{y}\mathbf{b}_{d-1}\mathbf{U}(i,j)]\right\|^{2}}\\
        &+\alpha\sqrt{\left\| D_{x}\mathbf{b}_{d}\mathbf{U}(i,j)\right\|^{2}+\left\| D_{y}\mathbf{b}_{d}\mathbf{U}(i,j)\right\|^{2}}\\
        &=\sum_{i=1}^{m}\sum_{j=1}^{n}\sqrt{\left\| [D_{x}\mathbf{\Phi}(i,j,1),\dots,D_{x}\mathbf{\Phi}(i,j,d-1)]\right\|^{2}+\left\| [D_{y}\mathbf{\Phi}(i,j,1),\dots,D_{y}\mathbf{\Phi}(i,j,d-1)]\right\|^{2}}\\
        &+\alpha\sqrt{\left\| D_{x}\mathbf{\Phi}(i,j,d)\right\|^{2}+\left\| D_{y}\mathbf{\Phi}(i,j,d)\right\|^{2}}.
    \end{aligned}
\end{equation}
}
    From deduction \cref{GOTTVVTVdedu}, we know GOTTV is the composition of an affine function ($\mathbf{\Phi}(i,j)=B\mathbf{U}(i,j)$) and a convex function (VTV), so GOTTV is convex. Therefore, under the assumption, that F is convex, coercive, and proper. By the generalized Weierstrass existence theorem \cite{zeidler2012applied}, the existence of the solution can be guaranteed. Moreover, when $\text{Null}(\mathbf{K})=\{0\}$, then F is strictly convex, so the solution is unique. 
\end{proof}

\subsection{Numerical Algorithm}

 The alternating direction method of multipliers (ADMM) is a widely used iterative algorithm for solving optimization problems with convex objectives. In order to apply the classic ADMM algorithm to obtain the minimizer of $F$, we introduce the variable $\mathbf{\Phi}$, same as \cref{operator}. Also, denote $\tilde{\mathbf{K}}=\mathbf{B}\mathbf{K}\mathbf{B}^{T}$ and $\tilde{\mathbf{V}}(i,j)=B\mathbf{V}(i,j)$, for $1\leq i\leq m, 1\leq j\leq n$. Then the minimization formula in \cref{orimiinformf} can be written as:
\begin{equation}
\label{Lageq1}
\begin{aligned}
    \min_{\mathbf{\Phi}} &\sum_{i=1}^{m}\sum_{j=1}^{n}\sqrt{\left\| [D_{x}\mathbf{\Phi}(i,j,1),\dots,D_{x}\mathbf{\Phi}(i,j,d-1)]\right\|^{2}+\left\| [D_{y}\mathbf{\Phi}(i,j,1),\dots,D_{y}\mathbf{\Phi}(i,j,d-1)]\right\|^{2}}\\
        &+\alpha\sqrt{\left\| D_{x}\mathbf{\Phi}(i,j,d)\right\|^{2}+\left\| D_{y}\mathbf{\Phi}(i,j,d)\right\|^{2}}+\frac{\lambda}{2}\left\| \tilde{\mathbf{K}}\star\mathbf{\Phi}-\tilde{\mathbf{V}}\right\|^{2}.\\
\end{aligned}
\end{equation}
Then we introduce auxiliary variables $\mathbf{W}_{x (y)}\in\mathbb{R}^{m\times n\times d}$ to act as the x(y)-direction forward difference of $\mathbf{\Phi}$ and rewritten the minimization formula \cref{Lageq1} as:
\begin{equation}
\label{Lageq2}
\begin{aligned}
    \min_{\mathbf{\Phi}} &\sum_{i=1}^{m}\sum_{j=1}^{n}\sqrt{\left\| [\mathbf{W}_{x}(i,j,1),\dots,\mathbf{W}_{x}(i,j,d-1)]\right\|^{2}+\left\| [\mathbf{W}_{y}(i,j,1),\dots,\mathbf{W}_{y}(i,j,d-1)]\right\|^{2}}\\
        &+\alpha\sqrt{\left\| \mathbf{W}_{x}(i,j,d)\right\|^{2}+\left\| \mathbf{W}_{y}(i,j,d)\right\|^{2}}+\frac{\lambda}{2}\left\| \tilde{\mathbf{K}}\star\mathbf{\Phi}-\tilde{\mathbf{V}}\right\|^{2}\\
        \text{s.t. }& \mathbf{W}_{x}=D_{x}\mathbf{\Phi}, \mathbf{W}_{y}=D_{y}\mathbf{\Phi}.
\end{aligned}
\end{equation}
At last, we write down the augmented Lagrange formula of \cref{Lageq2}, which is: 
\begin{equation}
    \label{ALMofT}
    \begin{aligned}
    &\mathscr{L}(\mathbf{\Phi},\mathbf{W}_{x},\mathbf{W}_{y};\mathbf{\Upsilon}_{x},\mathbf{\Upsilon}_{y})\\
    &=\sum_{i=1}^{m}\sum_{j=1}^{n}\sqrt{\left\| [\mathbf{W}_{x}(i,j,1),\dots,\mathbf{W}_{x}(i,j,d-1)]\right\|^{2}+\left\| [\mathbf{W}_{y}(i,j,1),\dots,\mathbf{W}_{y}(i,j,d-1)]\right\|^{2}}\\
    &+\alpha \sqrt{\left\|\mathbf{W}_{x}(i,j,d)\right\|^{2}+\left\|\mathbf{W}_{y}(i,j,d)\right\|^{2}}+\frac{\lambda}{2}\left\| \tilde{\mathbf{K}}\star \mathbf{\Phi}-\tilde{\mathbf{V}}\right\|^{2}\\
    &+<\mathbf{\Upsilon}_{x},\mathbf{W}_{x}-D_{x}\mathbf{\Phi}>+<\mathbf{\Upsilon}_{y},\mathbf{W}_{y}-D_{y}\mathbf{\Phi}>+\frac{r}{2}\left\| \mathbf{W}_{x}-D_{x}\mathbf{\Phi}\right\|^{2}+\frac{r}{2}\left\| \mathbf{W}_{y}-D_{y}\mathbf{\Phi}\right\|^{2},
    \end{aligned}
\end{equation}
where $\mathbf{\Upsilon}_{x(y)}\in\mathbb{R}^{m\times n\times d}$ are Lagrangian multiplier.
Then the ADMM algorithm is implemented by the following \cref{alg:ADMM}.
\begin{algorithm}[h]
\caption{ADMM}
\label{alg:ADMM}
\begin{algorithmic}[1]
\REQUIRE maxitr, relerr, $\alpha$, $r$,$r_{max}$, $\rho$ ,$\lambda$, $\mathbf{W}_{x(y)}^{(0)}\leftarrow 0,\mathbf{\Upsilon}_{x(y)}^{(0)}\leftarrow 0$\;

    \FOR {$k\leftarrow 1$ to $maxitr$}          
    \STATE $\mathbf{\Phi}^{(k)}\leftarrow\underset{\mathbf{\Phi}}{\arg\min}\mathscr{L}(\mathbf{\Phi},\mathbf{W}^{(k-1)}_{x},\mathbf{W}^{(k-1)}_{y};\mathbf{\Upsilon}_{x}^{(k-1)},\mathbf{\Upsilon}_{y}^{(k-1)})$;\\
    \STATE $(\mathbf{W}^{(k)}_{x},\mathbf{W}^{(k)}_{y})\leftarrow\underset{(\mathbf{W}_{x},\mathbf{W}_{y})}{\arg\min}\mathscr{L}(\mathbf{\Phi}^{(k)},\mathbf{W}_{x},\mathbf{W}_{y};\mathbf{\Upsilon}_{x}^{(k-1)},\mathbf{\Upsilon}_{y}^{(k-1)})$;\\
    \STATE $(\mathbf{\Upsilon}_{x}^{(k)},\mathbf{\Upsilon}_{y}^{(k)})\leftarrow(\mathbf{\Upsilon}_{x}^{(k-1)},\mathbf{\Upsilon}_{y}^{(k-1)})+r(\mathbf{W}_{x}^{(k)}-D_{x}\mathbf{\Phi}^{(k)},\mathbf{W}_{y}^{(k)}-D_{y}\mathbf{\Phi}^{(k)})$;
       \STATE$\mathbf{U}^{(k)}(i,j)\leftarrow B^{T}\mathbf{\Phi}^{(k)}(i,j),\forall 1\leq i\leq m, 1\leq j\leq n$;
    \IF{$r<r_{max}$}
\STATE $r=r*\rho$;
    \ENDIF
        
    \IF {$\frac{\left\| \mathbf{\Phi^{(k)}}-\mathbf{\Phi^{(k-1)}}\right\|}{\left\| \mathbf{\Phi^{(k)}}\right\|}<$relerr}

    \STATE break;
    \ENDIF

    \ENDFOR

\RETURN $\mathbf{U}^{(k)}$
\end{algorithmic}
\end{algorithm}

\vspace{3mm}

To provide a theoretical foundation for the convergence of the ADMM algorithm, we state the following convergence theorem:
\begin{theorem}
    Assume that $\mathbf{U}^{\ast}$ is the solution of (\ref{orimiinformf}), then the sequence $\mathbf{U}^{k}$ generated by algorithm \ref{alg:ADMM} satisfies:
    \begin{equation}
        \label{convergence property 1}
        \lim_{k\to \infty} F(\mathbf{U}^{(k)})=F(\mathbf{U}^{\ast}).
    \end{equation}
    If we further have $\text{Null}(\textbf{K})=0$, then 
    \begin{equation}
        \label{convergence property 2}
        \lim_{k\to \infty}\mathbf{U}^{(k)}=\mathbf{U}^{\ast}.
    \end{equation}
\end{theorem}
\begin{proof}
    Utilizes the conclusion of theorem 4.4 from the paper \cite{wu2010augmented}, we can easily get the following results:
    \begin{equation}
        \label{proof convergence property 1}
        \lim_{k\to \infty} G(\mathbf{\Phi}^{(k)})=G(\mathbf{\Phi}^{\ast}),
    \end{equation}
    Where G is the function corresponding to the minimization problem of $\mathbf{\Phi}$ in \cref{Lageq1}. Also, if we further have $\text{Null}(\textbf{K})=0$, then 
    \begin{equation}
        \label{proof convergence property 2}
        \lim_{k\to \infty}\mathbf{\Phi}^{(k)}=\mathbf{\Phi}^{\ast}.
    \end{equation}
    By leveraging the properties outlined in Equation (\ref{GOTTVVTVdedu}), we know $F(\mathbf{U}^{(k)})=G(\mathbf{\Phi}^{(k)})$, $F(\mathbf{U}^{\ast})=G(\mathbf{\Phi}^{\ast})$. Therefore, the conclusion is derived.
\end{proof}

In conclusion, the ADMM effectively decomposes the original minimization problem into two tractable subproblems associated with the parameters $\mathbf{\Phi}$ and $\mathbf{W}_{x(y)}$. During the iterative process, each subproblem is solved with the other parameters held constant, and the algorithm converges when the relative error of the iterative solution falls below a predefined threshold or when the maximum iteration limit is reached. These two subproblems can be solved efficiently using the fast Fourier transform and soft-thresholding operators, as shown in the paper \cite{wu2010augmented}. The primary factor of the computational complexity per iteration of the algorithm is the Fast Fourier Transforms (FFTs), which is $\mathcal{O}(dmn\log(mn))$ see \cite{ng1999fast}. Since we initially transformed the image into the space corresponding to opponent transformation: $B$, the final result should transform back into the original space, i.e., using the transpose of $B$ to multiply the iteration result. 

\section{Numerical Experiments}
\label{sec:Numerical Experiments}

Our simulated experiments take five representative objects from the Columbia multispectral image database \cite{yasuma2010generalized} to test. The database contains object information for many materials in the wavelength range 400nm to 700nm with equidistant 10nm steps (31 bands in total). To facilitate visual and comparative analysis of the image restoration performance among different methods, we present the results of each method's image restoration as pseudocolored images (R-26, G-20, B-4), as demonstrated in \cref{MSI for comparison}. We chose R-26, G-20, and B-4 as the pseudo-color image for display for two main reasons. Firstly, the human eye is more sensitive to color information than grayscale information, allowing for an intuitive assessment of the image recovery results. Secondly, the pseudo-color image's red, green, and blue tones are similar to the objects' actual red, green, and blue tones in the dataset \cite{yasuma2010generalized}. This choice was made to avoid introducing unnatural colors that could interfere with the viewer's judgment of the recovery results. Additionally, to understand the differences in information across different channels in the dataset, we calculated the channel-wise Frobenius norm for the images of cloth and face, as shown in \cref{Channel-wise Frobenius norm information for cloth and face}. As can be seen, the norms between different channels exhibit a linearly increasing trend and vary significantly. Indeed, this indicates that the signal intensity and structural information vary across different channels. To further validate the conclusion above and examine the differences in texture information among different channels, we showed part of the cloth and selected eight channels from its 31 channels, as shown in \cref{Partially enlarged view of cloth in different channels}. It can be observed that the first channel appears blurry, the middle channels have relatively similar information, and the later channels exhibit large areas of white, indicating more abundant texture information.

We adopt two image quality assessment indexes to help to judge the quality of denoising results. They are, respectively, the mean peak signal-to-noise ratio (MPSNR) and mean structural similarity (MSSIM), which are the average of the PSNR and SSIM \cite{wang2004image} results of each channel of the denoising results. The PSNR represents the mean square error of the difference between the clean image and the denoising result, so a higher PSNR means the denoising result is closer to the clean image and a higher denoising quality \cite{hore2010image}. The value of SSIM is between 0 to 1, and the closer to 1, the more similar the two images are under the human visual system \cite{hore2010image}. The methods we compare with are mentioned in the introduction. They are respectively band-by-band TV \cite{rudin1992nonlinear}\cite{tao2009alternating}, VTV \cite{wu2010augmented}\cite{tao2009alternating}, ASSTV \cite{chang2015anisotropic}, SSAHTV \cite{yuan2012hyperspectral}. For each model, we tested a wide range of regularization parameters and selected the result with the highest MPSNR value for presentation. Regarding the stopping criteria, we stop iteration when the relative error of successive iterations is less than $10^{-5}$. The proposed algorithm is implemented in MATLAB. 

\begin{figure}[H]
	\centering
 \subfigure[Cloth]{
	\begin{minipage}{0.15\linewidth}
		\centering
		\includegraphics[width=\linewidth]{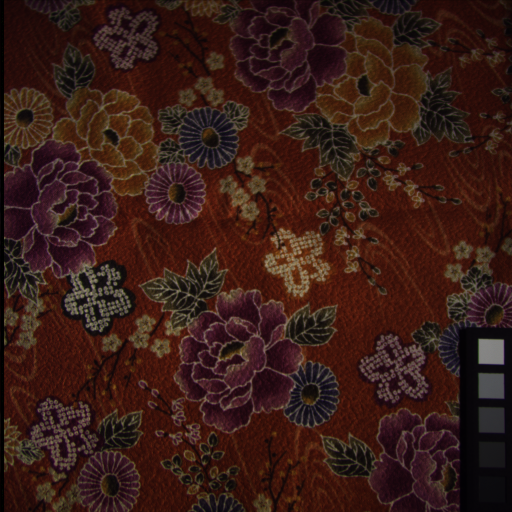}
	\end{minipage}}
 \subfigure[Face]{
 \begin{minipage}{0.15\linewidth}
		\centering
		\includegraphics[width=\linewidth]{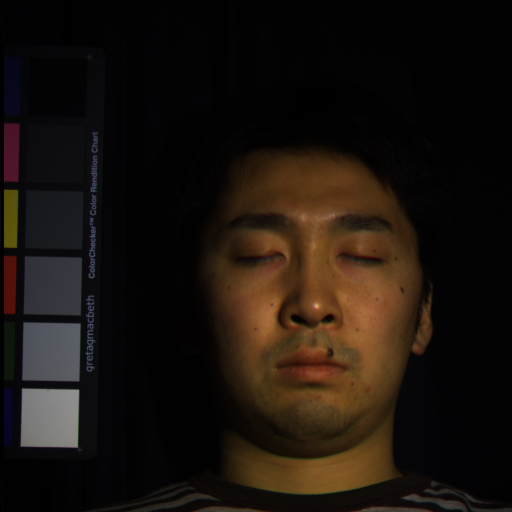}
	\end{minipage}}
 \subfigure[Jelly]{
 \begin{minipage}{0.15\linewidth}
		\centering
		\includegraphics[width=\linewidth]{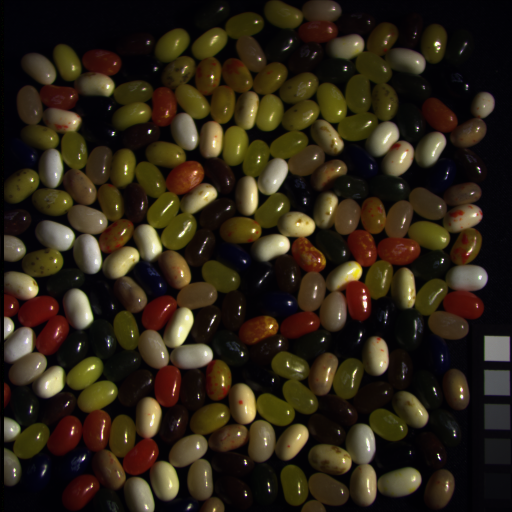}
	\end{minipage}}
 \subfigure[Picture]{
 \begin{minipage}{0.15\linewidth}
		\centering
		\includegraphics[width=\linewidth]{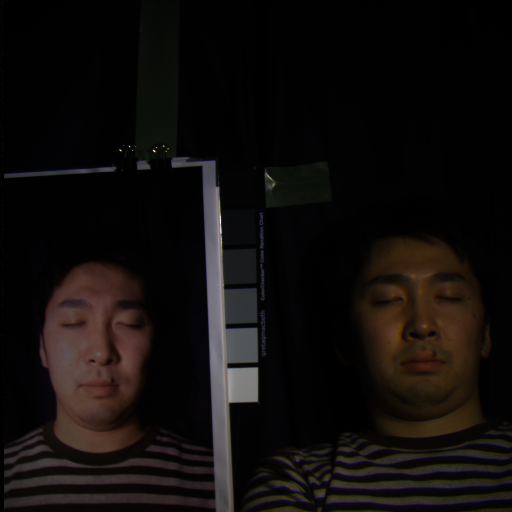}
	\end{minipage}}   
 \subfigure[Thread]{
 \begin{minipage}{0.15\linewidth}
		\centering
		\includegraphics[width=\linewidth]{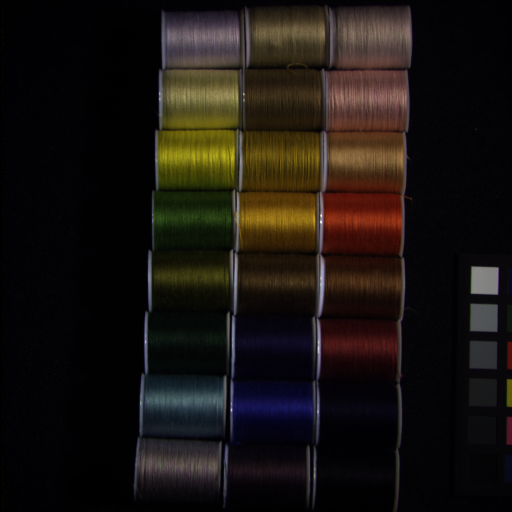}
	\end{minipage}}
 \caption{Pseudo-color images (R-26, G-20, B-4) of objects for comparison.}
  \label{MSI for comparison}
\end{figure}

\begin{figure}[H]
	\centering
 \subfigure[Cloth]{
	\begin{minipage}{0.35\linewidth}
		\centering
		\includegraphics[width=\linewidth]{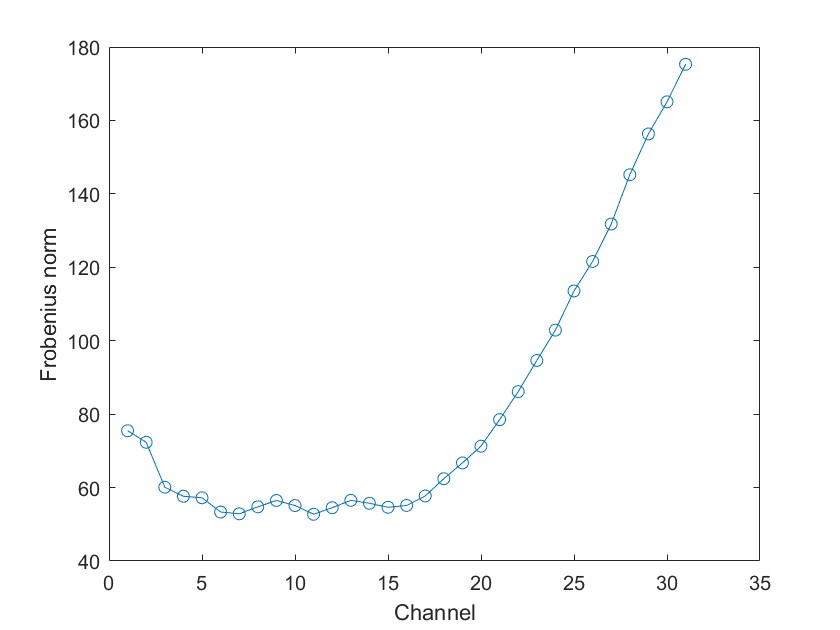}

	\end{minipage}}
 \subfigure[Face]{
	\begin{minipage}{0.35\linewidth}
		\centering
		\includegraphics[width=\linewidth]{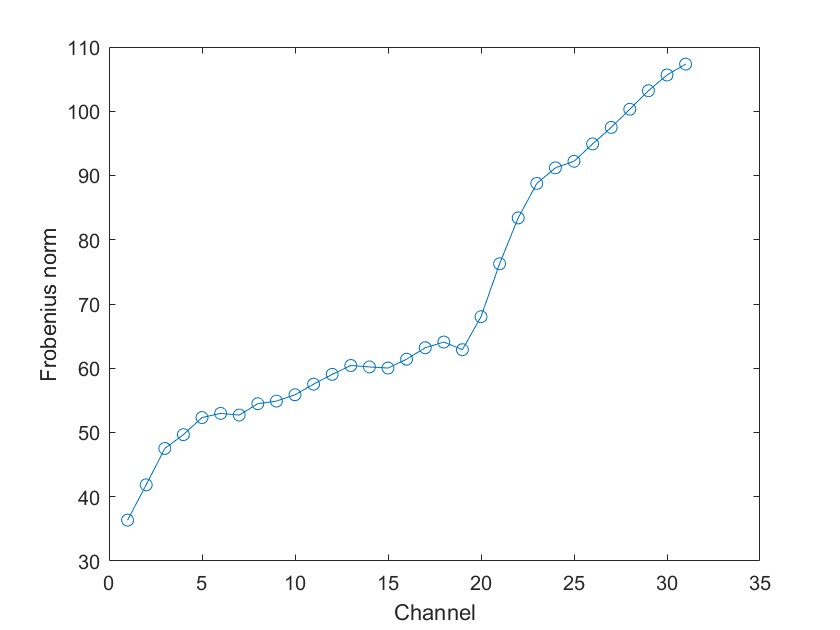}

	\end{minipage}}

 \caption{Channel-wise Frobenius norm information for cloth and face.}
  \label{Channel-wise Frobenius norm information for cloth and face}
\end{figure}

\begin{figure}[H]
	\centering
 \subfigure[Channel 1]{
	\begin{minipage}{0.2\linewidth}
		\centering
		\includegraphics[width=\linewidth]{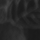}

	\end{minipage}}
 \subfigure[Channel 5]{
	\begin{minipage}{0.2\linewidth}
		\centering
		\includegraphics[width=\linewidth]{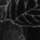}

	\end{minipage}}
 \subfigure[Channel 9]{
	\begin{minipage}{0.2\linewidth}
		\centering
		\includegraphics[width=\linewidth]{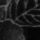}

	\end{minipage}}
  \subfigure[Channel 13]{
	\begin{minipage}{0.2\linewidth}
		\centering
		\includegraphics[width=\linewidth]{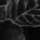}

	\end{minipage}}
  \subfigure[Channel 17]{
	\begin{minipage}{0.2\linewidth}
		\centering
		\includegraphics[width=\linewidth]{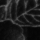}

	\end{minipage}}
  \subfigure[Channel 21]{
	\begin{minipage}{0.2\linewidth}
		\centering
		\includegraphics[width=\linewidth]{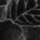}

	\end{minipage}}
  \subfigure[Channel 25]{
	\begin{minipage}{0.2\linewidth}
		\centering
		\includegraphics[width=\linewidth]{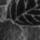}

	\end{minipage}}
  \subfigure[Channel 29]{
	\begin{minipage}{0.2\linewidth}
		\centering
		\includegraphics[width=\linewidth]{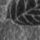}

	\end{minipage}}

 \caption{Partial view of cloth in different channels.}
  \label{Partially enlarged view of cloth in different channels}
\end{figure}

\subsection{Image denoising}

Our simulation experiment tested two different levels of Gaussian white noise. The Gaussian noise with standard deviations of 0.05 and 0.1 was added to each band in the ground-truth multispectral images. For each method, the denoising results with the highest MPSNR achieved after adjusting the regularization parameters are listed in \cref{Denoising Comparison results}. The bold number in the table means it is the best in its row, and the underlined number represents the second-best data. The result shows that no matter the MPSNR, the MSSIM, or time (second), GOTTV performs well. Regarding the setting of our parameters, we fixed initial r as 0.01, $r_{max}= 10^{6}, \rho=1.8$, the maximum iteration is $10^{4}$, and the other two parameters are shown in \cref{Denoising Parameter setting}. \cref{PSNR comparison for cloth,SSIM comparison for cloth,PSNR comparison for face,SSIM comparison for face} showcase the comparative results of denoising performance in terms of PSNR and SSIM metrics across different channels for various methods. From these figures, we can draw two conclusions. First, GOTTV outperforms the other methods in the majority of channels in terms of both PSNR and SSIM. Second, different methods exhibit variations in the recovery results across different channels. There are two main reasons for the differences in the recovery results of different channels. First, the differences in the Frobenius norms between channels indicate variations in signal intensity, image details, and structural information across different channels. These differences can lead to variations in the performance of recovery methods on different channels. Second, the regularization parameters influence the recovery results in each channel. Our chosen parameters are adjusted to maximize the average PSNR across all channels. This means that parameter selection is based on the overall performance rather than individual channel performance. Therefore, the final image recovery result prioritizes the overall performance rather than precise optimization for each channel. These reasons contribute to the differences in the recovery results of different channels.

We can understand the overall performance of different methods from the above results. Next, we will visually display the recovery of details and textures of different methods by zooming in the following pseudo-color \cref{First Part of Cloth,Second Part of Cloth,Second Part of face,First Part of face}. The image located within the large green boxes at the four corners is an enlarged view of the image within the smaller green boxes. From \cref{First Part of Cloth}, obviously, only our GOTTV method effectively eliminates noise, whereas denoising outcomes from alternative methods exhibit green and red unnaturally stains, which are not present in the Ground-truth image. Also, it can be observed that only GOTTV effectively preserves the veins of the leaves, while the other methods result in the loss of this texture in the denoised outcome. From \cref{Second Part of Cloth}, it is evident that only GOTTV is capable of accurately separating each white circle without introducing any color artifacts. The other methods exhibit unnatural red-green speckles around the white circles. From the results shown in \cref{First Part of face}, only GOTTV can fully recover the three geometric lines on the clothing. The denoising outcomes of the remaining methods fail to separate the two uppermost white stripes. Finally, let's focus on the shadow area between the forehead and the eye in \cref{Second Part of face}. It can be observed that GOTTV effectively restores the facial features as well as the shadow, while other methods exhibit red-green artifacts in the shadow and forehead region. These comparisons highlight the superior denoising performance of GOTTV in terms of both noise reduction and preservation of texture and details.

Our real data experiment tests three widely used datasets corrupted by real noise: Jasper Ridge, Pavia Centre, and Pavia University, which can be downloaded from the website: https://rslab.ut.ac.ir/data. The data indicates evident noise contamination in channels 220-224 of Jasper Ridge and channels 1-6 of Pavia Centre and Pavia University. Therefore, denoising primarily targeted these channels. Given the absence of ground truth in real images, we employed the Q-metric \cite{zhu2010automatic} to compare denoising results, which is a blind image content measurement index widely utilized by recent researchers. A higher Q-metric value normally indicates better image quality. The \cref{Blind denoising result} below presents the MQ-metric of different methods, which is the average Q-metric for each channel. It is evident that our method achieves the highest or the second-highest Q-metric values. 

To better visualize the denoising results, we selected three representative channels for each image to form pseudo-color images (channels 220-222 for Jasper Ridge and 3-5 for the others), as shown in \cref{Jasper Ridge,Pavia Center,Pavia University}. From \cref{Jasper Ridge}, it can be observed that although the MQ-metric value of GOTTV is slightly lower than that of ASSTV, the denoising results of ASSTV do not effectively preserve smooth curves. This leads to instances where curves become predominantly vertical. Meanwhile, the denoising outcomes of other methods do not successfully eliminate noise, as evidenced by the shimmering phenomenon in the results. Our GOTTV method excels in both detail preservation and denoising, showcasing outstanding performance. From \cref{Pavia Center}, it is evident that the denoising results of TV, VTV, and SSAHTV fail to preserve the topographical information of the cultivated land, resulting in the loss of texture. Moreover, as depicted in \cref{Pavia Center,Pavia University}, although ASSTV can retain the stripe information, the boundaries of the stripes appear coarse and jagged. In contrast, among these methods, only the GOTTV approach not only effectively preserves texture information but also produces smoothly recovered results.

\begin{table}[h]
\caption{Parameter setting.}
\label{Denoising Parameter setting}
\centering
\begin{tabular}{|l|l|l|l|l|l|l|}
\hline
Std                   & Parameter & Cloth   & Face  & Jelly & Picture & Thread \\ \hline
\multirow{2}{*}{0.05} & $\lambda$         & 6.128  & 4.992   & 5.616   & 5.385     & 5.2    \\ \cline{2-7} 
                      & $\alpha$         & 0.129 & 0.206 & 0.197 & 0.216   & 0.157  \\ \hline
\multirow{2}{*}{0.1}  & $\lambda$         & 2.744     & 2.346   & 2.678   & 2.508    & 2.435   \\ \cline{2-7} 
                      & $\alpha$         & 0.139   & 0.212 & 0.208   & 0.216   & 0.174   \\ \hline
\end{tabular}
\end{table}

\begin{table}[htbp]
\centering
\label{Denoising Comparison results}
\caption{Denoising Comparison results.}
\scalebox{0.9}{
\begin{tabular}{|l|l|l|l|l|l|l|l|l|} 
\hline
Figure                   & Std                   & Measure & Degraded image & TV      & VTV     & ASSTV           & SSAHTV          & GOTTV             \\ 
\hline
\multirow{6}{*}{Cloth}   & \multirow{3}{*}{0.05} & MPSNR   & 26.0204     & 32.2596 & 33.0256 & 33.2623         & \underline{33.3446} & \textbf{35.7352}  \\ 
\cline{3-9}
                         &                       & MSSIM   & 0.5413      & 0.8448  & 0.8643  & \underline{0.8795}  & 0.8734          & \textbf{0.9187}   \\ 
\cline{3-9}
                         &                       & Time    & -          & 156.08 & 85.72 & 658.7 &\underline{15.68} &\textbf{12.19} \\ 
\cline{2-9}
                         & \multirow{3}{*}{0.1}  & MPSNR   & 19.9974     & 29.2079 & 30.0321 & 29.4339         & \underline{30.1990} & \textbf{33.2734}  \\ 
\cline{3-9}
                         &                       & MSSIM   & 0.2688      & 0.7226  & 0.7622  & 0.7395          & \underline{0.7714}  & \textbf{0.8743}   \\ 
\cline{3-9}              &                       & Time    & -          & 204.07  & 142.97 &654.39 & \underline{30.68} & \textbf{12.64}\\
                         
\hline
\multirow{6}{*}{Face} & \multirow{3}{*}{0.05} & MPSNR   & 26.0189     & 40.2321 & 41.444 & {41.4194} & \underline{41.8420}         & \textbf{43.1765}  \\ 
\cline{3-9}
                         &                       & MSSIM   & 0.2963      & 0.9558  & 0.9687  & 0.9637          & \underline{0.9706}  & \textbf{0.9784}   \\ 
\cline{3-9}
                         &                       & Time    & -          & 347.6 & 96 & 592.99 & \underline{43.11} & \textbf{13.14} \\
\cline{2-9}
                         & \multirow{3}{*}{0.1}  & MPSNR   & 19.9963     & 36.9756 & 38.3222 & {37.9545} & \underline{38.5904}         & \textbf{40.4463}  \\ 
\cline{3-9}
                         &                       & MSSIM   & 0.0985      & 0.9219  & 0.9467  & {0.9393}  & \underline{0.9487}          & \textbf{0.9636}   \\ 
\cline{3-9}             &                       & Time    & -            & 407.77 & 114.74 & 545.52 & \underline{54.22} & \textbf{13.07}\\
                         
\hline
\multirow{6}{*}{Jelly}   & \multirow{3}{*}{0.05} & MPSNR   & 26.0250     & 34.8628 & 35.9847 & {36.2345} & \underline{36.2948}         & \textbf{37.8589}  \\ 
\cline{3-9}
                         &                       & MSSIM   & 0.4801      & 0.9296  & 0.9424  & 0.9410          & \underline{0.9463}  & \textbf{0.9532}   \\ 
 \cline{3-9}
                        &                       &Time     & -           &198.53 &71.18&676.39&\underline{26.68}&\textbf{12.21}\\

\cline{2-9}

                         & \multirow{3}{*}{0.1}  & MPSNR   & 19.9995     & 31.0173 & 32.3106 & 31.5751         & \underline{32.6626} & \textbf{34.5847}  \\ 
\cline{3-9}
                         &                       & MSSIM   & 0.2456      & 0.8622  & 0.8906  & 0.8836          & \underline{0.8981}  & \textbf{0.912}   \\ 
\cline{3-9}
                        &                        &Time & - &222.87&127.91&646.85&\underline{34.57}&\textbf{13.2}\\

\hline

\multirow{6}{*}{Picture}    & \multirow{3}{*}{0.05} & MPSNR   & 26.0182     & 39.0883 & 40.3065 & \underline{40.6019}         & {40.5127} & \textbf{42.0049}  \\ 
\cline{3-9}
                         &                       & MSSIM   & 0.3095      & 0.9405  & 0.9583  & 0.9605          & \underline{0.9616}  & \textbf{0.9659}   \\ 
\cline{3-9}             &                       &Time & -&369.51 &78.69 & 655.35 &\underline{35.96} & \textbf{19.9}\\

\cline{2-9}
                         & \multirow{3}{*}{0.1}  & MPSNR   & 19.9977    & 35.4922 & 36.8778 & \underline{37.1563}         & {37.0249} & \textbf{39.2127}  \\ 
\cline{3-9}

                         &                       & MSSIM   & 0.1101      & 0.8793  & 0.9151  & \underline{0.9296}          & {0.9186}  & \textbf{0.9403}   \\ 
\cline{3-9}
                        &                        & Time & - & 418.93 & 110.07 & 630.46 & \underline{43.5}& \textbf{19.75}\\

\hline
\multirow{6}{*}{Thread}  & \multirow{3}{*}{0.05} & MPSNR   & 26.0202     & 37.0093 & 38.0824 & 38.3433         & \underline{38.4523} & \textbf{40.7385}  \\ 
\cline{3-9}
                         &                       & MSSIM   & 0.3428      & 0.9228  & 0.9396  & {0.9452}  & \underline{0.9459}          & \textbf{0.9601}   \\ 
\cline{3-9}
                         &                      & Time & - & 300.81 & 61.76 & 487.15 & \underline{33.39}& \textbf{19.98}\\
\cline{2-9}
                         & \multirow{3}{*}{0.1}  & MPSNR   & 20.0040     & 33.7676 & 35.0275 & 34.1843         & \underline{35.3306} & \textbf{37.958}  \\ 
\cline{3-9}

                         &                       & MSSIM   & 0.1341      & 0.868  & 0.8995  & 0.9056          & \underline{0.9073}  & \textbf{0.9336}   \\
\cline{3-9}
                        &                       & Time & - &344.34& 97.36& 479.05& \underline{42.3} & \textbf{19.51}\\
                         
\hline
\multirow{6}{*}{Dataset}  & \multirow{3}{*}{0.05} & MPSNR & 26.0209 & 38.7435     & 39.7296 & 39.4922 & \underline{40.1250}    & \textbf{41.5407}  \\ 
\cline{3-9}
                         &                       & MSSIM   &  0.3432     & 0.9416  & 0.9518  & 0.9522  & \underline{0.9556}          & \textbf{0.9653}   \\ 
\cline{3-9}
                         &                      & Time & - & 350.47 & 90.08 & 593.64 & \underline{39.47}& \textbf{15.28}\\
\cline{2-9}
                         & \multirow{3}{*}{0.1}  & MPSNR   &   20.0001   & 35.4611 & 36.5796 & 35.4822         & \underline{36.9093} & \textbf{38.6550}  \\ 
\cline{3-9}

                         &                       & MSSIM   &   0.1355    & 0.9002  & 0.9191  & 0.9136          & \underline{0.9235}  & \textbf{0.9416}   \\
\cline{3-9}
                        &                       & Time & - &414.99& 119.21& 575.91& \underline{50.58} & \textbf{15.68}\\
                         
\hline
\end{tabular}
}


\end{table}

{
\begin{figure}[H]
	\centering
 \subfigure[Std=0.05]{
	\begin{minipage}{0.48\linewidth}
		\centering
		\includegraphics[width=\linewidth]{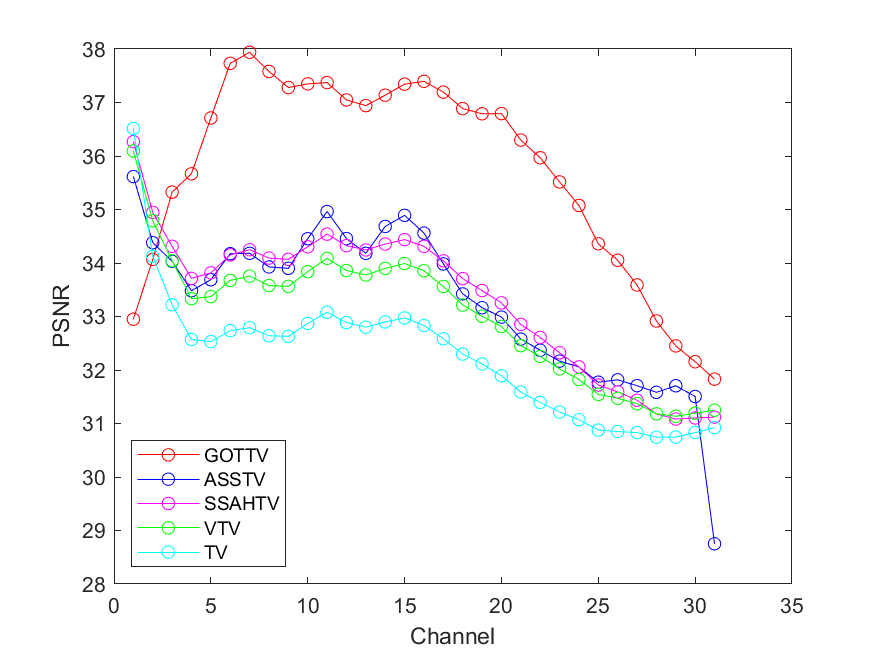}
	\end{minipage}}
 \subfigure[Std=0.1]{
	\begin{minipage}{0.48\linewidth}
		\centering
		\includegraphics[width=\linewidth]{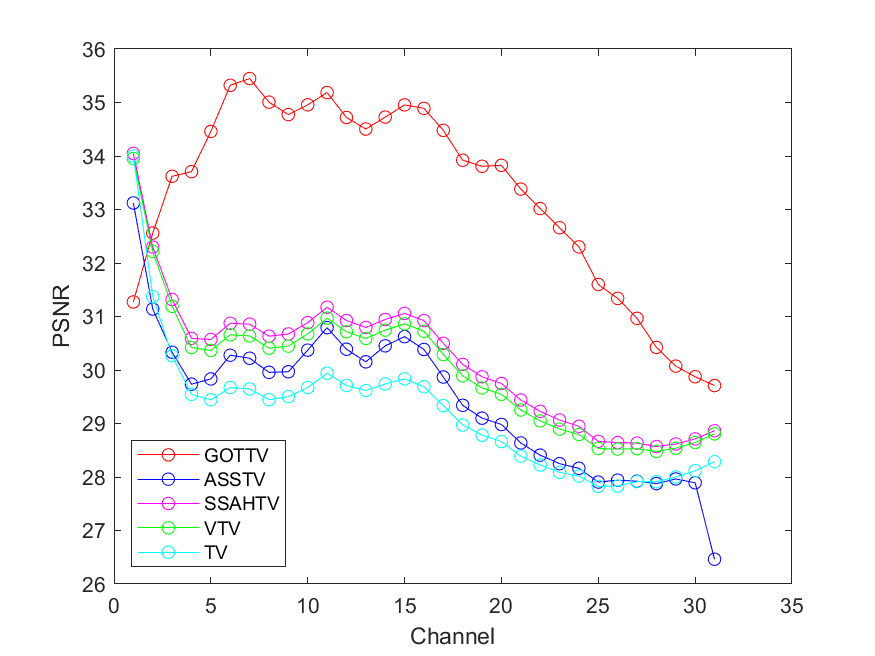}
	\end{minipage}}

   \caption{Spectrum-by-spectrum PSNR comparison chart of different methods for cloth.}
  \label{PSNR comparison for cloth}
\end{figure}
}

{
\begin{figure}[H]
	\centering
 \subfigure[Std=0.05]{
	\begin{minipage}{0.48\linewidth}
		\centering
		\includegraphics[width=\linewidth]{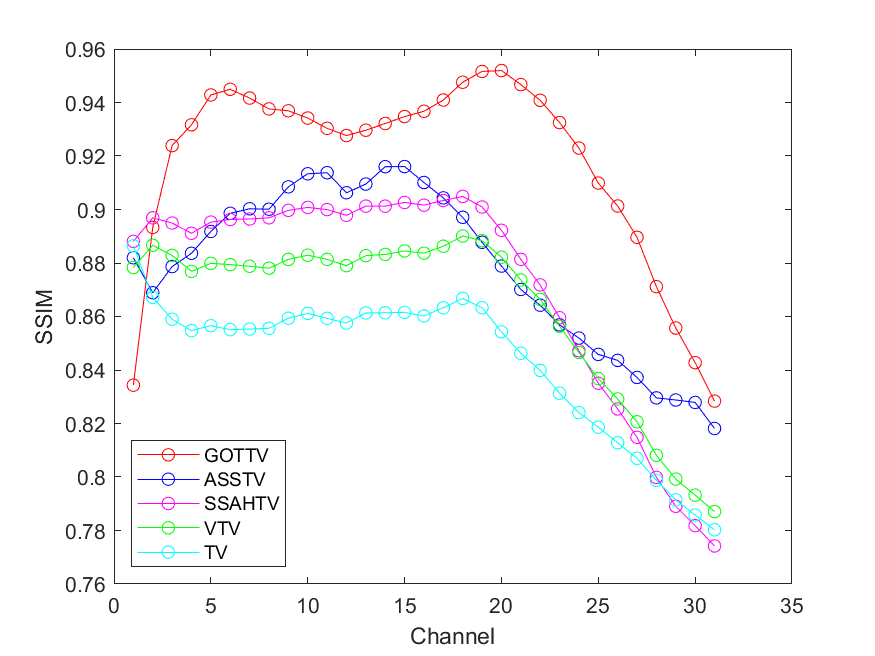}
	\end{minipage}}
 \subfigure[Std=0.1]{
	\begin{minipage}{0.48\linewidth}
		\centering
		\includegraphics[width=\linewidth]{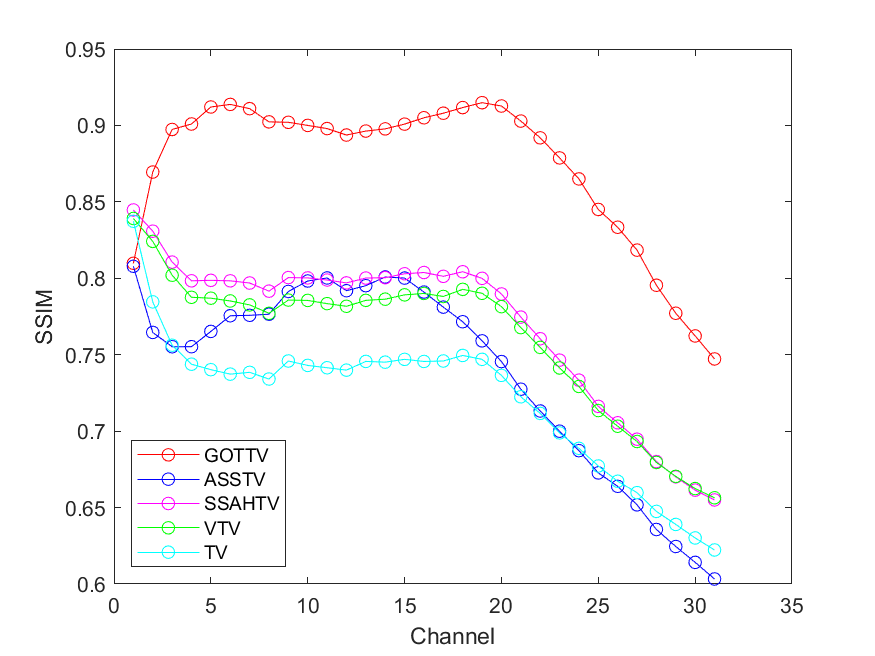}
	\end{minipage}}

   \caption{Spectrum-by-spectrum SSIM comparison chart of different methods for cloth.}
  \label{SSIM comparison for cloth}
\end{figure}
}

{
\begin{figure}[H]
	\centering
 \subfigure[Std=0.05]{
	\begin{minipage}{0.48\linewidth}
		\centering
		\includegraphics[width=\linewidth]{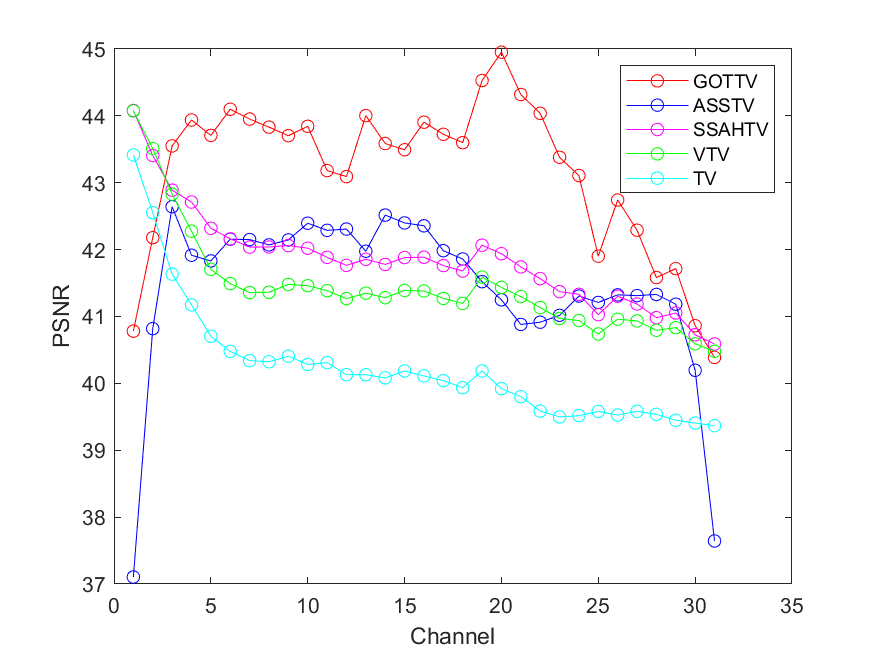}
	\end{minipage}}
 \subfigure[Std=0.1]{
	\begin{minipage}{0.48\linewidth}
		\centering
		\includegraphics[width=\linewidth]{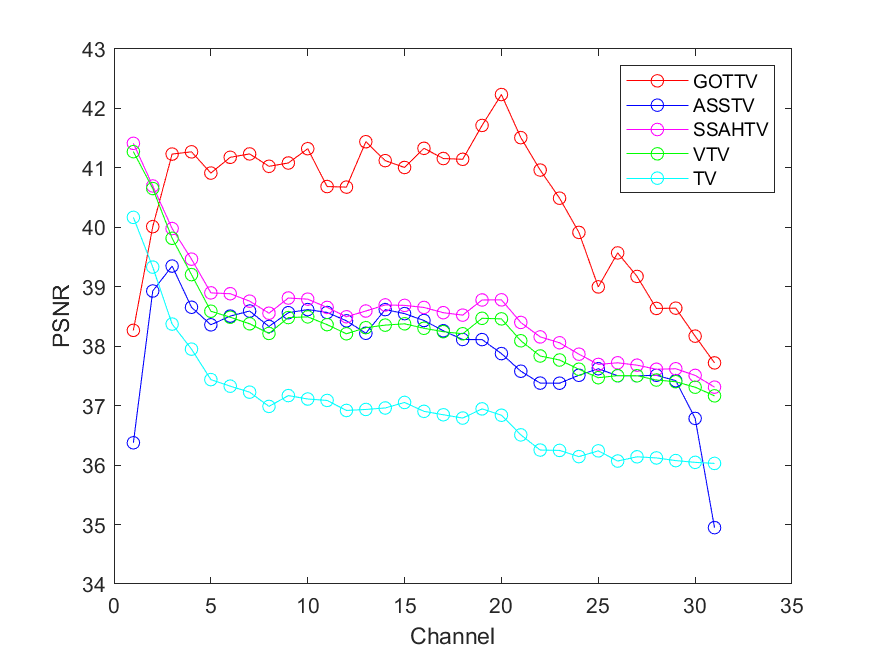}
	\end{minipage}}

   \caption{Spectrum-by-spectrum PSNR comparison chart of different methods for face.}
  \label{PSNR comparison for face}
\end{figure}
}

{
\begin{figure}[H]
	\centering
 \subfigure[Std=0.05]{
	\begin{minipage}{0.48\linewidth}
		\centering
		\includegraphics[width=\linewidth]{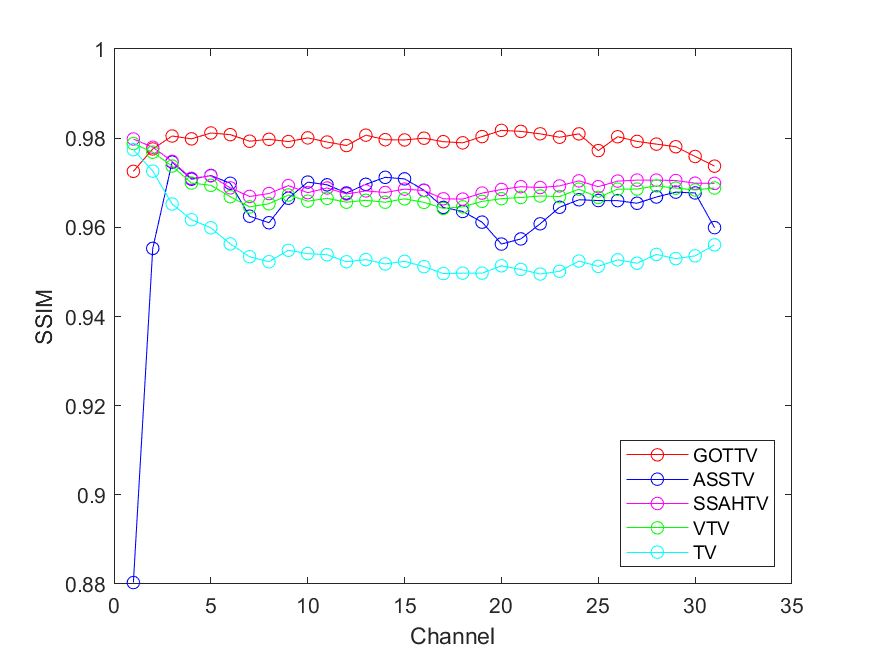}
	\end{minipage}}
 \subfigure[Std=0.1]{
	\begin{minipage}{0.48\linewidth}
		\centering
		\includegraphics[width=\linewidth]{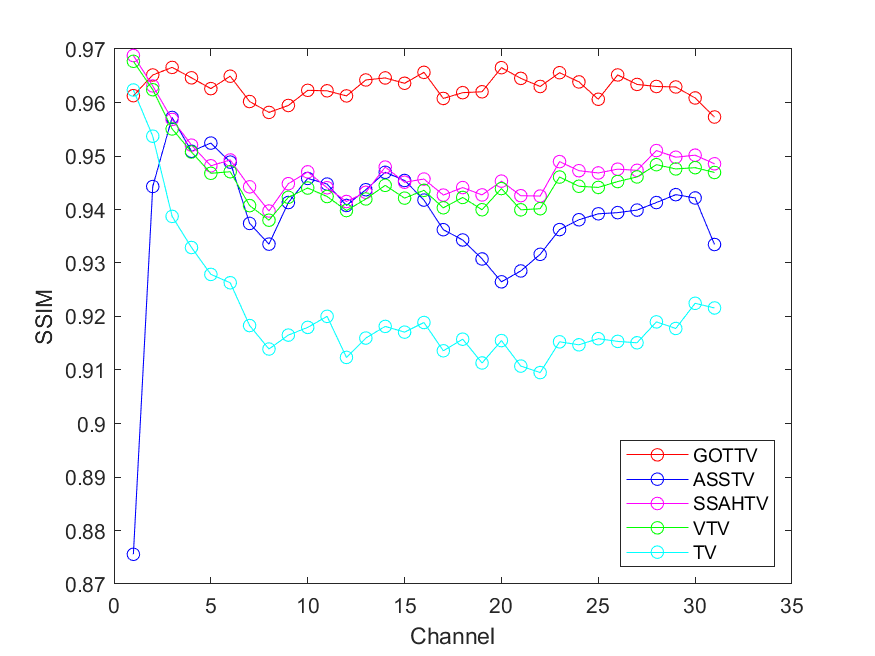}
	\end{minipage}}

   \caption{Spectrum-by-spectrum SSIM comparison chart of different methods for face.}
  \label{SSIM comparison for face}
\end{figure}
}
{
\begin{figure}[H]
		\centering
 \subfigure[Ground-truth]{
	\begin{minipage}{0.2\linewidth}
		\centering
		\includegraphics[width=\linewidth]{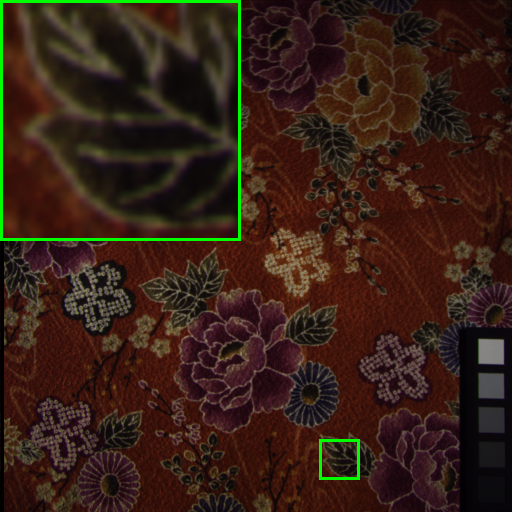}
	\end{minipage}}
 \subfigure[Degraded]{
	\begin{minipage}{0.2\linewidth}
		\centering
		\includegraphics[width=\linewidth]{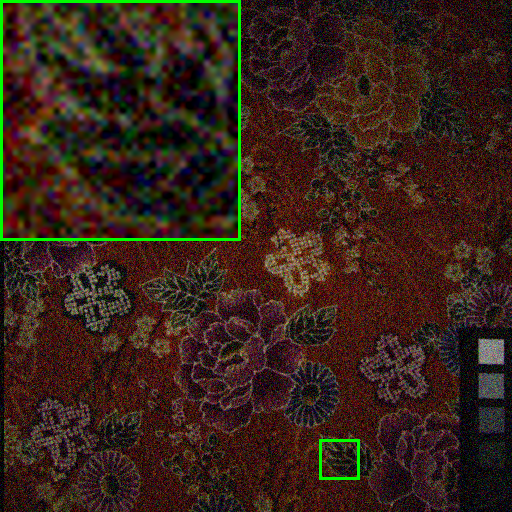}
	\end{minipage}}
 \subfigure[TV]{
	\begin{minipage}{0.2\linewidth}
		\centering
		\includegraphics[width=\linewidth]{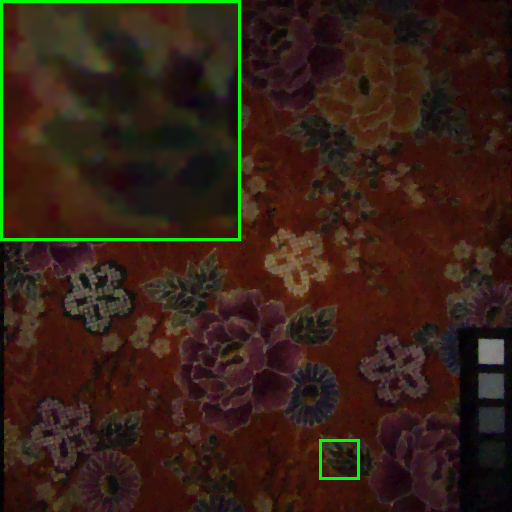}
	\end{minipage}}
  \subfigure[VTV]{
	\begin{minipage}{0.2\linewidth}
		\centering
		\includegraphics[width=\linewidth]{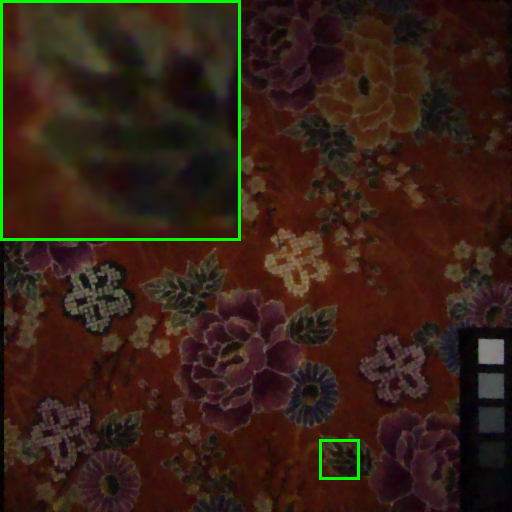}
	\end{minipage}}
   \subfigure[ASSTV]{
	\begin{minipage}{0.2\linewidth}
		\centering
		\includegraphics[width=\linewidth]{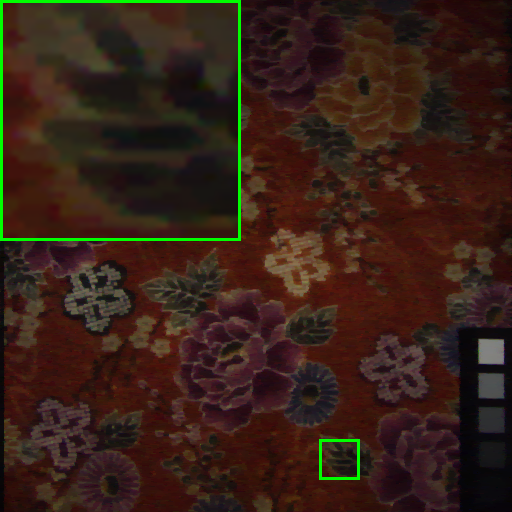}
	\end{minipage}}
    \subfigure[SSAHTV]{
	\begin{minipage}{0.2\linewidth}
		\centering
		\includegraphics[width=\linewidth]{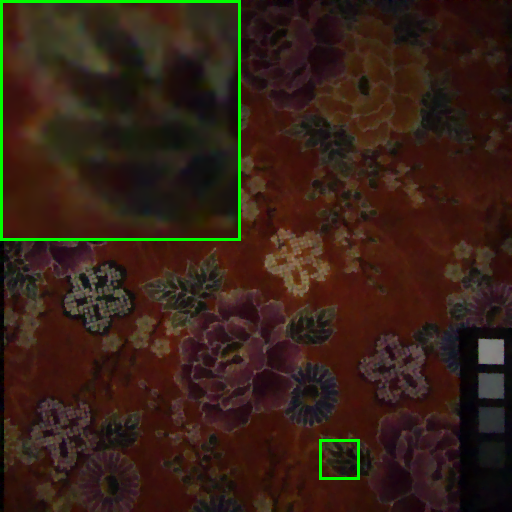}
	\end{minipage}}
    \subfigure[GOTTV]{
	\begin{minipage}{0.2\linewidth}
		\centering
		\includegraphics[width=\linewidth]{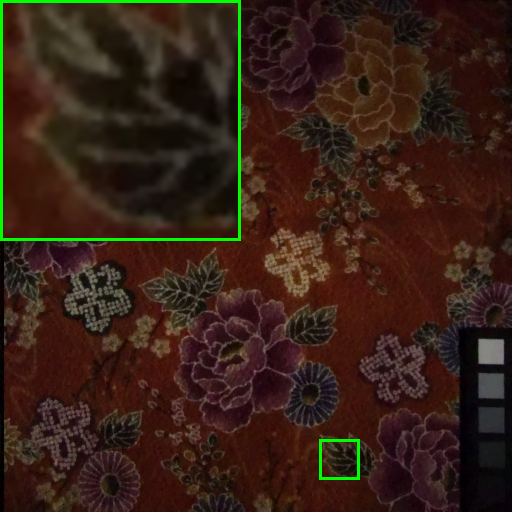}
	\end{minipage}}
 \caption{Comparison denoising results of various methods of the cloth image.}
  \label{First Part of Cloth}
\end{figure}
}

{
\begin{figure}[H]
		\centering
 \subfigure[Ground-truth]{
	\begin{minipage}{0.2\linewidth}
		\centering
		\includegraphics[width=\linewidth]{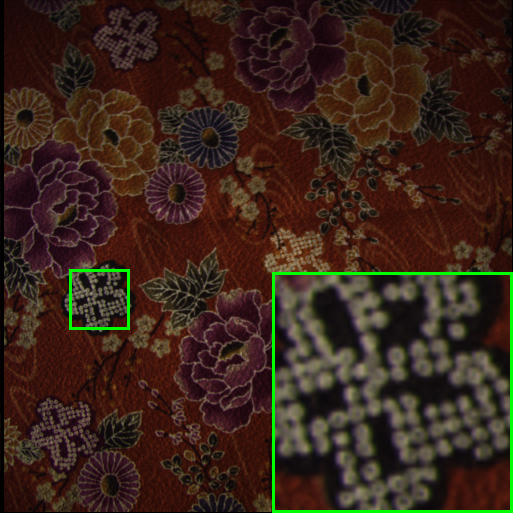}
	\end{minipage}}
 \subfigure[Degraded]{
	\begin{minipage}{0.2\linewidth}
		\centering
		\includegraphics[width=\linewidth]{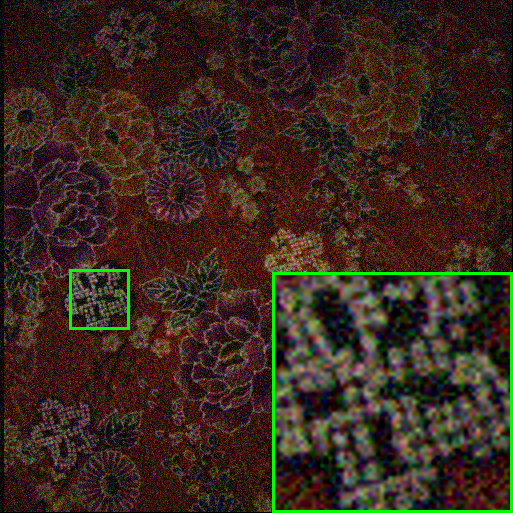}
	\end{minipage}}
 \subfigure[TV]{
	\begin{minipage}{0.2\linewidth}
		\centering
		\includegraphics[width=\linewidth]{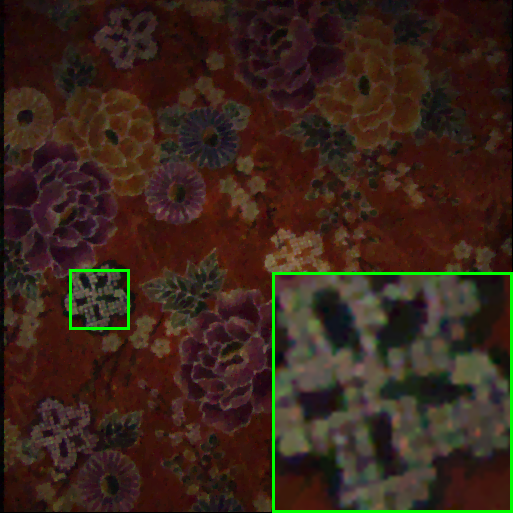}
	\end{minipage}}
  \subfigure[VTV]{
	\begin{minipage}{0.2\linewidth}
		\centering
		\includegraphics[width=\linewidth]{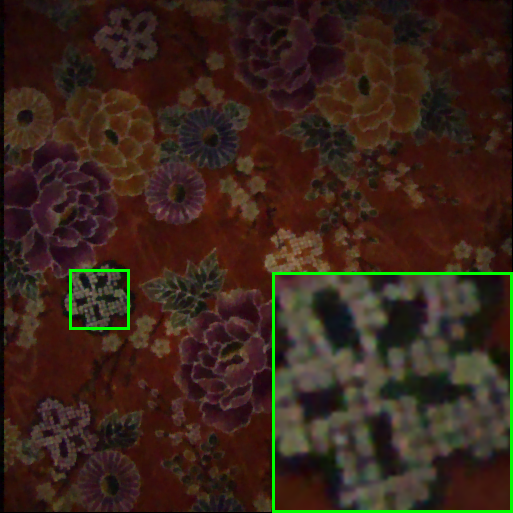}
	\end{minipage}}
   \subfigure[ASSTV]{
	\begin{minipage}{0.2\linewidth}
		\centering
		\includegraphics[width=\linewidth]{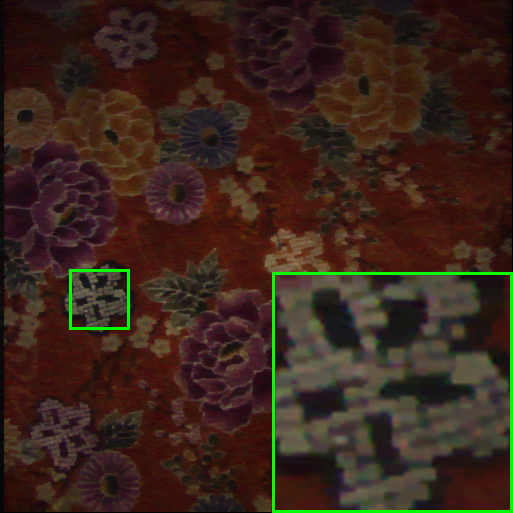}
	\end{minipage}}
    \subfigure[SSAHTV]{
	\begin{minipage}{0.2\linewidth}
		\centering
		\includegraphics[width=\linewidth]{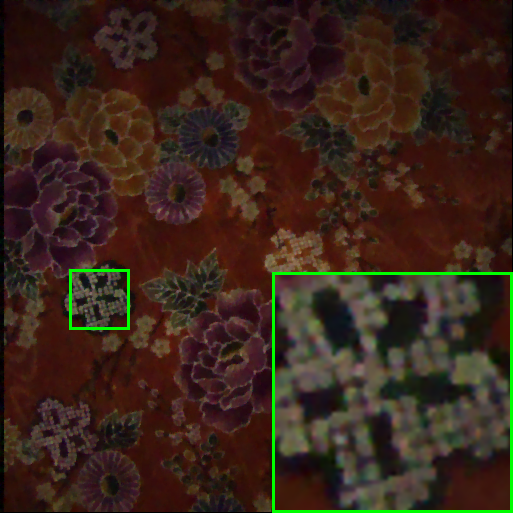}
	\end{minipage}}
    \subfigure[GOTTV]{
	\begin{minipage}{0.2\linewidth}
		\centering
		\includegraphics[width=\linewidth]{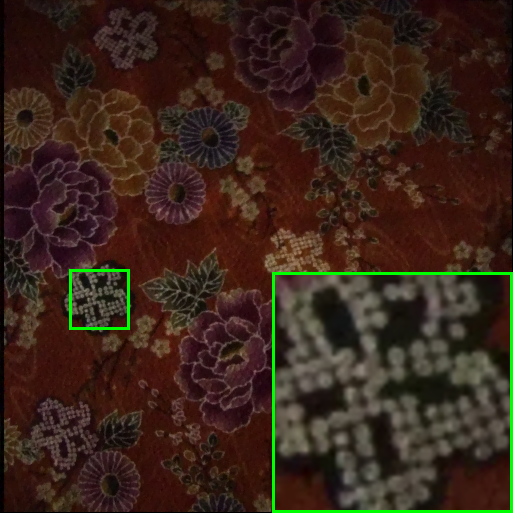}
	\end{minipage}}
 \caption{Comparison denoising results of various methods of the cloth image.}
  \label{Second Part of Cloth}
\end{figure}
}

{
\begin{figure}[H]
		\centering
 \subfigure[Ground-truth]{
	\begin{minipage}{0.2\linewidth}
		\centering
		\includegraphics[width=\linewidth]{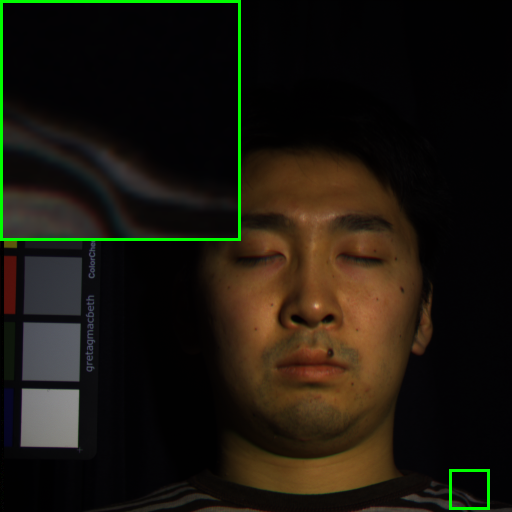}
	\end{minipage}}
 \subfigure[Degraded]{
	\begin{minipage}{0.2\linewidth}
		\centering
		\includegraphics[width=\linewidth]{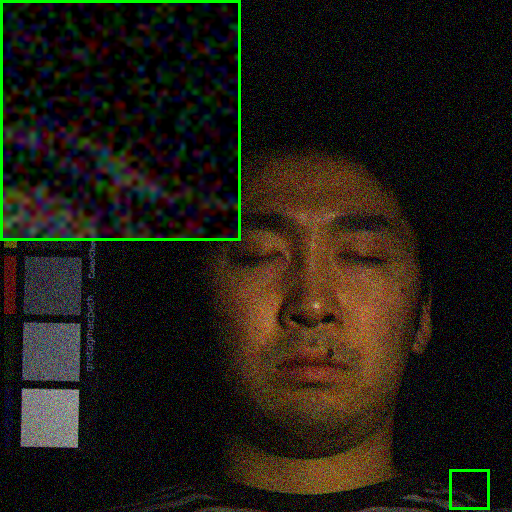}
	\end{minipage}}
 \subfigure[TV]{
	\begin{minipage}{0.2\linewidth}
		\centering
		\includegraphics[width=\linewidth]{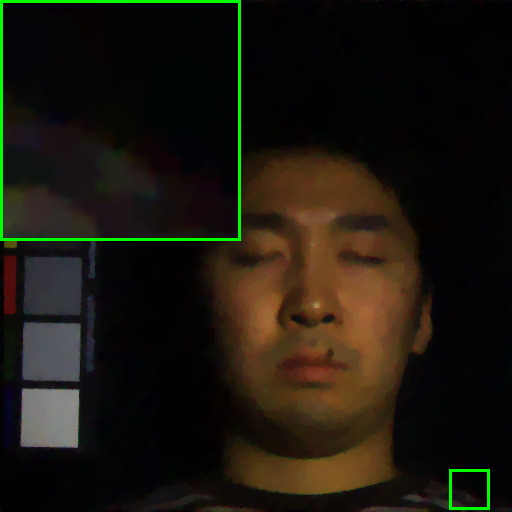}
	\end{minipage}}
  \subfigure[VTV]{
	\begin{minipage}{0.2\linewidth}
		\centering
		\includegraphics[width=\linewidth]{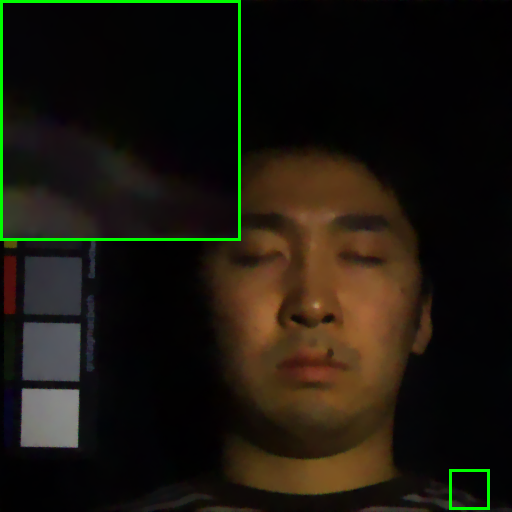}
	\end{minipage}}
   \subfigure[ASSTV]{
	\begin{minipage}{0.2\linewidth}
		\centering
		\includegraphics[width=\linewidth]{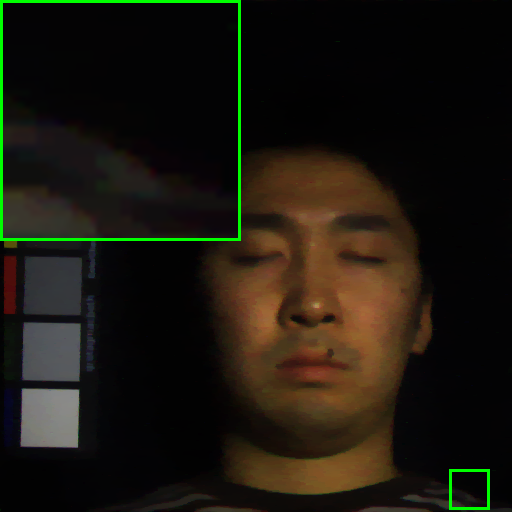}
	\end{minipage}}
    \subfigure[SSAHTV]{
	\begin{minipage}{0.2\linewidth}
		\centering
		\includegraphics[width=\linewidth]{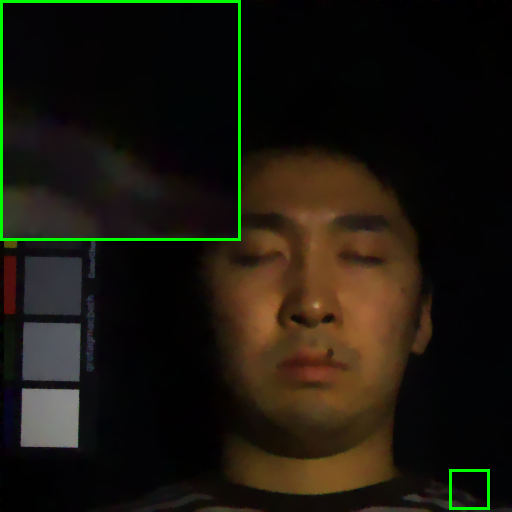}
	\end{minipage}}
    \subfigure[GOTTV]{
	\begin{minipage}{0.2\linewidth}
		\centering
		\includegraphics[width=\linewidth]{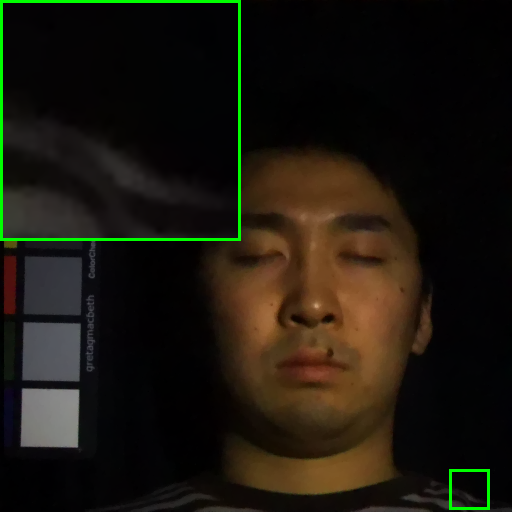}
	\end{minipage}}
 \caption{Comparison denoising results of various methods of the face image.}
  \label{First Part of face}
\end{figure}
}

{
\begin{figure}[H]
		\centering
 \subfigure[Ground-truth]{
	\begin{minipage}{0.2\linewidth}
		\centering
		\includegraphics[width=\linewidth]{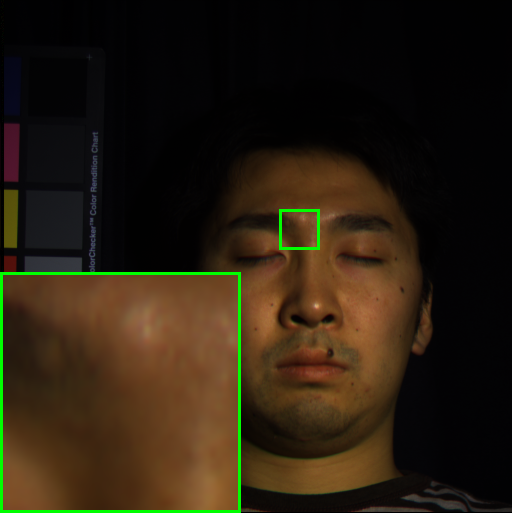}
	\end{minipage}}
 \subfigure[Degraded]{
	\begin{minipage}{0.2\linewidth}
		\centering
		\includegraphics[width=\linewidth]{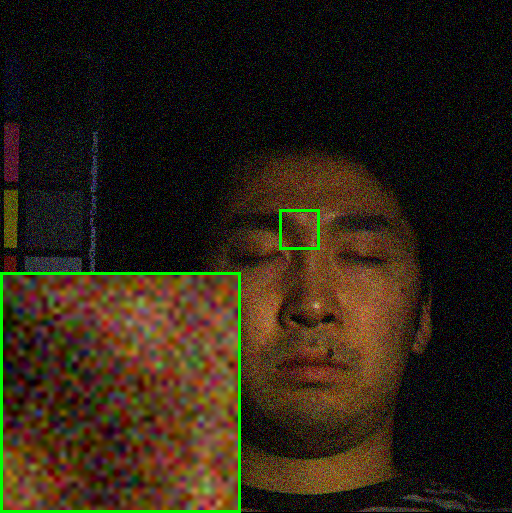}
	\end{minipage}}
 \subfigure[TV]{
	\begin{minipage}{0.2\linewidth}
		\centering
		\includegraphics[width=\linewidth]{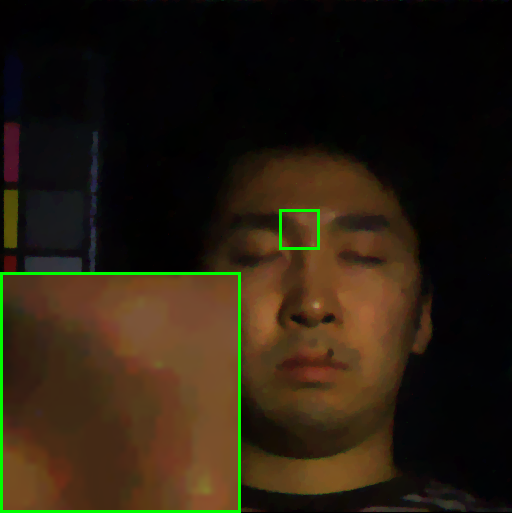}
	\end{minipage}}
  \subfigure[VTV]{
	\begin{minipage}{0.2\linewidth}
		\centering
		\includegraphics[width=\linewidth]{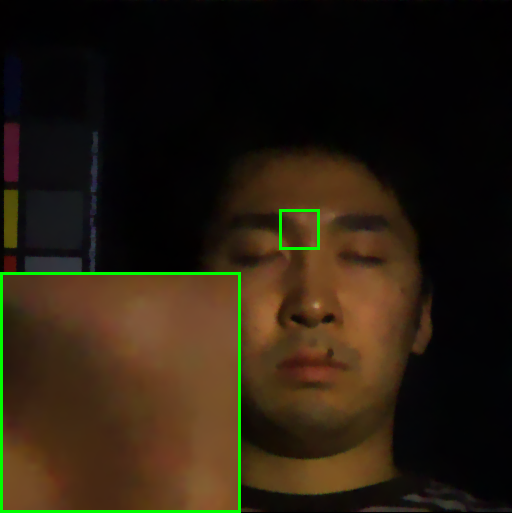}
	\end{minipage}}
   \subfigure[ASSTV]{
	\begin{minipage}{0.2\linewidth}
		\centering
		\includegraphics[width=\linewidth]{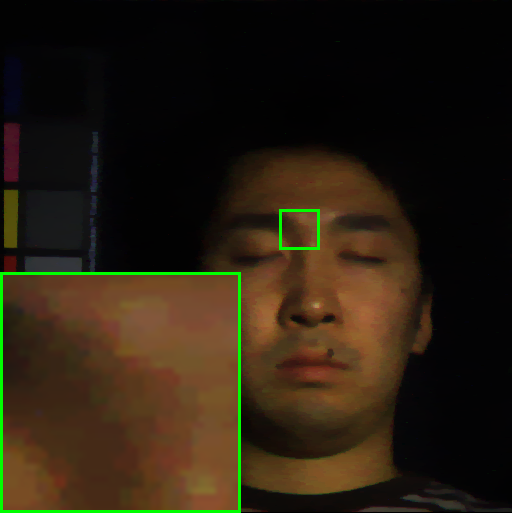}
	\end{minipage}}
    \subfigure[SSAHTV]{
	\begin{minipage}{0.2\linewidth}
		\centering
		\includegraphics[width=\linewidth]{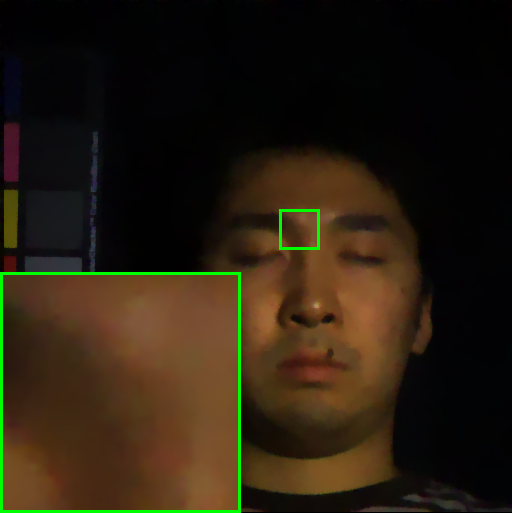}
	\end{minipage}}
    \subfigure[GOTTV]{
	\begin{minipage}{0.2\linewidth}
		\centering
		\includegraphics[width=\linewidth]{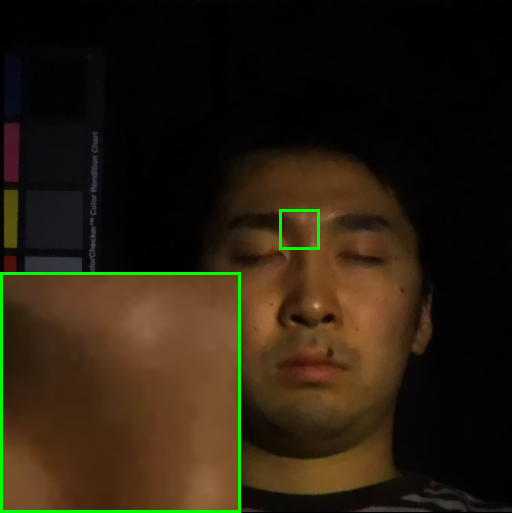}
	\end{minipage}}
 \caption{Comparison denoising results of various methods of the face image.}
  \label{Second Part of face}
\end{figure}
}


\begin{table}[H]
\caption{Blind denoising result.}
\label{Blind denoising result}
\centering
\scalebox{0.95}{
\begin{tabular}{|l|l|l|l|l|l|l|l|}
\hline
Figure            & \multicolumn{1}{c|}{Measure} & ORI    & CBCTV  & VTV          & ASSTV           & SSAHTV       & GOTTV           \\ \hline
Jasper Ridge     & MQ-metric                    & 0.0652 & 0.0964 & 0.1010       & \textbf{0.1069} & 0.1009       & {\underline{ 0.1045}}    \\ \hline
Pavia Centre     & MQ-metric                    & 0.0115 & 0.0189 & {\underline{0.0223}} & {\underline{0.0223}}    & {\underline{0.0223}} & \textbf{0.0236} \\ \hline
Pavia University & MQ-metric                    & 0.0192 & 0.0282 & {\underline{0.0322}} & 0.0316          & 0.0321       & \textbf{0.0340} \\ \hline
\end{tabular}
}
\end{table}
{
\begin{figure}[H]
	\centering
 \subfigure[Degraded]{
	\begin{minipage}{0.25\linewidth}
		\centering
		\includegraphics[width=\linewidth]{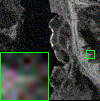}
	\end{minipage}}
 \subfigure[TV]{
	\begin{minipage}{0.25\linewidth}
		\centering
		\includegraphics[width=\linewidth]{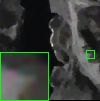}
	\end{minipage}}
  \subfigure[VTV]{
	\begin{minipage}{0.25\linewidth}
		\centering
		\includegraphics[width=\linewidth]{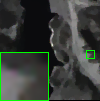}
	\end{minipage}}
    \subfigure[ASSTV]{
	\begin{minipage}{0.25\linewidth}
		\centering
		\includegraphics[width=\linewidth]{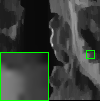}
	\end{minipage}}
     \subfigure[SSAHTV]{
	\begin{minipage}{0.25\linewidth}
		\centering
		\includegraphics[width=\linewidth]{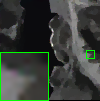}
	\end{minipage}}
     \subfigure[GOTTV]{
	\begin{minipage}{0.25\linewidth}
		\centering
		\includegraphics[width=\linewidth]{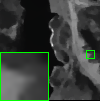}
	\end{minipage}}

   \caption{Jasper Ridge.}
  \label{Jasper Ridge}
\end{figure}
}
{
\begin{figure}[H]
	\centering
 \subfigure[Degraded]{
	\begin{minipage}{0.15\linewidth}
		\centering
		\includegraphics[width=\linewidth]{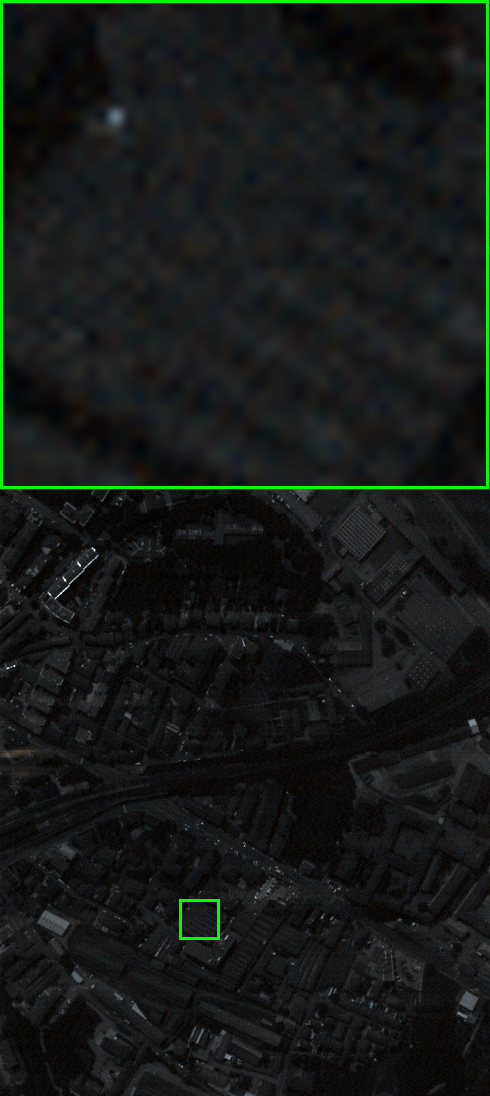}
	\end{minipage}}
      \subfigure[TV]{
	\begin{minipage}{0.15\linewidth}
		\centering
		\includegraphics[width=\linewidth]{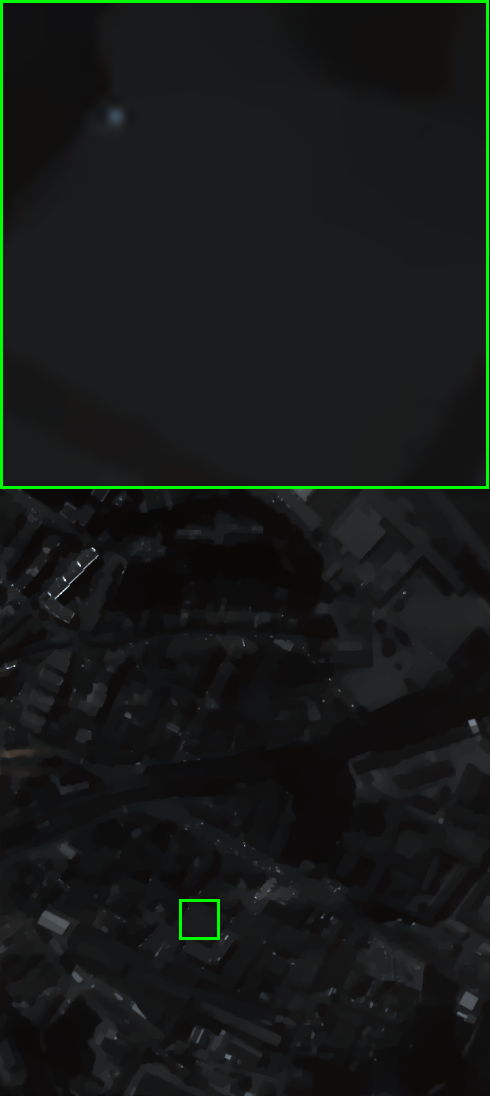}
	\end{minipage}}
  \subfigure[VTV]{
	\begin{minipage}{0.15\linewidth}
		\centering
		\includegraphics[width=\linewidth]{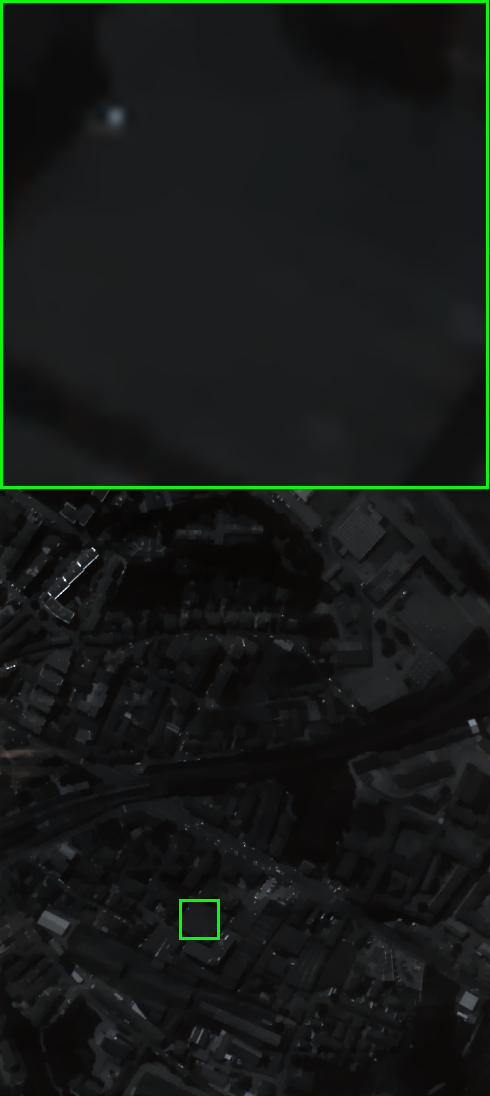}
	\end{minipage}}
    \subfigure[ASSTV]{
	\begin{minipage}{0.15\linewidth}
		\centering
		\includegraphics[width=\linewidth]{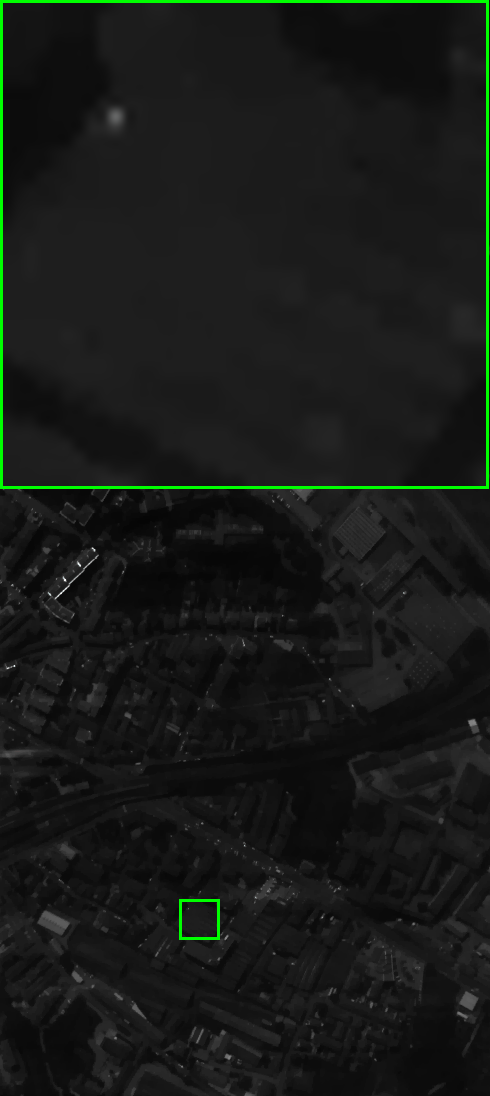}
	\end{minipage}}
     \subfigure[SSAHTV]{
	\begin{minipage}{0.15\linewidth}
		\centering
		\includegraphics[width=\linewidth]{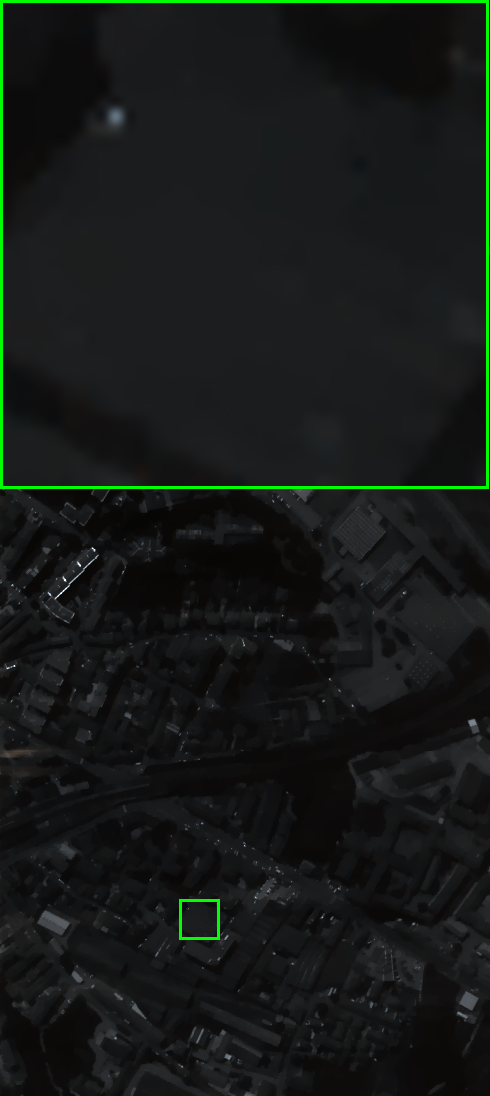}
	\end{minipage}}
     \subfigure[GOTTV]{
	\begin{minipage}{0.15\linewidth}
		\centering
		\includegraphics[width=\linewidth]{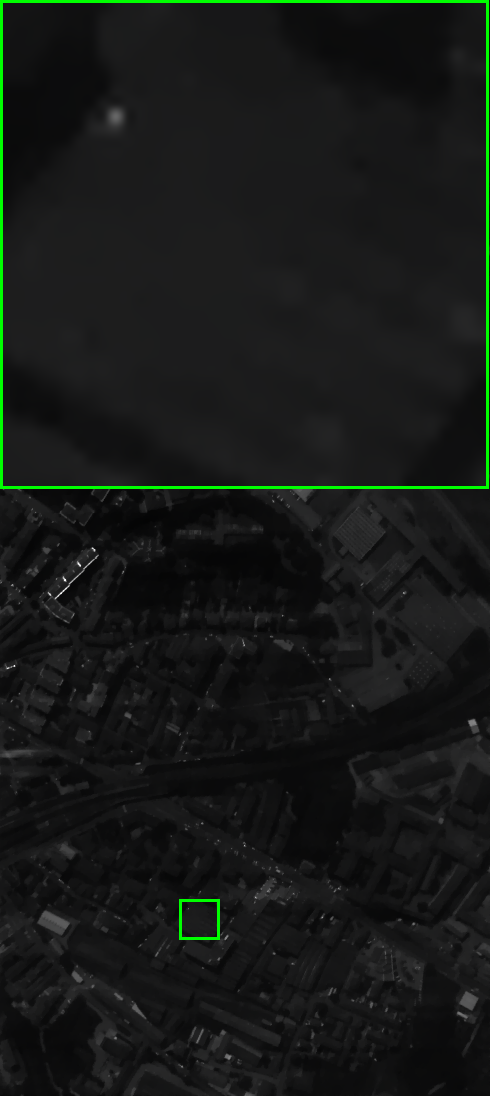}
	\end{minipage}}
   \caption{Pavia Center.}
  \label{Pavia Center}
\end{figure}
}
{
\begin{figure}[H]
	\centering
 \subfigure[Degraded]{
	\begin{minipage}{0.15\linewidth}
		\centering
		\includegraphics[width=\linewidth]{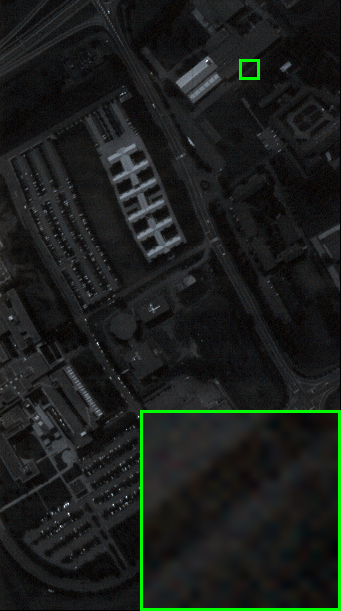}
	\end{minipage}}
      \subfigure[TV]{
	\begin{minipage}{0.15\linewidth}
		\centering
		\includegraphics[width=\linewidth]{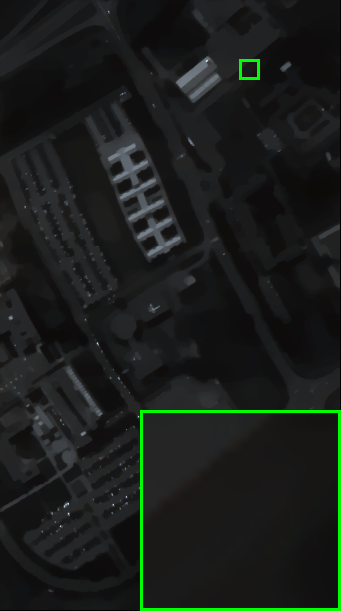}
	\end{minipage}}
  \subfigure[VTV]{
	\begin{minipage}{0.15\linewidth}
		\centering
		\includegraphics[width=\linewidth]{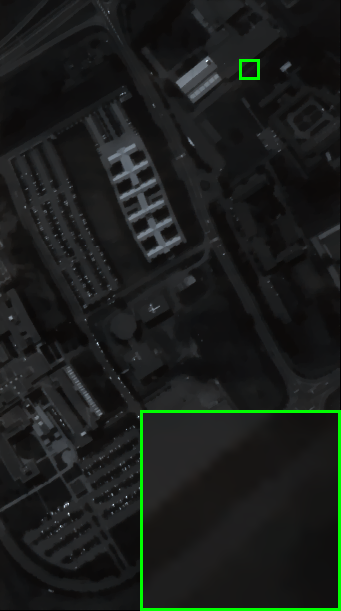}
	\end{minipage}}
    \subfigure[ASSTV]{
	\begin{minipage}{0.15\linewidth}
		\centering
		\includegraphics[width=\linewidth]{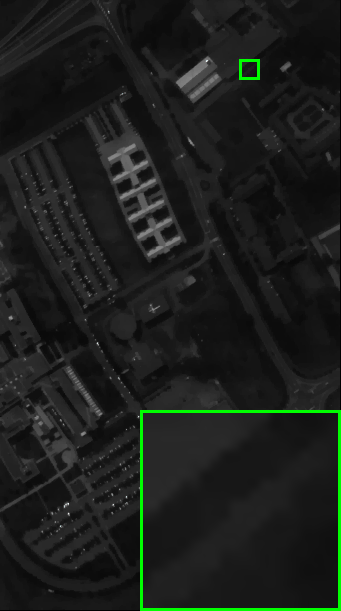}
	\end{minipage}}
     \subfigure[SSAHTV]{
	\begin{minipage}{0.15\linewidth}
		\centering
		\includegraphics[width=\linewidth]{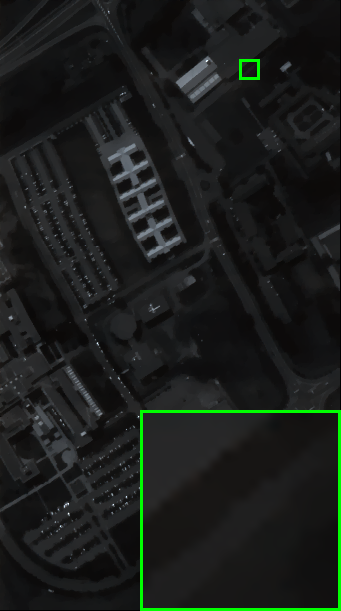}
	\end{minipage}}
     \subfigure[GOTTV]{
	\begin{minipage}{0.15\linewidth}
		\centering
		\includegraphics[width=\linewidth]{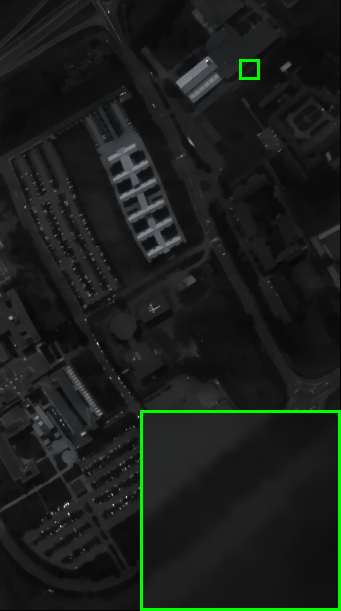}
	\end{minipage}}
   \caption{Pavia University.}
  \label{Pavia University}
\end{figure}
}

\subsection{Image deblurring}

In our non-blind deblurring experiment, we convolute the ground truth MSI with Gaussian kernel of the standard deviation 1/1.5 respectively and add Gaussian noise with standard deviations of 0.05 further. \Cref{Deblurring comparison results} displays the deblurring outcomes with the highest MPSNR after adjusting the regularization parameters. In \cref{Deblurring comparison results}, the best value in each row is highlighted in bold, while the second best value is underlined. The results indicate that regardless of the evaluation metrics used (MPSNR, MSSIM, time(second)), GOTTV consistently demonstrates excellent performance. In terms of parameter configuration, we set the value of initial r as 0.01, $r_{max}= 10^{6}, \rho=1.8$, and the maximum iteration as $10^{5}$. The remaining two parameters are detailed in \cref{Deblurring Parameter setting}. The deblurring performance in terms of PSNR and SSIM metrics across different channels for various methods are demonstrated in \cref{PSNR comparison for cloth deblurring,SSIM comparison for cloth deblurring,PSNR comparison for face deblurring,SSIM comparison for face deblurring}. Analyzing the results, it is evident that GOTTV surpasses other methods in both PSNR and SSIM for most channels. 

The above average and spectrum-by-spectrum PSNR and SSIM comparison results provide an understanding of the overall performance of different methods. To visually showcase the recovery of details and textures, we will zoom in on the following pseudo-color \cref{First Deblur of cloth,Second Deblur of cloth,First Deblur of face,Second Deblur of face}. Looking at the restoration results of the flower in \cref{First Deblur of cloth}, it is evident that only GOTTV effectively recovers the internal boundaries of the petals, while other methods fail to capture the internal texture information of the petals. Taking another look at \cref{Second Deblur of cloth}, only the restoration results of GOTTV preserve the gaps between the circles effectively. In contrast, the restoration results of the other methods not only fused the individual circles into a single entity but also contained unnatural red and green stains. \Cref{First Deblur of face} illustrates that only the results obtained by GOTTV effectively preserve the white stripes and successfully separate the three white stripes. In contrast, the restoration outcomes of other methods either entirely lose white stripes information or merge white stripes together. Finally, let us compare the restoration outcomes of various methods on the facial features, as depicted in \cref{Second Deblur of face}. TV and ASSTV methods fail to effectively eliminate noise, while VTV and SSAHTV, although successful in noise reduction, result in the loss of facial details, such as moles near the mouth. Only our GOTTV manages to preserve facial details while effectively removing noise. This indicates that GOTTV performs exceptionally well in restoring fine details and overall image quality. 

In the context of blind deblurring, we may not have explicit information about the specific characteristics of the blur kernels. Therefore, a common practice is to perform blur kernel estimation before engaging in non-blind image restoration \cite{li2019blur}. Numerous papers highlight that noise can significantly interfere with the accuracy of blur kernel estimation \cite{miao2020handling}. Hence, we adopt two methods to mitigate the impact of noise on blur kernel estimation. Firstly, we average along channels, consolidating the channels of a multispectral image into a grayscale representation. This channel fusion technique aims to diminish the intensity of noise as shown in \cref{15thAFuse}. Secondly, we adopt a blur kernel estimation method applying the outlier identifying and discarding (OID) method, which is robust against noise since it is devised to exploit clean elements while excluding those contaminated by outliers in the estimating kernel procedure \cite{chen2020oid}. By selecting techniques with inherent noise resilience, we aim to alleviate further the adverse effects of noise on the accuracy of blur kernel estimation. 

To ensure fairness in comparative experiments, we employed the two aforementioned methods for blur kernel estimation for each model. The \cref{Blind Deblurring comparison results} below presents the results of image recovery for each method using the estimated blur kernel. It is evident that our approach continues to maintain a leading advantage. In addition to the table, the following pseudo-color images in \cref{Blind-deblurring comparison results of the face image,Blind-deblurring comparison results of the cloth image} provide a more intuitive comparison of the blind deblurring results achieved by different methods. Observed \cref{Blind-deblurring comparison results of the cloth image}, it is evident that TV, VTV, SSAHTV, and ASSTV fail to effectively separate the white circular texture on the cloth, with the recovered results still exhibiting significant noise. Only our GOTTV demonstrates the capability to eliminate noise and effectively preserve the white circular texture. Observing \cref{Blind-deblurring comparison results of the face image}, it is apparent that TV, VTV, and SSAHTV do not effectively restore the upper two white patterns. Their recovery outcomes not only amalgamate the two white patterns into a single entity but also appear relatively blurry. Although ASSTV can separate the upper two white patterns, its recovery results appear overly sharp and square compared to the ground-truth image. Moreover, its noise removal is not entirely clean. Only our GOTTV demonstrates the ability to not only effectively restore the details of these white patterns but also to cleanly eliminate noise.

\begin{table}[H]
\caption{Deblurring comparison results.}
\label{Deblurring comparison results}
\centering
\scalebox{0.9}{
\begin{tabular}{|l|l|l|l|l|l|l|l|l|} 
\hline
Figure                   & Std                   & Measure & Degraded image & TV      & VTV     & ASSTV           & SSAHTV          & GOTTV             \\ 
\hline
\multirow{6}{*}{Cloth}   & \multirow{3}{*}{1} & MPSNR   & 25.2061     & 30.1702 & 30.9168 & \underline{31.3226}         & {30.9769} & \textbf{32.8568}  \\ 
\cline{3-9}
                         &                       & MSSIM   & 0.4412      & 0.7713  & 0.7968  & \underline{0.8240}  & 0.8000          & \textbf{0.8692}   \\ 
\cline{3-9}
                        &                       & Time & - & 428.24 & 254.98 & 3108.28 & \underline{28.18} & \textbf{12.86}\\

\cline{2-9}
                         & \multirow{3}{*}{1.5}  & MPSNR   & 24.5608     & 28.9749 & 29.5456 & \underline{29.9799}         & {29.5549} & \textbf{31.214}  \\ 
\cline{3-9}
                         &                       & MSSIM   & 0.3686      & 0.7099  & 0.7348  & \underline{0.7678}          & {0.7362}  & \textbf{0.8191}   \\ 
\cline{3-9}
                         &                      & Time & - & 514.99 & 403.03 & 3581.64 & \underline{42.69} & \textbf{12.86}\\
\hline
\multirow{6}{*}{Face} & \multirow{3}{*}{1} & MPSNR   & 25.9121     & 39.1002 & \underline{40.3380} & {40.3171} & 40.0962         & \textbf{41.3793}  \\ 
\cline{3-9}
                         &                       & MSSIM   & 0.2871      & 0.9507  & 0.9630  & 0.9627          & \underline{0.9631}  & \textbf{0.9734}   \\ 
\cline{3-9}

                        &                       &  Time  &  - & 818.36 & 608.99 & 1725.86 & \underline{50.34} & \textbf{13.28}\\
\cline{2-9}
                         & \multirow{3}{*}{1.5}  & MPSNR   & 25.8033     & 38.2502 & \underline{39.4376} & {39.5595} & 39.3234         & \textbf{40.3768}  \\ 
\cline{3-9}
                         &                       & MSSIM   & 0.2802      & 0.9436  & 0.9550  & \underline{0.9588}  & 0.9567          & \textbf{0.9684}   \\ 
\cline{3-9}             &                         & Time & - & 928.68 & 1030.37 & 2235.28 & \underline{74.19} & \textbf{13.28}\\
\hline
\multirow{6}{*}{Jelly}   & \multirow{3}{*}{1} & MPSNR   & 25.2537     & 31.3822 & 32.7269 & \underline{33.2629} & 32.8713         & \textbf{34.5739}  \\ 
\cline{3-9}
                         &                       & MSSIM   & 0.4539      & 0.9001  & 0.9127  & 0.9194        & \underline{0.9227}  & \textbf{0.938}   \\ 
\cline{3-9}
                        &                       & Time & - & 545.26 & 639.61 & 3178.01 & \underline{78.4} & \textbf{12.4} \\
\cline{2-9}
                         & \multirow{3}{*}{1.5}  & MPSNR   & 24.5277     & 29.5596 & 30.7526 & \underline{31.5427}         & {30.7444} & \textbf{32.6888}  \\ 
\cline{3-9}
                         &                       & MSSIM   & 0.4272      & 0.8695  & 0.8816  & \underline{0.8976}          & {0.8885}  & \textbf{0.9178}   \\ 
\cline{3-9}
                        &                           & Time & - & 724.04 & 860.1& 3914.04& \underline{85.58} & \textbf{12.05}\\

\hline

\multirow{6}{*}{Picture}    & \multirow{3}{*}{1} & MPSNR   & 25.89     & 37.8958 & 39.1887 & \underline{39.8298}         & {39.1153} & \textbf{40.4454}  \\ 
\cline{3-9}
                         &                       & MSSIM   & 0.2986      & 0.9321  & 0.9444  & \underline{0.9606}          & {0.9455}  & \textbf{0.9617}   \\ 
\cline{3-9}
                        &                        & Time & - & 745.34 & 640.97 & 2003.39 & \underline{47.19} & \textbf{12.83}\\
\cline{2-9}
                         & \multirow{3}{*}{1.5}  & MPSNR   & 25.7343    & 36.8457 & 38.0610 & \underline{38.8093}         & {37.9724} & \textbf{39.3872}  \\ 
\cline{3-9}
                         &                       & MSSIM   & 0.2877      & 0.9179  & 0.9260  & \underline{0.9515}         & {0.9275}  & \textbf{0.9538}   \\ 
\cline{3-9}             &                        & Time & - & 952.61 & 1075.39 & 2951.09 & \underline{78.01} & \textbf{12.87}\\
\hline
\multirow{6}{*}{Thread}  & \multirow{3}{*}{1} & MPSNR   & 25.7171     & 34.9327 &35.9682 & \underline{36.9602}         & {36.1031} & \textbf{37.9943}  \\ 
\cline{3-9}
                         &                       & MSSIM   & 0.3188      & 0.8961  & 0.9090  & \underline{0.9327}  & 0.9167          & \textbf{0.945}   \\ 
\cline{3-9}
                        &                       & Time & - & 551.83 & 449.77 & 1536.34 & \underline{37.54} & \textbf{12.67}\\
\cline{2-9}
                         & \multirow{3}{*}{1.5}  & MPSNR   & 25.3961     & 33.2929 & 34.2343 & \underline{35.6393}         & {34.3131} & \textbf{36.4153}  \\ 
\cline{3-9}
                         &                       & MSSIM   & 0.3014      & 0.8672  & 0.8798  & \underline{0.9116}         & {0.8857}  & \textbf{0.9279}   \\
\cline{3-9}
                        &                       & Time & - & 608.72 & 686.14 & 3270.96 & \underline{58.31} & \textbf{12.5}\\
\hline
\multirow{6}{*}{Dataset}  & \multirow{3}{*}{1} & MPSNR   &   25.7174   & 36.7721 &37.8611 & \underline{38.0087}         & {37.6923} & \textbf{39.1771}  \\ 
\cline{3-9}
                         &                       & MSSIM   &   0.3230    & 0.9263  & 0.9349  & \underline{0.9423}  & 0.9374          & \textbf{0.9513}   \\ 
\cline{3-9}
                        &                       & Time & - & 686.90 & 510.19 & 1889.57 & \underline{53.87} & \textbf{13.26}\\
\cline{2-9}
                         & \multirow{3}{*}{1.5}  & MPSNR   &   25.4386   & 35.5647 & 36.5761 & \underline{36.8451}         & {36.3719} & \textbf{37.8231}  \\ 
\cline{3-9}
                         &                       & MSSIM   &  0.3084    & 0.9130  & 0.9200  & \underline{0.9310}         & {0.9210}  & \textbf{0.9398}   \\
\cline{3-9}
                        &                       & Time & - & 608.72 & 686.14 & 3270.96 & \underline{58.31} & \textbf{12.5}\\
\hline
\end{tabular}
}

\end{table}

\begin{table}[H]
\caption{Deblurring parameter setting.}
\label{Deblurring Parameter setting}
\centering
\begin{tabular}{|l|l|l|l|l|l|l|}
\hline
Std                  & Parameter & Cloth    & Face   & Jelly   & Picture  & Thread   \\ \hline
\multirow{2}{*}{1}   & $\lambda$         & 10.965  & 7.527 & 9.324 & 8.471   & 7.812   \\ \cline{2-7} 
                     & $\alpha$         & 0.074  & 0.138 & 0.098  & 0.158  & 0.087 \\ \hline
\multirow{2}{*}{1.5} & $\lambda$         & 14.932   & 9.2      & 12.35   & 10.091   & 9.862  \\ \cline{2-7} 
                     & $\alpha$         & 0.045 & 0.139 & 0.057 & 0.110 & 0.053 \\ \hline
\end{tabular}

\end{table}

{
\begin{figure}[H]
	\centering
 \subfigure[Std=1]{
	\begin{minipage}{0.48\linewidth}
		\centering
		\includegraphics[width=\linewidth]{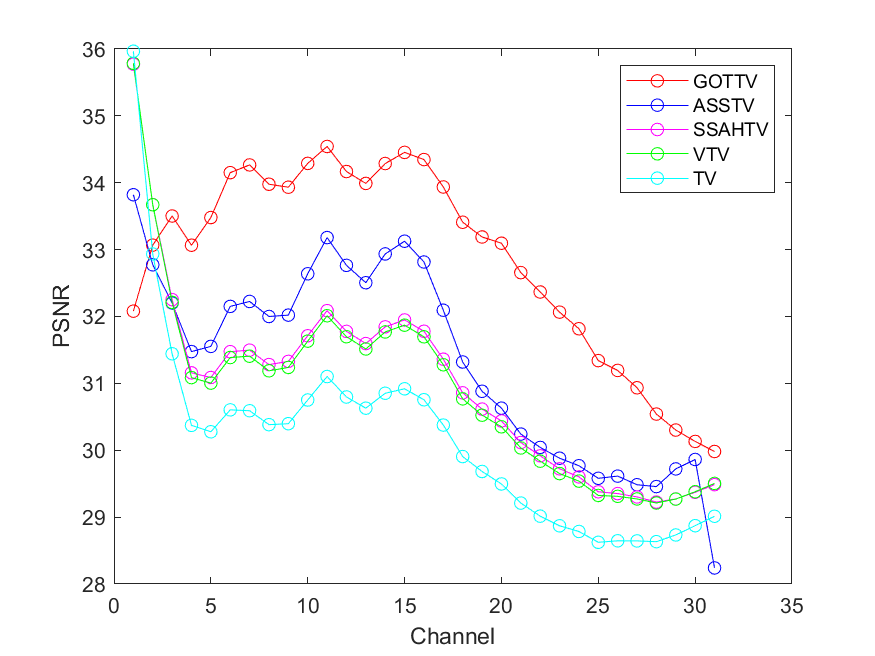}
	\end{minipage}}
 \subfigure[Std=1.5]{
	\begin{minipage}{0.48\linewidth}
		\centering
		\includegraphics[width=\linewidth]{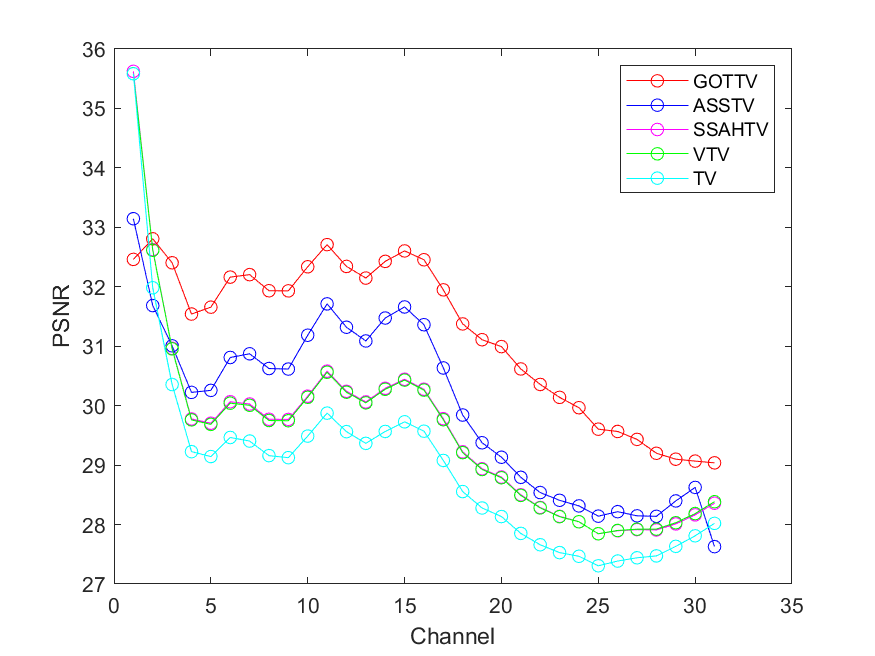}
	\end{minipage}}

   \caption{Spectrum-by-spectrum PSNR comparison chart of different methods for cloth.}
  \label{PSNR comparison for cloth deblurring}
\end{figure}
}

{
\begin{figure}[H]
	\centering
 \subfigure[Std=1]{
	\begin{minipage}{0.48\linewidth}
		\centering
		\includegraphics[width=\linewidth]{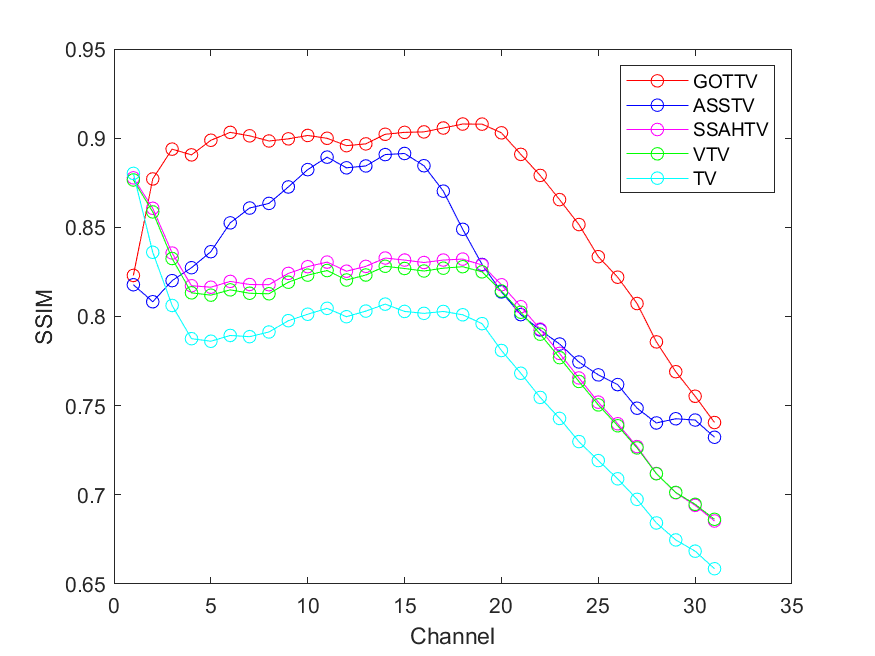}
	\end{minipage}}
 \subfigure[Std=1.5]{
	\begin{minipage}{0.48\linewidth}
		\centering
		\includegraphics[width=\linewidth]{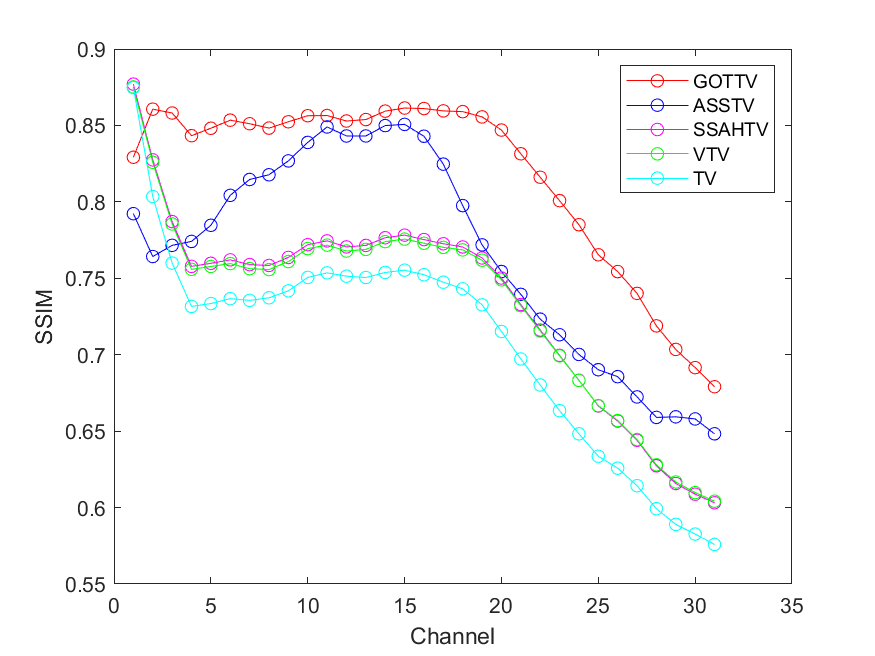}
	\end{minipage}}

   \caption{Spectrum-by-spectrum SSIM comparison chart of different methods for cloth.}
  \label{SSIM comparison for cloth deblurring}
\end{figure}
}

{
\begin{figure}[H]
	\centering
 \subfigure[Std=1]{
	\begin{minipage}{0.48\linewidth}
		\centering
		\includegraphics[width=\linewidth]{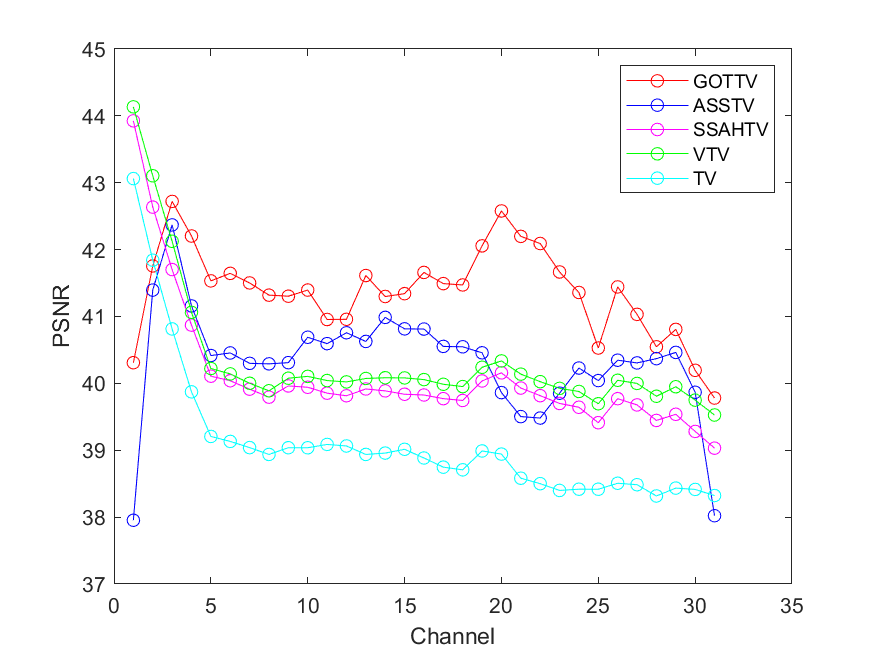}
	\end{minipage}}
 \subfigure[Std=1.5]{
	\begin{minipage}{0.48\linewidth}
		\centering
		\includegraphics[width=\linewidth]{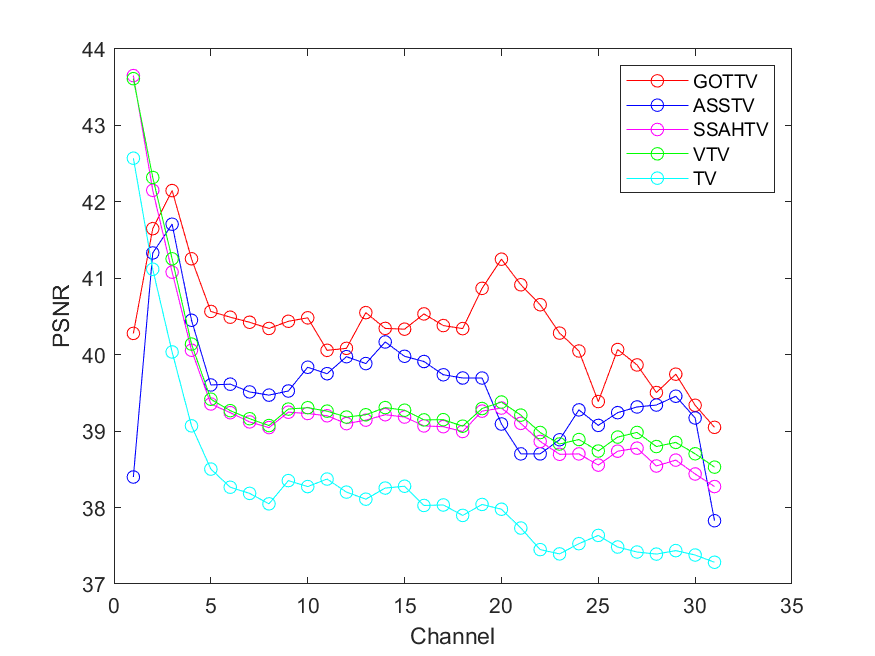}
	\end{minipage}}

   \caption{Spectrum-by-spectrum PSNR comparison chart of different methods for face.}
  \label{PSNR comparison for face deblurring}
\end{figure}
}

{
\begin{figure}[H]
	\centering
 \subfigure[Std=1]{
	\begin{minipage}{0.48\linewidth}
		\centering
		\includegraphics[width=\linewidth]{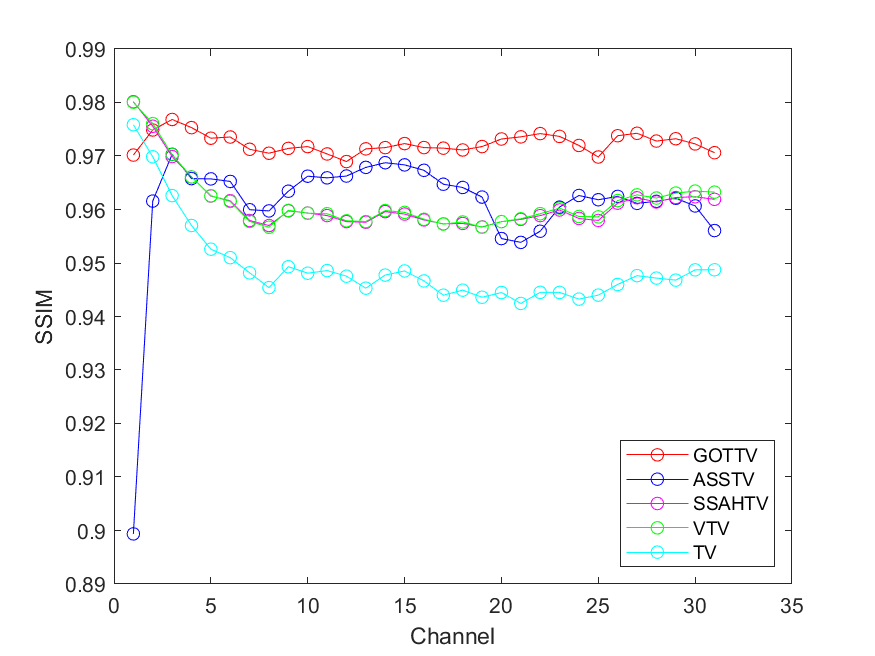}
	\end{minipage}}
 \subfigure[Std=1.5]{
	\begin{minipage}{0.48\linewidth}
		\centering
		\includegraphics[width=\linewidth]{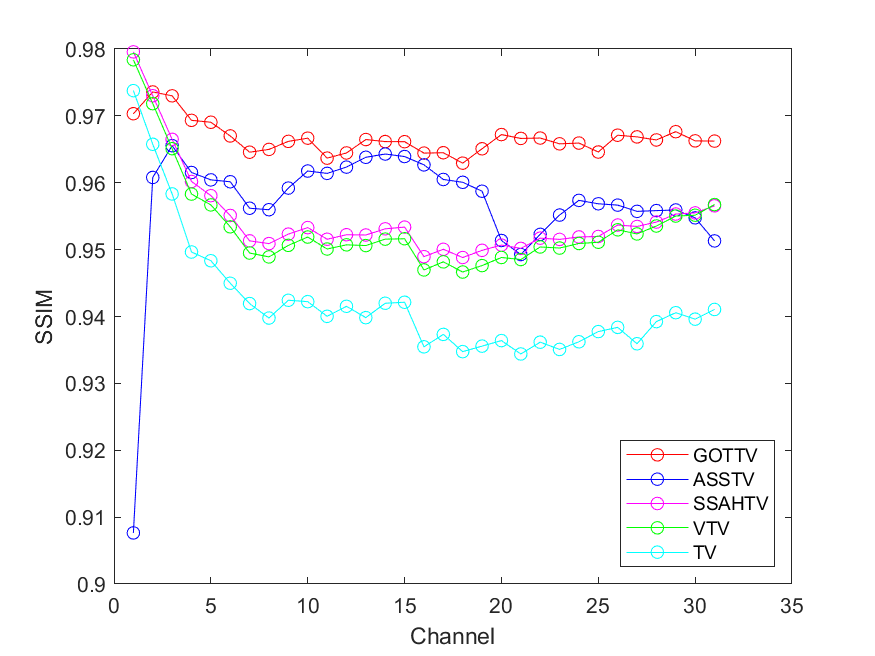}
	\end{minipage}}

   \caption{Spectrum-by-spectrum SSIM comparison chart of different methods for face.}
  \label{SSIM comparison for face deblurring}
\end{figure}
}


{
\begin{figure}[H]
		\centering
 \subfigure[Ground-truth]{
	\begin{minipage}{0.2\linewidth}
		\centering
		\includegraphics[width=\linewidth]{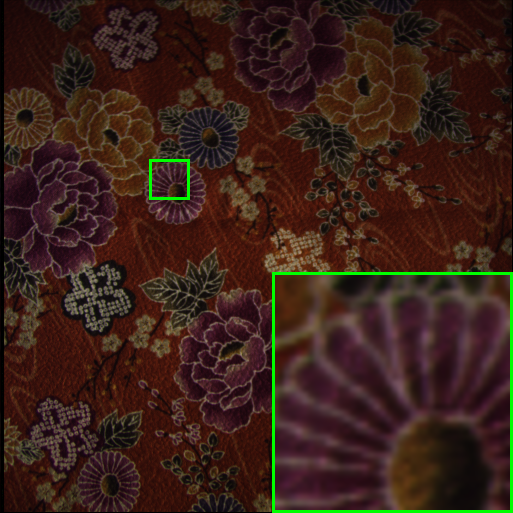}
	\end{minipage}}
 \subfigure[Degraded]{
	\begin{minipage}{0.2\linewidth}
		\centering
		\includegraphics[width=\linewidth]{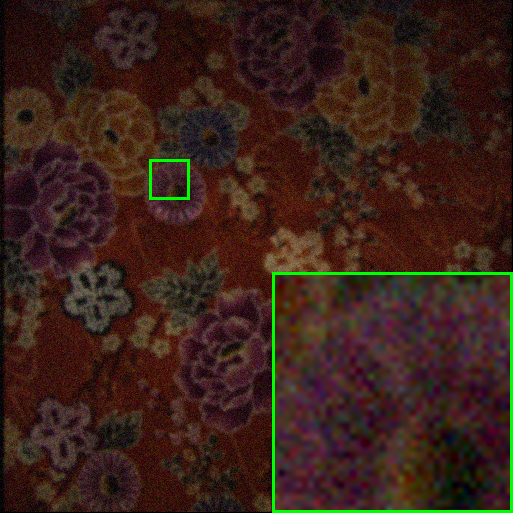}
	\end{minipage}}
 \subfigure[TV]{
	\begin{minipage}{0.2\linewidth}
		\centering
		\includegraphics[width=\linewidth]{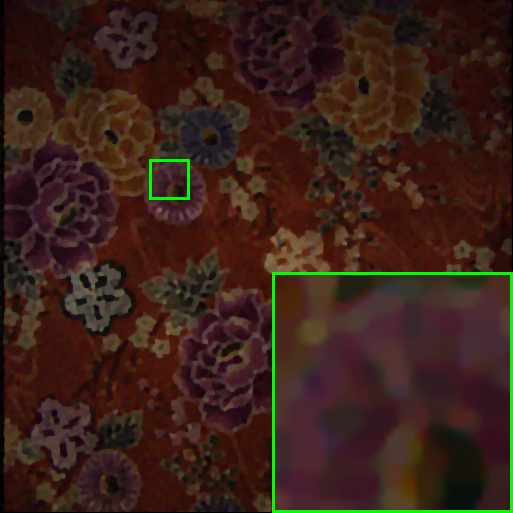}
	\end{minipage}}
  \subfigure[VTV]{
	\begin{minipage}{0.2\linewidth}
		\centering
		\includegraphics[width=\linewidth]{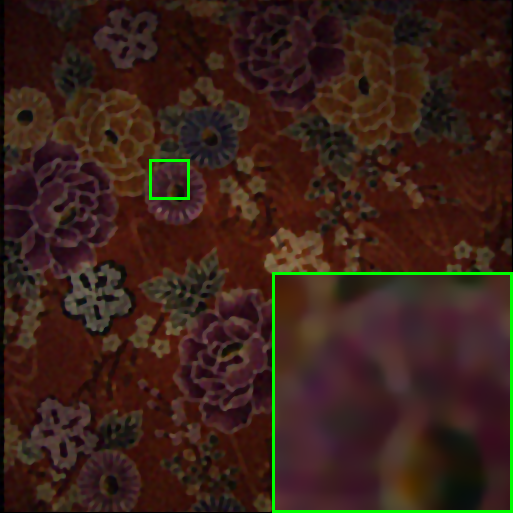}
	\end{minipage}}
   \subfigure[ASSTV]{
	\begin{minipage}{0.2\linewidth}
		\centering
		\includegraphics[width=\linewidth]{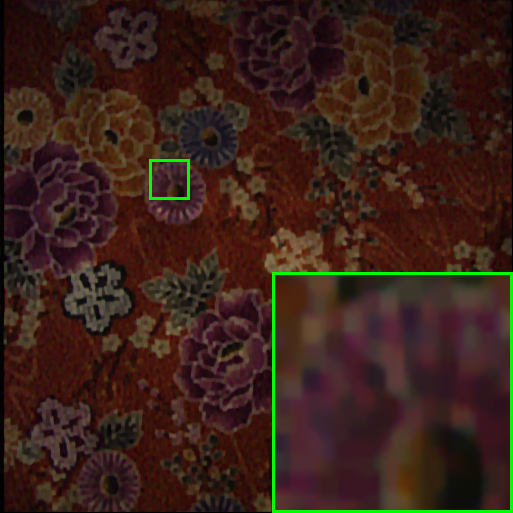}
	\end{minipage}}
    \subfigure[SSAHTV]{
	\begin{minipage}{0.2\linewidth}
		\centering
		\includegraphics[width=\linewidth]{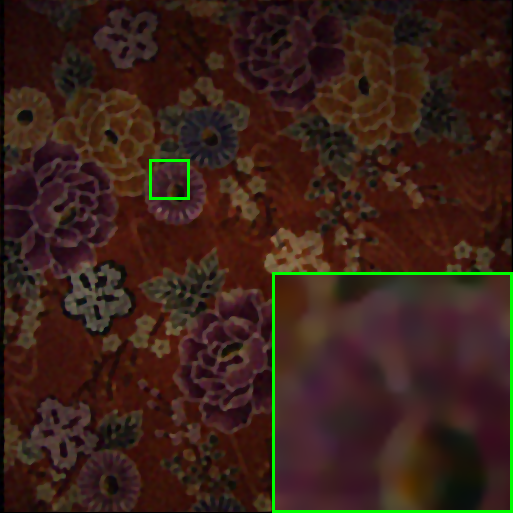}
	\end{minipage}}
    \subfigure[GOTTV]{
	\begin{minipage}{0.2\linewidth}
		\centering
		\includegraphics[width=\linewidth]{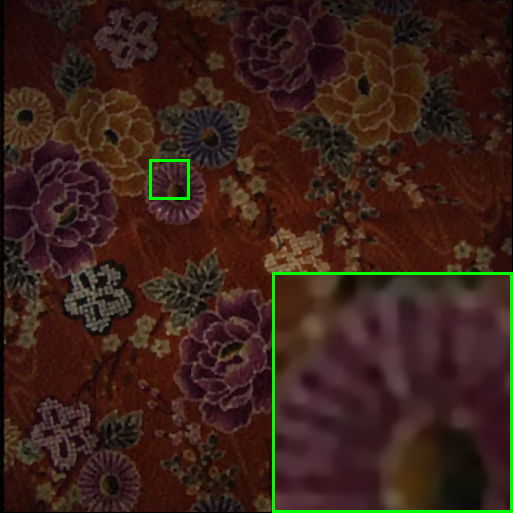}
	\end{minipage}}
 \caption{Comparison deblurring results of various methods of the cloth image.}
  \label{First Deblur of cloth}
\end{figure}
}

{
\begin{figure}[H]
		\centering
 \subfigure[Ground-truth]{
	\begin{minipage}{0.2\linewidth}
		\centering
		\includegraphics[width=\linewidth]{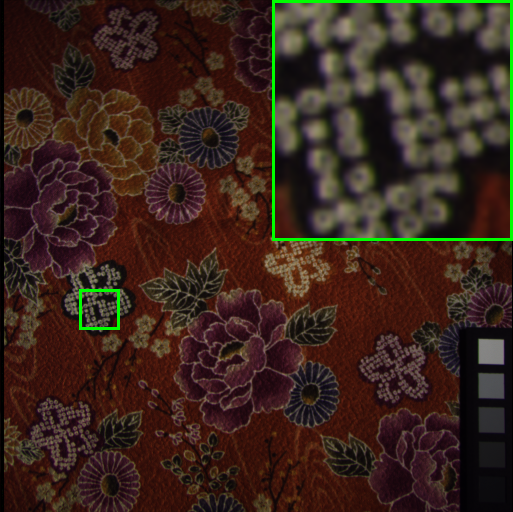}
	\end{minipage}}
 \subfigure[Degraded]{
	\begin{minipage}{0.2\linewidth}
		\centering
		\includegraphics[width=\linewidth]{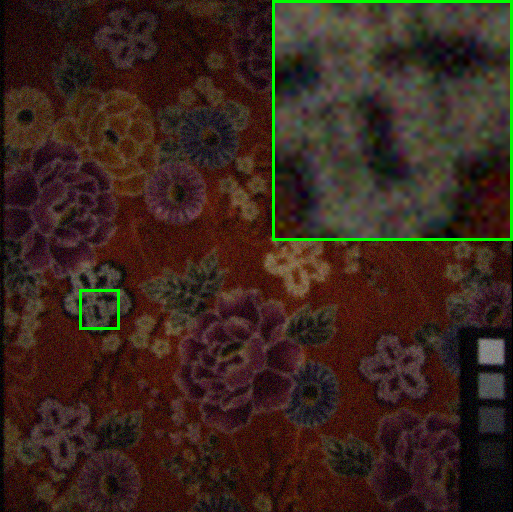}
	\end{minipage}}
 \subfigure[TV]{
	\begin{minipage}{0.2\linewidth}
		\centering
		\includegraphics[width=\linewidth]{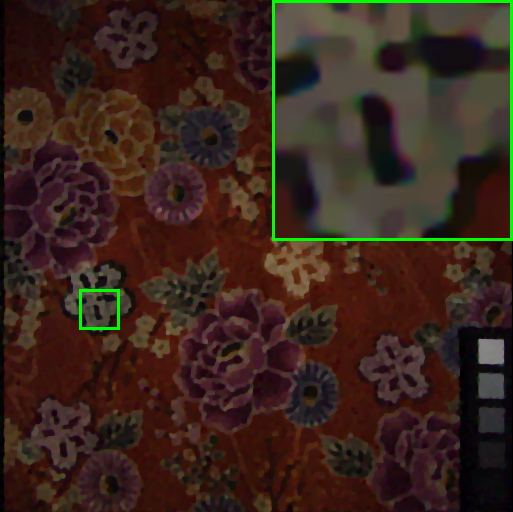}
	\end{minipage}}
  \subfigure[VTV]{
	\begin{minipage}{0.2\linewidth}
		\centering
		\includegraphics[width=\linewidth]{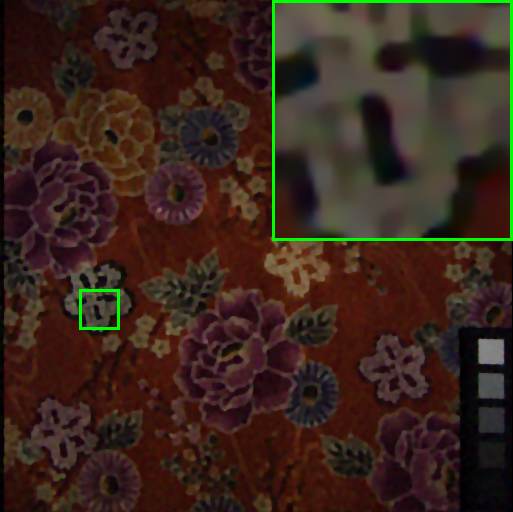}
	\end{minipage}}
   \subfigure[ASSTV]{
	\begin{minipage}{0.2\linewidth}
		\centering
		\includegraphics[width=\linewidth]{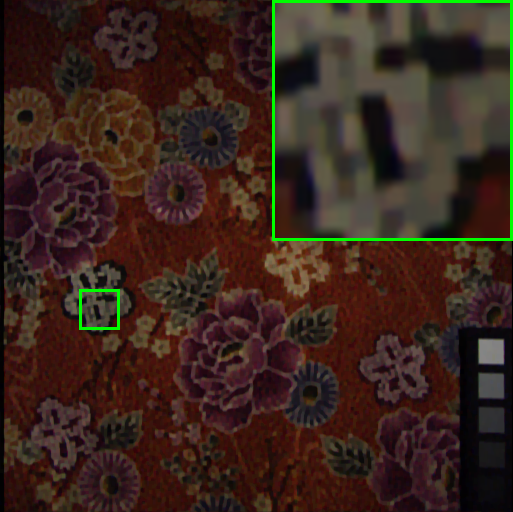}
	\end{minipage}}
    \subfigure[SSAHTV]{
	\begin{minipage}{0.2\linewidth}
		\centering
		\includegraphics[width=\linewidth]{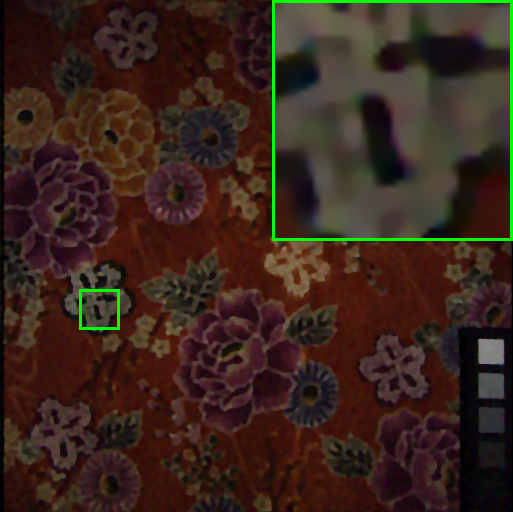}
	\end{minipage}}
    \subfigure[GOTTV]{
	\begin{minipage}{0.2\linewidth}
		\centering
		\includegraphics[width=\linewidth]{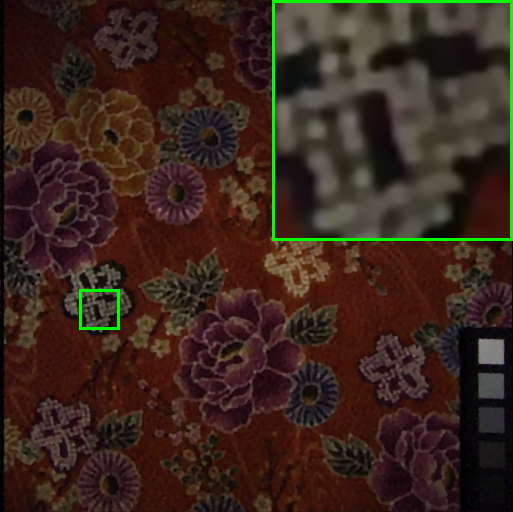}
	\end{minipage}}
 \caption{Comparison deblurring results of various methods of the cloth image.}
  \label{Second Deblur of cloth}
\end{figure}
}

{
\begin{figure}[H]
		\centering
 \subfigure[Ground-truth]{
	\begin{minipage}{0.2\linewidth}
		\centering
		\includegraphics[width=\linewidth]{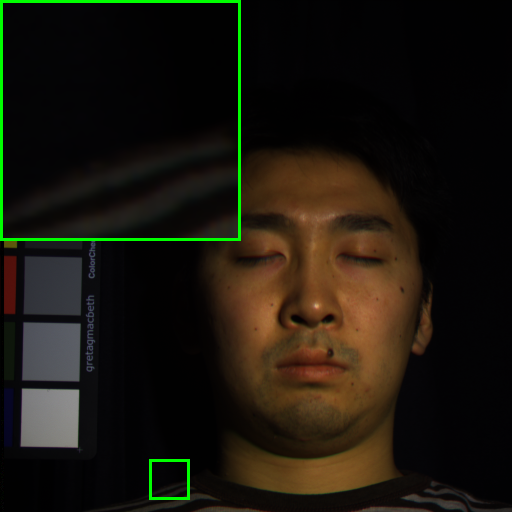}
	\end{minipage}}
 \subfigure[Degraded]{
	\begin{minipage}{0.2\linewidth}
		\centering
		\includegraphics[width=\linewidth]{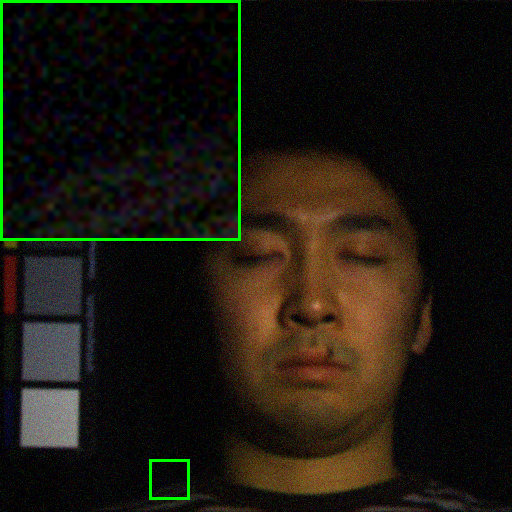}
	\end{minipage}}
 \subfigure[TV]{
	\begin{minipage}{0.2\linewidth}
		\centering
		\includegraphics[width=\linewidth]{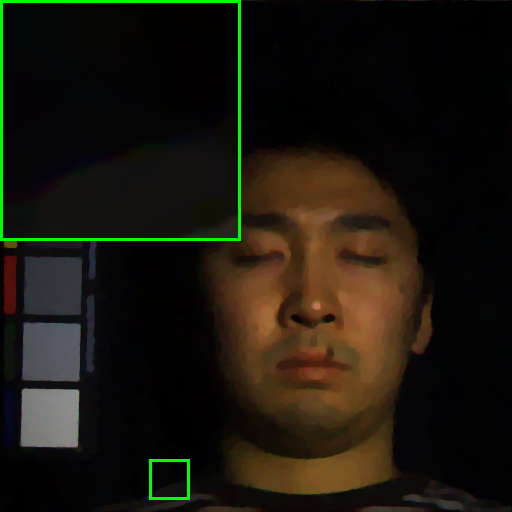}
	\end{minipage}}
  \subfigure[VTV]{
	\begin{minipage}{0.2\linewidth}
		\centering
		\includegraphics[width=\linewidth]{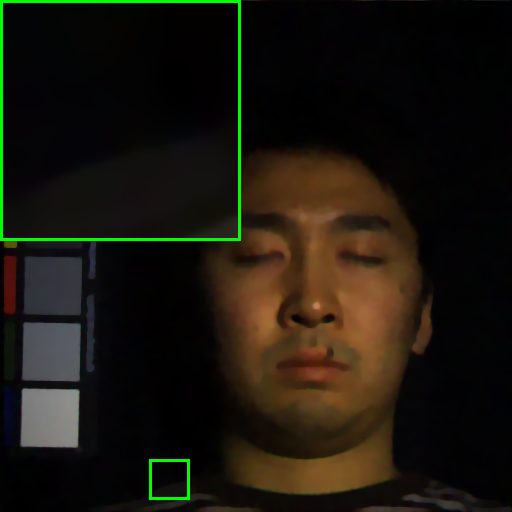}
	\end{minipage}}
   \subfigure[ASSTV]{
	\begin{minipage}{0.2\linewidth}
		\centering
		\includegraphics[width=\linewidth]{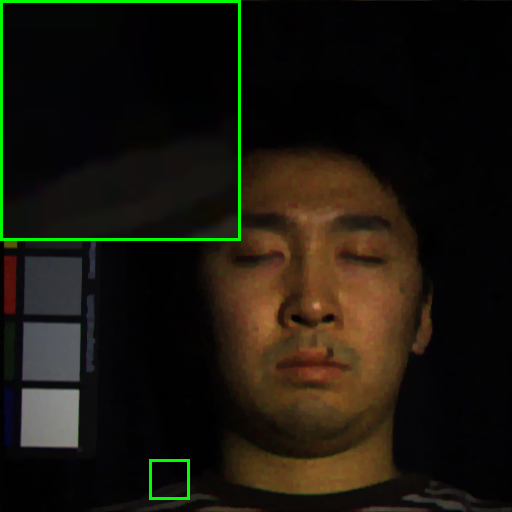}
	\end{minipage}}
    \subfigure[SSAHTV]{
	\begin{minipage}{0.2\linewidth}
		\centering
		\includegraphics[width=\linewidth]{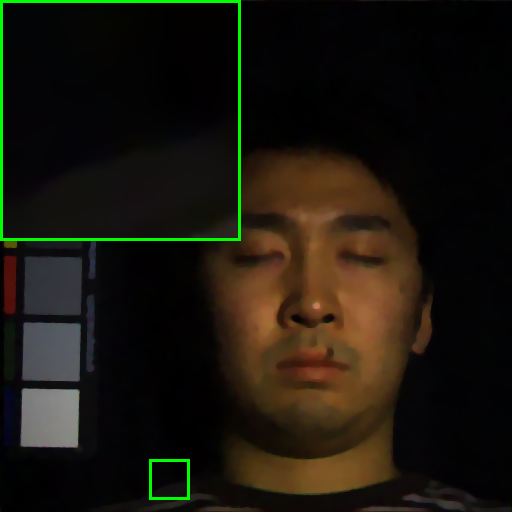}
	\end{minipage}}
    \subfigure[GOTTV]{
	\begin{minipage}{0.2\linewidth}
		\centering
		\includegraphics[width=\linewidth]{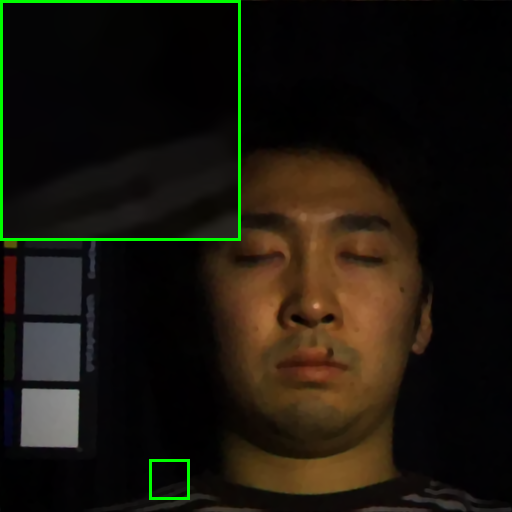}
	\end{minipage}}
 \caption{Comparison deblurring results of various methods of the face image.}
  \label{First Deblur of face}
\end{figure}
}

{
\begin{figure}[H]
		\centering
 \subfigure[Ground-truth]{
	\begin{minipage}{0.2\linewidth}
		\centering
		\includegraphics[width=\linewidth]{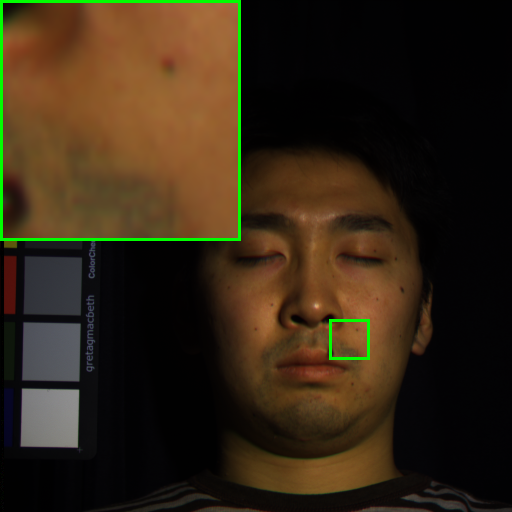}
	\end{minipage}}
 \subfigure[Degraded]{
	\begin{minipage}{0.2\linewidth}
		\centering
		\includegraphics[width=\linewidth]{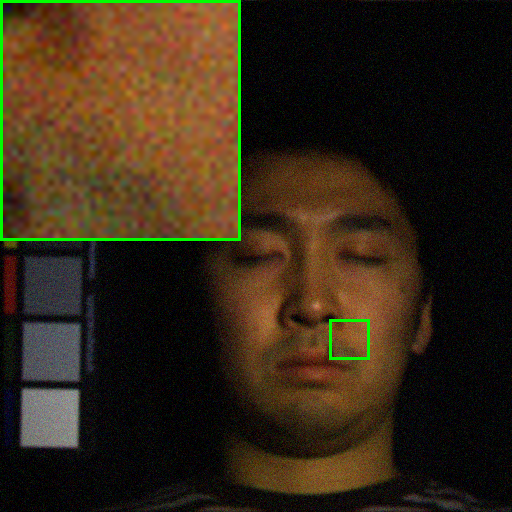}
	\end{minipage}}
 \subfigure[TV]{
	\begin{minipage}{0.2\linewidth}
		\centering
		\includegraphics[width=\linewidth]{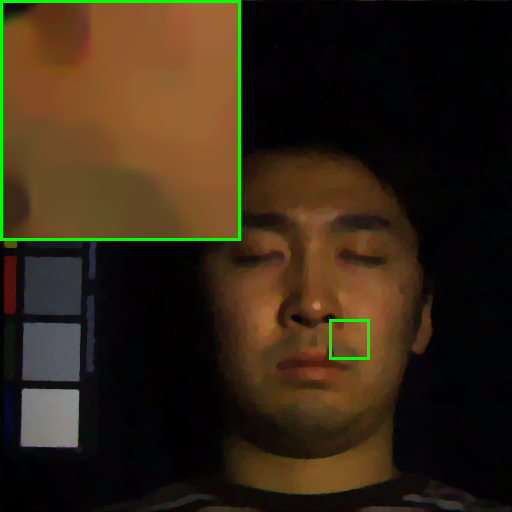}
	\end{minipage}}
  \subfigure[VTV]{
	\begin{minipage}{0.2\linewidth}
		\centering
		\includegraphics[width=\linewidth]{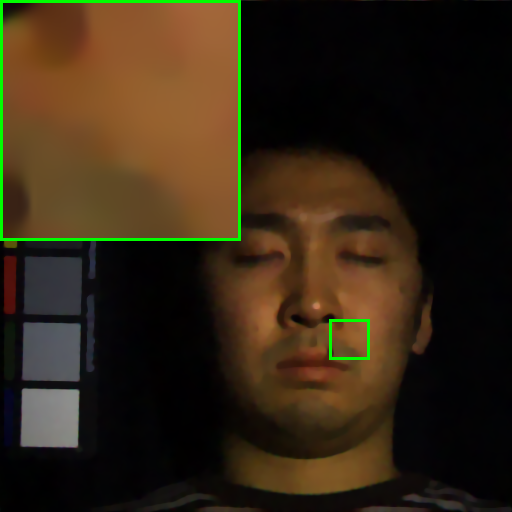}
	\end{minipage}}
   \subfigure[ASSTV]{
	\begin{minipage}{0.2\linewidth}
		\centering
		\includegraphics[width=\linewidth]{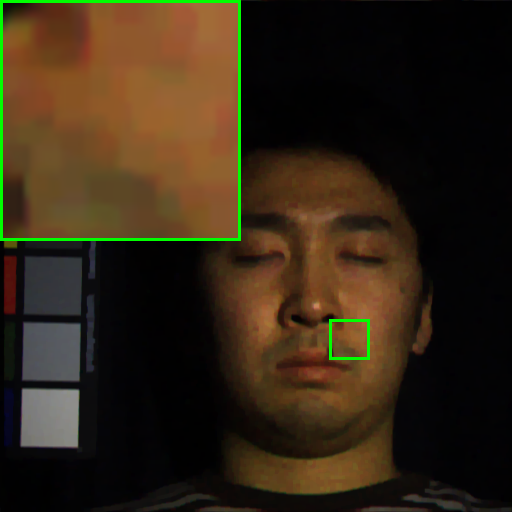}
	\end{minipage}}
    \subfigure[SSAHTV]{
	\begin{minipage}{0.2\linewidth}
		\centering
		\includegraphics[width=\linewidth]{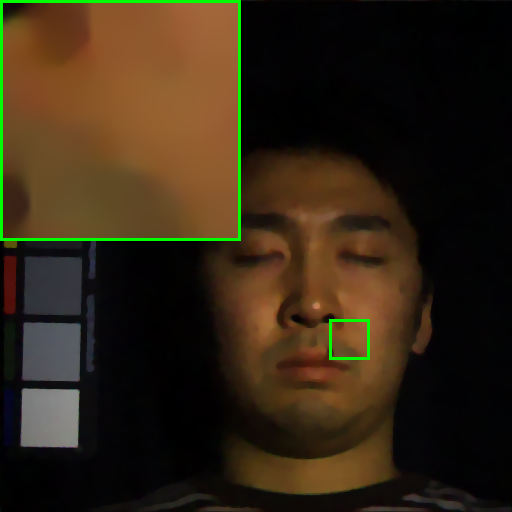}
	\end{minipage}}
    \subfigure[GOTTV]{
	\begin{minipage}{0.2\linewidth}
		\centering
		\includegraphics[width=\linewidth]{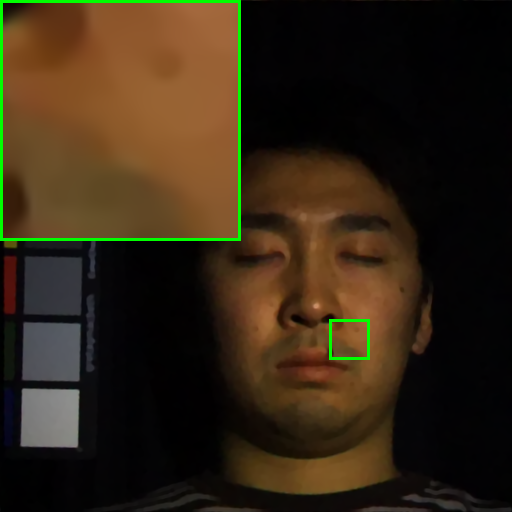}
	\end{minipage}}
 \caption{Comparison deblurring results of various methods of the face image.}
  \label{Second Deblur of face}
\end{figure}
}
{
\begin{figure}[H]
			\centering
 \subfigure[15th channel of blurring cloth]{
	\begin{minipage}{0.22\linewidth}
		\centering
		\includegraphics[width=\linewidth]{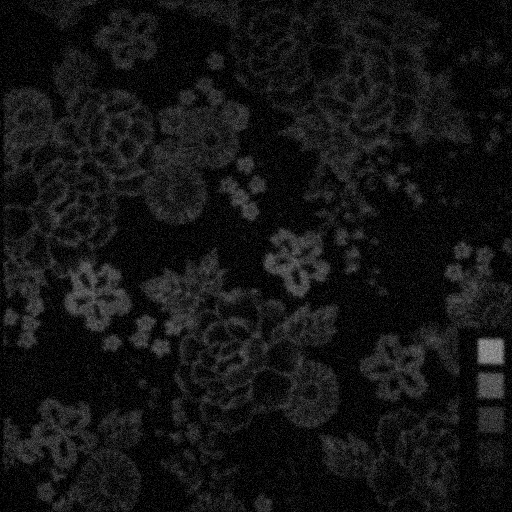}
	\end{minipage}}
 \subfigure[Fused cloth]{
	\begin{minipage}{0.22\linewidth}
		\centering
		\includegraphics[width=\linewidth]{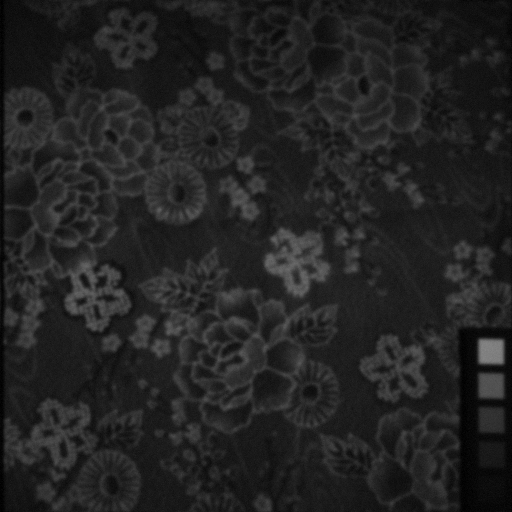}
	\end{minipage}}\\
 \subfigure[15th channel of blurring face]{
	\begin{minipage}{0.22\linewidth}
		\centering
		\includegraphics[width=\linewidth]{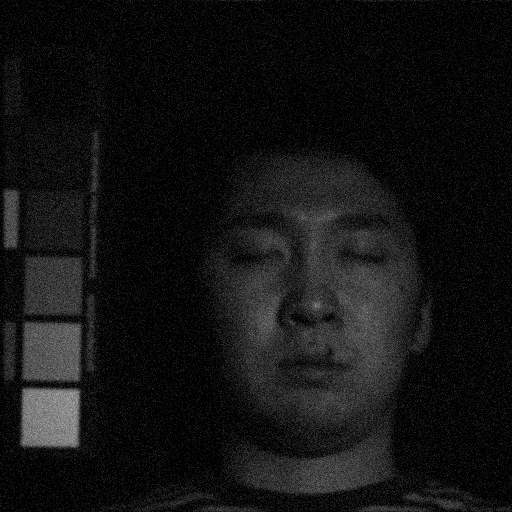}
	\end{minipage}}
  \subfigure[Fused face]{
	\begin{minipage}{0.22\linewidth}
		\centering
		\includegraphics[width=\linewidth]{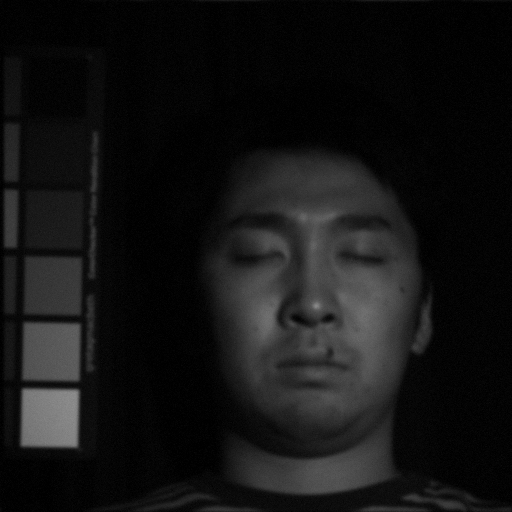}
	\end{minipage}}
 \caption{Comparison between the 15th channel of the contaminated image and the fused result.}
  \label{15thAFuse}
\end{figure}
}
\begin{table}[H]
\caption{Blind Deblurring comparison results.}
\label{Blind Deblurring comparison results}
\centering
\scalebox{0.9}{
\begin{tabular}{|l|l|l|l|l|l|l|l|l|}
\hline
Figure                   & \multicolumn{1}{c|}{Std} & Measure & Degraded image & TV      & VTV     & ASSTV           & SSAHTV  & GOTTV            \\ \hline
\multirow{4}{*}{Cloth}   & \multirow{2}{*}{1}       & MPSNR   & 25.2061        & 30.133  & 30.8451 & {\underline{31.5103}}   & 30.9057 & \textbf{32.5559} \\ \cline{3-9} 
                         &                          & MSSIM   & 0.4412         & 0.7694  & 0.7938  & 0.828           & 0.7973  & \textbf{0.8579}  \\ \cline{2-9} 
                         & \multirow{2}{*}{1.5}     & MPSNR   & 24.5608        & 28.9521 & 29.4969 & 29.9878         & 29.5098 & \textbf{30.8387} \\ \cline{3-9} 
                         &                          & MSSIM   & 0.3686         & 0.7082  & 0.7323  & 0.7658          & 0.7337  & \textbf{0.8011}  \\ \hline
\multirow{4}{*}{Face}    & \multirow{2}{*}{1}       & MPSNR   & 25.9121        & 39.03   & 40.293  & 40.1296         & 39.9156 & \textbf{40.707}  \\ \cline{3-9} 
                         &                          & MSSIM   & 0.2871         & 0.9524  & 0.9641  & 0.9642          & 0.9643  & \textbf{0.9715}  \\ \cline{2-9} 
                         & \multirow{2}{*}{1.5}     & MPSNR   & 25.8033        & 38.0462 & 39.2893 & 39.2794         & 39.0602 & \textbf{39.6478} \\ \cline{3-9} 
                         &                          & MSSIM   & 0.2802         & 0.9459  & 0.957   & 0.9606          & 0.9584  & \textbf{0.9663}  \\ \hline
\multirow{4}{*}{Jelly}   & \multirow{2}{*}{1}       & MPSNR   & 25.2537        & 31.1635 & 32.4262 & 32.8067         & 32.5375 & \textbf{33.8965} \\ \cline{3-9} 
                         &                          & MSSIM   & 0.4539         & 0.9009  & 0.9125  & 0.9231          & 0.9226  & \textbf{0.9293}  \\ \cline{2-9} 
                         & \multirow{2}{*}{1.5}     & MPSNR   & 24.5277        & 29.4468 & 30.6461 & 31.1598         & 30.6259 & \textbf{32.0397} \\ \cline{3-9} 
                         &                          & MSSIM   & 0.4272         & 0.8743  & 0.8865  & 0.8987          & 0.8926  & \textbf{0.9035}  \\ \hline
\multirow{4}{*}{Picture} & \multirow{2}{*}{1}       & MPSNR   & 25.89          & 37.8357 & 39.0866 & 39.5056         & 38.8731 & \textbf{39.939}  \\ \cline{3-9} 
                         &                          & MSSIM   & 0.2986         & 0.9338  & 0.9451  & \textbf{0.9622} & 0.9466  & {\underline{0.958}}      \\ \cline{2-9} 
                         & \multirow{2}{*}{1.5}     & MPSNR   & 25.7343        & 36.4509 & 37.487  & 37.932          & 37.3433 & \textbf{38.2969} \\ \cline{3-9} 
                         &                          & MSSIM   & 0.2877         & 0.9198  & 0.928   & \textbf{0.95}   & 0.9295  & {\underline{0.9454}}     \\ \hline
\multirow{4}{*}{Thread}  & \multirow{2}{*}{1}       & MPSNR   & 25.7171        & 34.8465 & 35.9132 & {\underline{36.7808}}   & 36.025  & \textbf{37.5774} \\ \cline{3-9} 
                         &                          & MSSIM   & 0.3188         & 0.8973  & 0.9103  & {\underline{0.9321}}    & 0.917   & \textbf{0.9324}  \\ \cline{2-9} 
                         & \multirow{2}{*}{1.5}     & MPSNR   & 25.3961        & 33.0605 & 33.9458 & {\underline{34.8356}}   & 34.0577 & \textbf{35.4359} \\ \cline{3-9} 
                         &                          & MSSIM   & 0.3014         & 0.8677  & 0.8799  & \textbf{0.914}  & 0.8876  & {\underline{0.9102}}     \\ \hline
\end{tabular}}
\end{table}

{
\begin{figure}[H]
		\centering
 \subfigure[Ground-truth]{
	\begin{minipage}{0.2\linewidth}
		\centering
		\includegraphics[width=\linewidth]{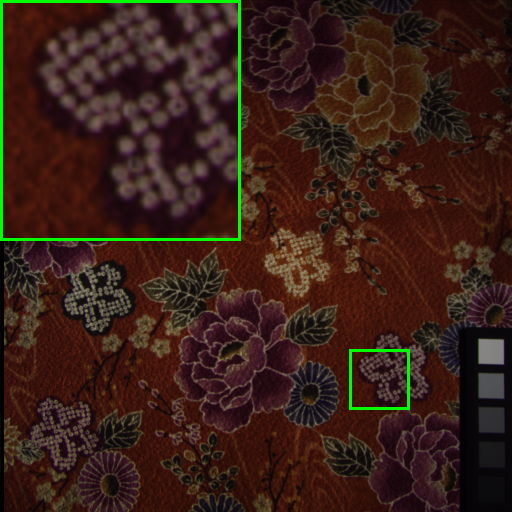}
	\end{minipage}}
 \subfigure[Degraded]{
	\begin{minipage}{0.2\linewidth}
		\centering
		\includegraphics[width=\linewidth]{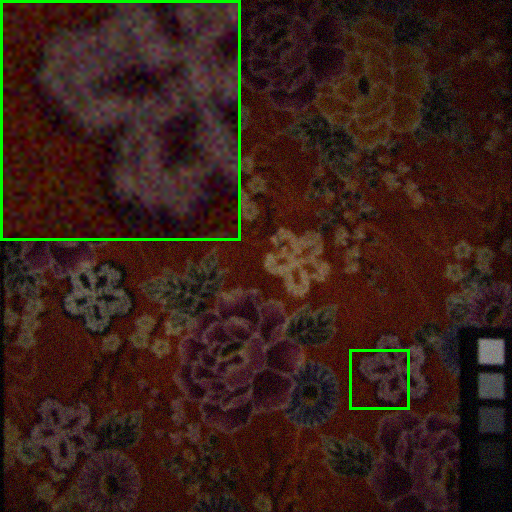}
	\end{minipage}}
 \subfigure[TV]{
	\begin{minipage}{0.2\linewidth}
		\centering
		\includegraphics[width=\linewidth]{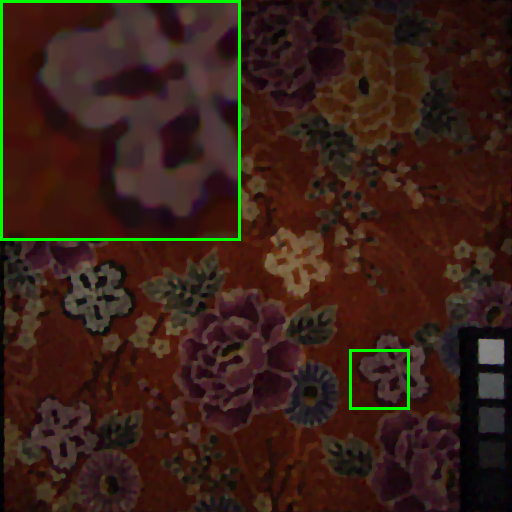}
	\end{minipage}}
  \subfigure[VTV]{
	\begin{minipage}{0.2\linewidth}
		\centering
		\includegraphics[width=\linewidth]{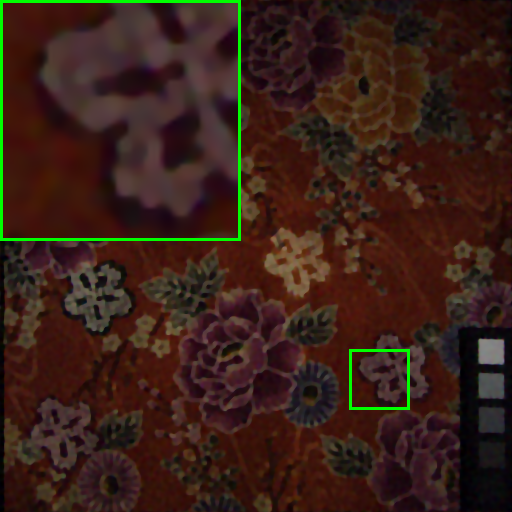}
	\end{minipage}}
   \subfigure[ASSTV]{
	\begin{minipage}{0.2\linewidth}
		\centering
		\includegraphics[width=\linewidth]{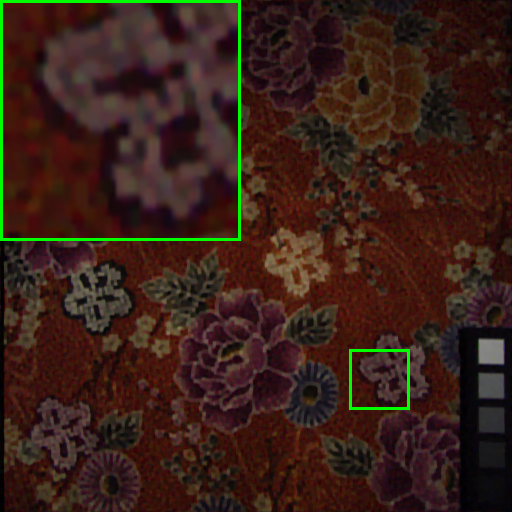}
	\end{minipage}}
    \subfigure[SSAHTV]{
	\begin{minipage}{0.2\linewidth}
		\centering
		\includegraphics[width=\linewidth]{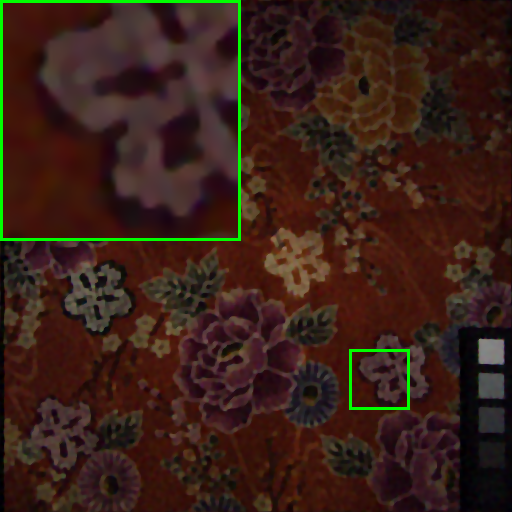}
	\end{minipage}}
    \subfigure[GOTTV]{
	\begin{minipage}{0.2\linewidth}
		\centering
		\includegraphics[width=\linewidth]{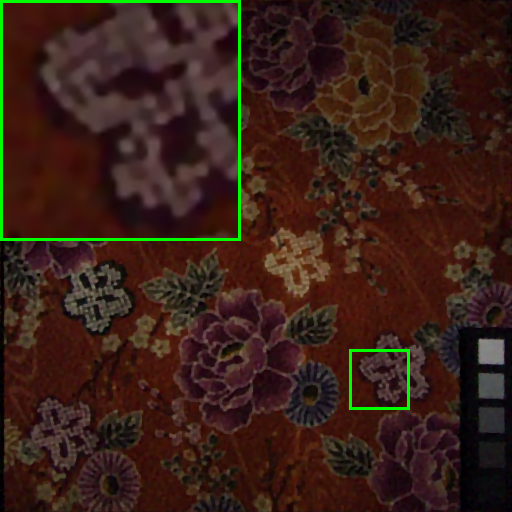}
	\end{minipage}}
 \caption{Blind deblurring comparison results of the cloth image.}
  \label{Blind-deblurring comparison results of the cloth image}
\end{figure}
}

{
\begin{figure}[H]
		\centering
 \subfigure[Ground-truth]{
	\begin{minipage}{0.2\linewidth}
		\centering
		\includegraphics[width=\linewidth]{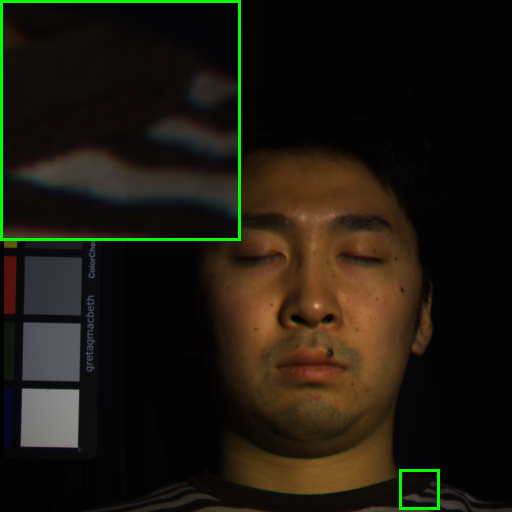}
	\end{minipage}}
 \subfigure[Degraded]{
	\begin{minipage}{0.2\linewidth}
		\centering
		\includegraphics[width=\linewidth]{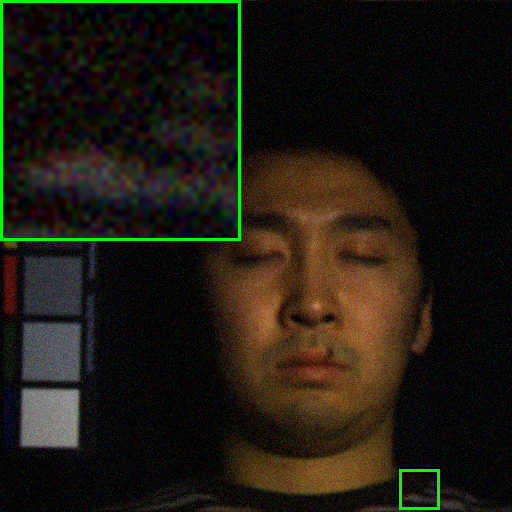}
	\end{minipage}}
 \subfigure[TV]{
	\begin{minipage}{0.2\linewidth}
		\centering
		\includegraphics[width=\linewidth]{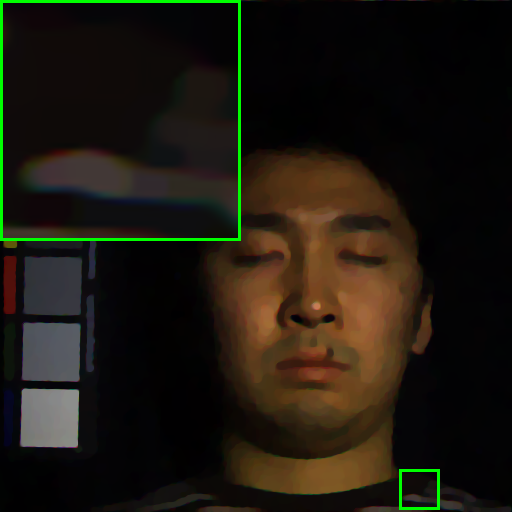}
	\end{minipage}}
  \subfigure[VTV]{
	\begin{minipage}{0.2\linewidth}
		\centering
		\includegraphics[width=\linewidth]{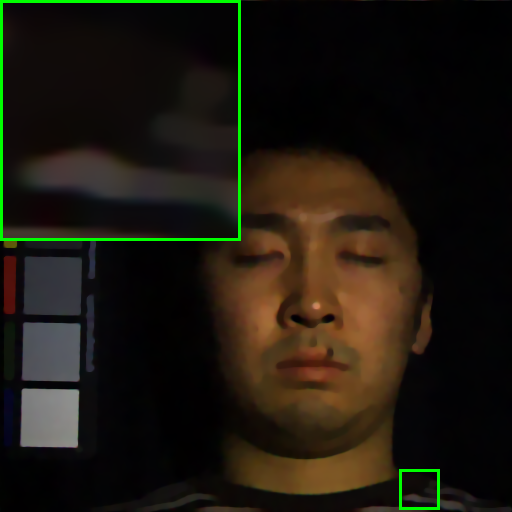}
	\end{minipage}}
   \subfigure[ASSTV]{
	\begin{minipage}{0.2\linewidth}
		\centering
		\includegraphics[width=\linewidth]{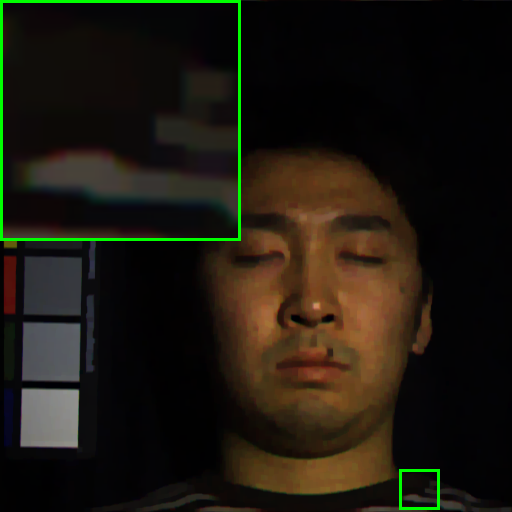}
	\end{minipage}}
    \subfigure[SSAHTV]{
	\begin{minipage}{0.2\linewidth}
		\centering
		\includegraphics[width=\linewidth]{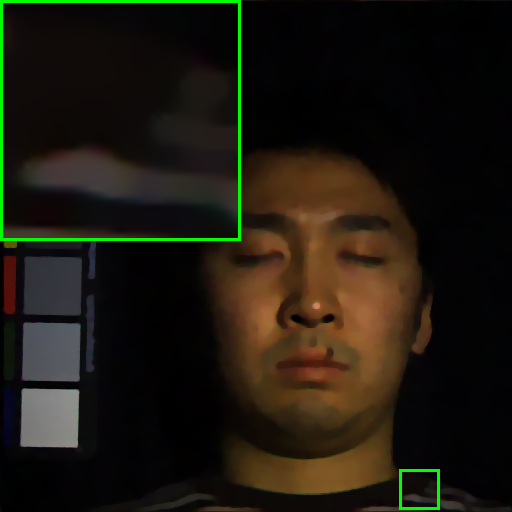}
	\end{minipage}}
    \subfigure[GOTTV]{
	\begin{minipage}{0.2\linewidth}
		\centering
		\includegraphics[width=\linewidth]{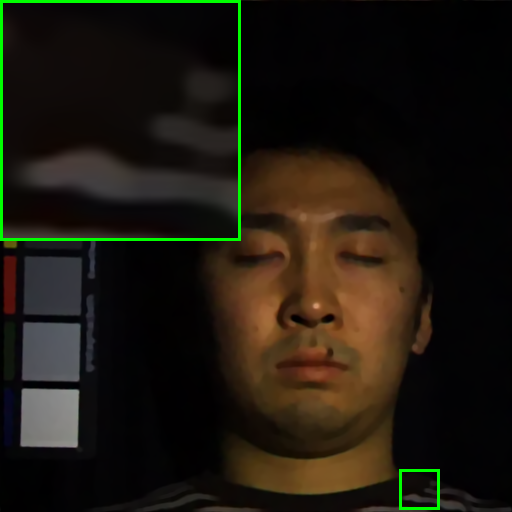}
	\end{minipage}}
 \caption{Blind deblurring comparison results of the face image.}
  \label{Blind-deblurring comparison results of the face image}
\end{figure}
}

\subsection{Robustness Testing}
We demonstrate the robustness of our model from two perspectives. The first involves validating the robustness of denoising outcomes under different opponent transformations. The second focuses on the robustness of denoising results across multiple experiments with random noise.

The preceding discussion has already established theoretically, through equation (\ref{opponent trans eq1}), that employing different opponent transformations does not alter the model. In this subsection, we empirically verify this conclusion through numerical experiments. The images tested are formed by extracting the 1st, 11th, 21st, and 31st channels from multispectral data, resulting in an image with four channels. Our testing approach involves comparing the denoising results of applying all different matrices in $\mathcal{Q}_{d=4}$, totaling 12 variations. Given that all elements in $\mathcal{Q}_{d=4}$ take on a form of $BP$, the below \cref{Different Permutation Comparison results} illustrates the MPSNR and MSSIM corresponding to denoising outcomes for different permutation matrices $P$. A remark that in the table below, $P$ represents the permutation matrix. For example, $P=(4\ 3\ 2\ 1)$ represents the permutation matrix:$ \begin{bmatrix}
    0 &0 &0 &1\\
    0 &0 &1 &0\\
    0 &1 &0 &0\\
    1 &0 &0 &0
\end{bmatrix}$. The experimental results validate our theory that different opponent transformations do not alter the denoising model. Therefore, our method exhibits robustness to various opponent transformations. This means that we can confidently apply our approach in various scenarios without being overly concerned about choosing $P$. This broad applicability enhances the practicality and versatility of our method in real-world denoising problems. It also eases implementation, making our approach more accessible and user-friendly for researchers.

\begin{table}[H]
\centering
\label{Different Permutation Comparison results}
\caption{Different Permutation Comparison results.}
\scalebox{0.75}{
\begin{tabular}{|l|l|l|l|l|l|l|l|l|l|}
\hline
Figure                    & Measure & $P$                          & Result  & $P$                          & Result  & $P$                          & Result  & $P$                          & Result  \\ \hline
\multirow{6}{*}{Cloth}   & MPSNR   & \multirow{2}{*}{(1 2 3 4)} & 31.1327 & \multirow{2}{*}{(3 4 2 1)} & 31.1327 & \multirow{2}{*}{(3 4 1 2)} & 31.1327 & \multirow{2}{*}{(2 4 3 1)} & 31.1327 \\ \cline{2-2} \cline{4-4} \cline{6-6} \cline{8-8} \cline{10-10} 
                         & MSSIM   &                            & 0.8015  &                            & 0.8015  &                            & 0.8015  &                            & 0.8015  \\ \cline{2-10} 
                         & MPSNR   & \multirow{2}{*}{(2 4 1 3)} & 31.1327 & \multirow{2}{*}{(2 3 4 1)} & 31.1327 & \multirow{2}{*}{(2 3 1 4)} & 31.1327 & \multirow{2}{*}{(1 4 3 2)} & 31.1327 \\ \cline{2-2} \cline{4-4} \cline{6-6} \cline{8-8} \cline{10-10} 
                         & MSSIM   &                            & 0.8015  &                            & 0.8015  &                            & 0.8015  &                            & 0.8015  \\ \cline{2-10} 
                         & MPSNR   & \multirow{2}{*}{(1 4 2 3)} & 31.1327 & \multirow{2}{*}{(1 3 4 2)} & 31.1327 & \multirow{2}{*}{(1 3 2 4)} & 31.1327 & \multirow{2}{*}{(1 2 4 3)} & 31.1327 \\ \cline{2-2} \cline{4-4} \cline{6-6} \cline{8-8} \cline{10-10} 
                         & MSSIM   &                            & 0.8015  &                            & 0.8015  &                            & 0.8015  &                            & 0.8015  \\ \hline
\multirow{6}{*}{Face}    & MPSNR   & \multirow{2}{*}{(1 2 3 4)} & 38.7133 & \multirow{2}{*}{(3 4 2 1)} & 38.7133 & \multirow{2}{*}{(3 4 1 2)} & 38.7133 & \multirow{2}{*}{(2 4 3 1)} & 38.7133 \\ \cline{2-2} \cline{4-4} \cline{6-6} \cline{8-8} \cline{10-10} 
                         & MSSIM   &                            & 0.9583  &                            & 0.9583  &                            & 0.9583  &                            & 0.9583  \\ \cline{2-10} 
                         & MPSNR   & \multirow{2}{*}{(2 4 1 3)} & 38.7133 & \multirow{2}{*}{(2 3 4 1)} & 38.7133 & \multirow{2}{*}{(2 3 1 4)} & 38.7133 & \multirow{2}{*}{(1 4 3 2)} & 38.7133 \\ \cline{2-2} \cline{4-4} \cline{6-6} \cline{8-8} \cline{10-10} 
                         & MSSIM   &                            & 0.9583  &                            & 0.9583  &                            & 0.9583  &                            & 0.9583  \\ \cline{2-10} 
                         & MPSNR   & \multirow{2}{*}{(1 4 2 3)} & 38.7133 & \multirow{2}{*}{(1 3 4 2)} & 38.7133 & \multirow{2}{*}{(1 3 2 4)} & 38.7133 & \multirow{2}{*}{(1 2 4 3)} & 38.7133 \\ \cline{2-2} \cline{4-4} \cline{6-6} \cline{8-8} \cline{10-10} 
                         & MSSIM   &                            & 0.9583  &                            & 0.9583  &                            & 0.9583  &                            & 0.9583  \\ \hline
\multirow{6}{*}{Jelly}   & MPSNR   & \multirow{2}{*}{(1 2 3 4)} & 32.3175 & \multirow{2}{*}{(3 4 2 1)} & 32.3175 & \multirow{2}{*}{(3 4 1 2)} & 32.3175 & \multirow{2}{*}{(2 4 3 1)} & 32.3175 \\ \cline{2-2} \cline{4-4} \cline{6-6} \cline{8-8} \cline{10-10} 
                         & MSSIM   &                            & 0.8975  &                            & 0.8975  &                            & 0.8975  &                            & 0.8975  \\ \cline{2-10} 
                         & MPSNR   & \multirow{2}{*}{(2 4 1 3)} & 32.3175 & \multirow{2}{*}{(2 3 4 1)} & 32.3175 & \multirow{2}{*}{(2 3 1 4)} & 32.3175 & \multirow{2}{*}{(1 4 3 2)} & 32.3175 \\ \cline{2-2} \cline{4-4} \cline{6-6} \cline{8-8} \cline{10-10} 
                         & MSSIM   &                            & 0.8975  &                            & 0.8975  &                            & 0.8975  &                            & 0.8975  \\ \cline{2-10} 
                         & MPSNR   & \multirow{2}{*}{(1 4 2 3)} & 32.3175 & \multirow{2}{*}{(1 3 4 2)} & 32.3175 & \multirow{2}{*}{(1 3 2 4)} & 32.3175 & \multirow{2}{*}{(1 2 4 3)} & 32.3175 \\ \cline{2-2} \cline{4-4} \cline{6-6} \cline{8-8} \cline{10-10} 
                         & MSSIM   &                            & 0.8975  &                            & 0.8975  &                            & 0.8975  &                            & 0.8975  \\ \hline
\multirow{6}{*}{Picture} & MPSNR   & \multirow{2}{*}{(1 2 3 4)} & 37.4538 & \multirow{2}{*}{(3 4 2 1)} & 37.4538 & \multirow{2}{*}{(3 4 1 2)} & 37.4538 & \multirow{2}{*}{(2 4 3 1)} & 37.4538 \\ \cline{2-2} \cline{4-4} \cline{6-6} \cline{8-8} \cline{10-10} 
                         & MSSIM   &                            & 0.9404  &                            & 0.9404  &                            & 0.9404  &                            & 0.9404  \\ \cline{2-10} 
                         & MPSNR   & \multirow{2}{*}{(2 4 1 3)} & 37.4538 & \multirow{2}{*}{(2 3 4 1)} & 37.4538 & \multirow{2}{*}{(2 3 1 4)} & 37.4538 & \multirow{2}{*}{(1 4 3 2)} & 37.4538 \\ \cline{2-2} \cline{4-4} \cline{6-6} \cline{8-8} \cline{10-10} 
                         & MSSIM   &                            & 0.9404  &                            & 0.9404  &                            & 0.9404  &                            & 0.9404  \\ \cline{2-10} 
                         & MPSNR   & \multirow{2}{*}{(1 4 2 3)} & 37.4538 & \multirow{2}{*}{(1 3 4 2)} & 37.4538 & \multirow{2}{*}{(1 3 2 4)} & 37.4538 & \multirow{2}{*}{(1 2 4 3)} & 37.4538 \\ \cline{2-2} \cline{4-4} \cline{6-6} \cline{8-8} \cline{10-10} 
                         & MSSIM   &                            & 0.9404  &                            & 0.9404  &                            & 0.9404  &                            & 0.9404  \\ \hline
\multirow{6}{*}{Thread}  & MPSNR   & \multirow{2}{*}{(1 2 3 4)} & 35.4816 & \multirow{2}{*}{(3 4 2 1)} & 35.4816 & \multirow{2}{*}{(3 4 1 2)} & 35.4816 & \multirow{2}{*}{(2 4 3 1)} & 35.4816 \\ \cline{2-2} \cline{4-4} \cline{6-6} \cline{8-8} \cline{10-10} 
                         & MSSIM   &                            & 0.9061  &                            & 0.9061  &                            & 0.9061  &                            & 0.9061  \\ \cline{2-10} 
                         & MPSNR   & \multirow{2}{*}{(2 4 1 3)} & 35.4816 & \multirow{2}{*}{(2 3 4 1)} & 35.4816 & \multirow{2}{*}{(2 3 1 4)} & 35.4816 & \multirow{2}{*}{(1 4 3 2)} & 35.4816 \\ \cline{2-2} \cline{4-4} \cline{6-6} \cline{8-8} \cline{10-10} 
                         & MSSIM   &                            & 0.9061  &                            & 0.9061  &                            & 0.9061  &                            & 0.9061  \\ \cline{2-10} 
                         & MPSNR   & \multirow{2}{*}{(1 4 2 3)} & 35.4816 & \multirow{2}{*}{(1 3 4 2)} & 35.4816 & \multirow{2}{*}{(1 3 2 4)} & 35.4816 & \multirow{2}{*}{(1 2 4 3)} & 35.4816 \\ \cline{2-2} \cline{4-4} \cline{6-6} \cline{8-8} \cline{10-10} 
                         & MSSIM   &                            & 0.9061  &                            & 0.9061  &                            & 0.9061  &                            & 0.9061  \\ \hline
\end{tabular}
}
\end{table}
Also, we conducted multiple tests on different denoising methods under Gaussian noise with a standard deviation of 0.1. The results of these tests are summarized in the following \cref{Different Permutation Comparison results2}, which provides the average, minimum, maximum MPSNR, and MSSIM values for the multiple denoising outcomes. The results of the multiple experiments demonstrate that our method exhibits robustness to noise, maintaining superior performance across various metrics throughout multiple trials. This underscores the advantageous stability of our approach, ensuring consistently optimal results.

\begin{table}[H]
\centering
\label{Different Permutation Comparison results2}
\caption{Multiple tests results.}
\scalebox{1}{
\begin{tabular}{|l|ll|l|l|l|l|l|}
\hline
Image                    & \multicolumn{2}{c|}{Measure}                       & CBCTV   & VTV     & ASSTV         & SSAHTV        & GOTTV            \\ \hline
\multirow{6}{*}{Cloth}   & \multicolumn{1}{l|}{\multirow{3}{*}{MPSNR}} & mean & 29.2046 & 30.0293 & 29.4281       & {\underline{30.1969}} & \textbf{33.2714} \\ \cline{3-8} 
                         & \multicolumn{1}{l|}{}                       & min  & 29.1983 & 30.0218 & 29.4188       & {\underline{30.1886}} & \textbf{33.2610} \\ \cline{3-8} 
                         & \multicolumn{1}{l|}{}                       & max  & 29.2122 & 30.0369 & 29.4339       & {\underline{30.2071}} & \textbf{33.2821} \\ \cline{2-8} 
                         & \multicolumn{1}{l|}{\multirow{3}{*}{MSSIM}} & mean & 0.7224  & 0.7620  & 0.7396        & {\underline{0.7712}}  & \textbf{0.8740}  \\ \cline{3-8} 
                         & \multicolumn{1}{l|}{}                       & min  & 0.7220  & 0.7617  & 0.7391        & {\underline{0.7708}}  & \textbf{0.8738}  \\ \cline{3-8} 
                         & \multicolumn{1}{l|}{}                       & max  & 0.7227  & 0.7622  & 0.7401        & {\underline{0.7714}}  & \textbf{0.8743}  \\ \hline
\multirow{6}{*}{Face}    & \multicolumn{1}{l|}{\multirow{3}{*}{MPSNR}} & mean & 36.9643 & 38.3135 & 37.9387       & {\underline{38.5726}} & \textbf{40.4344} \\ \cline{3-8} 
                         & \multicolumn{1}{l|}{}                       & min  & 36.9424 & 38.2884 & 37.9188       & {\underline{38.5489}} & \textbf{40.4046} \\ \cline{3-8} 
                         & \multicolumn{1}{l|}{}                       & max  & 36.9854 & 38.3346 & 37.9545       & {\underline{38.5904}} & \textbf{40.4516} \\ \cline{2-8} 
                         & \multicolumn{1}{l|}{\multirow{3}{*}{MSSIM}} & mean & 0.9211  & 0.9464  & 0.9389        & {\underline{0.9483}}  & \textbf{0.9635}  \\ \cline{3-8} 
                         & \multicolumn{1}{l|}{}                       & min  & 0.9202  & 0.9457  & 0.9379        & {\underline{0.9475}}  & \textbf{0.9629}  \\ \cline{3-8} 
                         & \multicolumn{1}{l|}{}                       & max  & 0.9219  & 0.9470  & 0.9393        & {\underline{0.9489}}  & \textbf{0.9640}  \\ \hline
\multirow{6}{*}{Jelly}   & \multicolumn{1}{l|}{\multirow{3}{*}{MPSNR}} & mean & 31.0280 & 32.3230 & 31.5866       & {\underline{32.6763}} & \textbf{34.5981} \\ \cline{3-8} 
                         & \multicolumn{1}{l|}{}                       & min  & 31.0173 & 32.3106 & 31.5751       & {\underline{32.6626}} & \textbf{34.5847} \\ \cline{3-8} 
                         & \multicolumn{1}{l|}{}                       & max  & 31.0372 & 32.3303 & 31.5957       & {\underline{32.6827}} & \textbf{34.6070} \\ \cline{2-8} 
                         & \multicolumn{1}{l|}{\multirow{3}{*}{MSSIM}} & mean & 0.8626  & 0.8911  & 0.8836        & {\underline{0.8985}}  & \textbf{0.9125}  \\ \cline{3-8} 
                         & \multicolumn{1}{l|}{}                       & min  & 0.8622  & 0.8906  & 0.8834        & {\underline{0.8981}}  & \textbf{0.9120}  \\ \cline{3-8} 
                         & \multicolumn{1}{l|}{}                       & max  & 0.8628  & 0.8913  & 0.8839        & {\underline{0.8987}}  & \textbf{0.9129}  \\ \hline
\multirow{6}{*}{Picture} & \multicolumn{1}{l|}{\multirow{3}{*}{MPSNR}} & mean & 35.4978 & 36.8850 & {\underline{37.1554}} & 37.0422       & \textbf{39.2225} \\ \cline{3-8} 
                         & \multicolumn{1}{l|}{}                       & min  & 35.4922 & 36.8778 & {\underline{37.1477}} & 37.0249       & \textbf{39.2127} \\ \cline{3-8} 
                         & \multicolumn{1}{l|}{}                       & max  & 35.5028 & 36.8929 & {\underline{37.1684}} & 37.0562       & \textbf{39.2300} \\ \cline{2-8} 
                         & \multicolumn{1}{l|}{\multirow{3}{*}{MSSIM}} & mean & 0.8796  & 0.9157  & {\underline{0.9302}}  & 0.9191        & \textbf{0.9406}  \\ \cline{3-8} 
                         & \multicolumn{1}{l|}{}                       & min  & 0.8792  & 0.9151  & {\underline{0.9296}}  & 0.9186        & \textbf{0.9401}  \\ \cline{3-8} 
                         & \multicolumn{1}{l|}{}                       & max  & 0.8803  & 0.9165  & {\underline{0.9307}}  & 0.9198        & \textbf{0.9413}  \\ \hline
\multirow{6}{*}{Thread}  & \multicolumn{1}{l|}{\multirow{3}{*}{MPSNR}} & mean & 33.7628 & 35.0226 & 34.1741       & {\underline{35.3243}} & \textbf{37.9447} \\ \cline{3-8} 
                         & \multicolumn{1}{l|}{}                       & min  & 33.7549 & 35.0141 & 34.1611       & {\underline{35.3132}} & \textbf{37.9206} \\ \cline{3-8} 
                         & \multicolumn{1}{l|}{}                       & max  & 33.7698 & 35.0304 & 34.1843       & {\underline{35.3312}} & \textbf{37.9623} \\ \cline{2-8} 
                         & \multicolumn{1}{l|}{\multirow{3}{*}{MSSIM}} & mean & 0.8689  & 0.9002  & 0.9055        & {\underline{0.9078}}  & \textbf{0.9342}  \\ \cline{3-8} 
                         & \multicolumn{1}{l|}{}                       & min  & 0.8679  & 0.8995  & 0.9052        & {\underline{0.9072}}  & \textbf{0.9336}  \\ \cline{3-8} 
                         & \multicolumn{1}{l|}{}                       & max  & 0.8696  & 0.9009  & 0.9059        & {\underline{0.9083}}  & \textbf{0.9346}  \\ \hline
\end{tabular}
}
\end{table}

\section{Conclusion}
\label{sec:Conclusion}

In this paper, we have made several contributions. First, we introduced the concept of the generalized opponent transformation, which inherits the key characteristic of 3-dimensional opponent transformations. This transformation allows us to decorrelate the opponent and average information between different spectral bands in multispectral images. Next, we investigated the eigendecomposition properties of the generalized opponent transformation, which provided valuable insights into its behavior and properties. Based on the generalized opponent transformation, we proposed the generalized opponent total variation multispectral image restoration model. This model incorporates both the opponent and average information between image spectra, taking advantage of the benefits of total variation regularization in preserving image boundaries and suppressing noise. We conducted numerical experiments to evaluate the performance of the GOTTV model. The results demonstrated that GOTTV outperforms other methods in terms of metrics such as MPSNR and MSSIM. Additionally, GOTTV showed superior performance in texture and detail recovery compared to the other methods considered in the comparison. Overall, our proposed approach utilizing the generalized opponent transformation and the GOTTV model provides an effective and promising method for multispectral image restoration, yielding improved results in terms of both quantitative metrics and visual quality.


\bibliographystyle{siamplain}
\bibliography{reference_new}

\end{document}